\RequirePackage{pgf,tikz}

\documentclass[10pt,a4paper]{article}
\usepackage[latin1]{inputenc}
\usepackage[english]{babel}
\usepackage[none]{hyphenat}
\usepackage[T1]{fontenc}
\usepackage[labelfont=bf]{caption}  
\usepackage[absolute]{textpos}

\usepackage[left=1.5cm,right=1.5cm,top=1.5cm,bottom=1.5cm]{geometry}

\usepackage[pdfborder={0 0 0},colorlinks=true,urlcolor=blue,linkcolor=blue,citecolor=blue,bookmarks=true]{hyperref}

\hypersetup{pdftitle={Solving Travelling Thief Problems using Coordination Based Methods}}
\hypersetup{pdfauthor={Majid Namazi, M.A. Hakim Newton, Conrad Sanderson, Abdul Sattar}}

\usepackage{newtxtext}     
\usepackage{inconsolata}   

\usepackage{microtype}
\DisableLigatures[f]{encoding = *, family = * }

\usepackage{algorithm}
\usepackage{adjustbox}
\usepackage[noEnd]{algpseudocodex}
\usepackage{graphicx}
\usepackage{textcomp}
\usepackage{verbatim}
\usepackage{amsmath}
\usepackage{amssymb}
\usepackage{amsthm}
\usepackage{amsfonts}
\usepackage{url}
\usepackage{multirow}
\usepackage{multicol}
\usepackage{titlesec}

\usetikzlibrary{calc}

\theoremstyle{thmstylethree}%
\newtheorem{lemma}{Lemma}%
\newtheorem{definition}{Definition}%

\algnewcommand\LeftComment[1]{%
\textcolor{gray}{$\triangleright$~{#1}} \hfill %
}

\newcommand{\ignore}[1]{}

\begin{document}

\begin{textblock}{13.4}(1.3,14.9)
\hrule
\vspace{1ex}
\noindent
\scriptsize
\textbf{{$^\ast$}~Published in:} Journal of Heuristics, 2023. DOI:~\href{https://doi.org/10.1007/s10732-023-09518-7}{\tt 10.1007/s10732-023-09518-7}
\end{textblock}

\renewcommand{\baselinestretch}{1.1}\small\normalsize

\pagestyle{empty}

\begin{center}

\scalebox{1.6}{\bf Solving Travelling Thief Problems using Coordination Based Methods}

\vspace{2ex}

Majid Namazi{\tiny~}{$^{1,2}$},
M.A. Hakim Newton{\tiny~}{$^{2,3}$}, 
Conrad Sanderson{\tiny~}{$^{1,2}$},
Abdul Sattar{\tiny~}{$^{2}$}

\vspace{2ex}

\begin{small}
\begin{tabular}{l}
$^{1}$~Data61~/~CSIRO, Australia\\
$^{2}$~Griffith University, Australia\\
$^{3}$~University of Newcastle, Australia\\
\end{tabular}
\end{small}

\end{center}

\hrule
\vspace{-2ex}
\section*{Abstract}
\vspace{-1ex}

A travelling thief problem (TTP) is a proxy to real-life problems such as postal collection. TTP 
comprises an entanglement of a travelling salesman problem (TSP) and a knapsack problem (KP) since 
items of KP are scattered over cities of TSP, and a thief has to visit cities to collect items.
In~TTP, city selection and item selection decisions need close coordination since the thief's 
travelling speed depends on the knapsack's weight and the order of visiting cities affects the order 
of item collection. Existing TTP solvers deal with city selection and item selection separately, 
keeping decisions for one type unchanged while dealing with the other type. This separation 
essentially means very poor coordination between two types of decision. In this paper, we first show 
that a simple local search based coordination approach does not work in TTP. Then, to address the 
aforementioned problems, we propose a human designed coordination heuristic that makes changes to 
collection plans during exploration of cyclic tours. We further propose another human designed 
coordination heuristic that explicitly exploits the cyclic tours in item selections during 
collection plan exploration. Lastly, we propose a machine learning based coordination heuristic that 
captures characteristics of the two human designed coordination heuristics. Our proposed 
coordination based approaches help our TTP solver significantly outperform existing state-of-the-art 
TTP solvers on a set of benchmark problems. Our solver is named Cooperation Coordination (CoCo) and 
its source code is available from \href{https://github.com/majid75/CoCo}{https://github.com/majid75/CoCo}


\vspace{2ex}
\hrule

\section{Introduction}
\label{sec:introduction}

Travelling salesman problem (TSP) and knapsack problem (KP) are two well-known NP-Hard combinatorial optimisation problems. In TSP~\cite{gutin2006traveling}, a salesman performs a {\em cyclic tour} through a set of cities with a goal of minimising the length (or hence the travelling time) of the cyclic tour. In KP~\cite{inbookKP}, using a {\em collection plan}, a knapsack with a given capacity is filled in with a subset of given profitable items with a goal of maximising the total profit. 

TSP and KP are classical problems. However, real-world applications such as postal or waste collection problems \cite{mei2014improving,polyakovskiy2017packing,hannan2020waste} need more complex problem models in which both TSP and KP characteristics are intrinsically and interdependently present simultaneously. Problems in which KP items are scattered over TSP cities are modelled in various ways such as cumulative capacitated routing problem \cite{ngueveu2010effective}, orienteering problem \cite{vansteenwegen2011orienteering}, and selective/prize collecting travelling salesman problem \cite{balas2007prize,laporte1990selective}. In some of these models, instead of a single tour, multiple tours are involved  while in some other models, a city might not be visited if no item is collected from the city.

In this paper, we study a particular problem model named {\em travelling thief problem} \cite{bonyadi2013travelling} that essentially encompasses both TSP and KP characteristics in an entangled fashion. In an example problem having this kind of model, a postal truck mandatorily visits each city to collect letters. Moreover, the postal truck makes profits as it optionally collects heavy parcels from the cities. However, the gradual change in the truck load as the truck picks heavy parcels affects its travelling speed between cities, and hence affect the travelling time, fuel consumption, air pollution, and travelling cost. A solution for such a problem is a mandatory cyclic tour of cities to be visited successively with a sequence of optional services to be given at the cities such that the total profit made minus the total cost incurred is maximised.

In a travelling thief problem (TTP), a thief ({\it i}) rents a knapsack having a certain capacity at a given renting cost per unit of time, ({\it ii}) performs a cyclic tour through a set of cities, and ({\it iii}) collects a subset of profitable items in the knapsack with the objective of maximising the net profit, which equals total profit minus total cost. The entanglement of TSP and KP in TTP comes from two factors: ({\it i}) As the thief collects more items, the knapsack gets heavier, the thief gets slower, the tour takes longer time, and the knapsack rent goes up; and ({\it ii}) the order of cities in the tour affects the order of items that could be collected without exceeding the knapsack capacity. TTP is a multi-component problem since it has both TSP and KP as components. Solving such multi-component problems is more challenging because finding an optimal overall solution to a multi-component problem cannot be guaranteed by simply finding an optimal solution to each underlying component~\cite{michalewicz2012quo,mei2016investigation,bonyadi2019evolutionary}.

\newpage
TTP solving methods have obtained some progress over the years but further improvement is needed. Here we summarise five types of TTP methods in the context of the paper but a detailed exploration is presented later. Constructive methods \cite{polyakovskiy2014comprehensive,bonyadi2014socially,mei2014improving,faulkner2015approximate} use Chained Lin-Kernighan heuristic~\cite{applegate2003chained} to get a cyclic tour and then use various heuristics to construct a collection plan. Fixed tour methods \cite{maity2020efficient,polyakovskiy2014comprehensive,faulkner2015approximate,wu2017exact,polyakovskiy2017packing} generate cyclic tours like constructive methods and then use exact or approximate methods to find collection plans. Cooperative methods \cite{bonyadi2014socially,el2016population,wagner2017case,el2018efficiently,namazi2020surrogate,zhang2021solving} iteratively alternate between a search for a cyclic tour and another search of a collection plan, keeping one of the two unchanged while searching for the other until no further improvement. Full-encoding methods \cite{mei2014improving,wagner2016stealing,el2016population,el2017local,wuijts2019investigation} deals with the entire TTP problem at a time using cyclic tour and collection plan changing operators within the same search framework. Hyper-heuristic methods \cite{ali2020hyper,mei2015heuristic,martins2017hseda,el2018hyperheuristic} generate or select low level heuristics as neighbourhood operators for cyclic tours or collection plans and use them in search.

Given the TTP literature summarised above, one aspect common among all approaches is that the search for one component's solution (cyclic tour or collection plan) takes only the other component's unchanged current solution (collection plan or cyclic tour) into account. Moreover, some approaches adopts an iterative strategy to alternate between the aforementioned neighbourhood operators for the two components. However, even these approaches might not really help get the overall search direction best for solving the entire multi-component problem. The reason is for the best solution of the entire problem, solutions for the two components, at the same time, should mutually best correspond to each other. Note that the coordination issue has been partially addressed by exact evaluation of the collection plans or city selections in related problems other than TTP and such problems include generalized traveling salesman problem \cite{bontoux2010memetic}, cumulative capacitated vehicle routing problem \cite{ngueveu2010effective}, and vehicle routing problems with profits \cite{vidal2016large}. However, such exact evaluation via dynamic programming or labelling algorithms is usually costly and do not scale well for large problems. Ideally, a heuristic based cheaper but strong coordination method is needed between solving methods for the TSP and KP components in TTP meaning any considerable change on one component's solution should take into account the other component's all possible future solutions subject to the one component's same solution.

In this paper, we first show that even a simple local search based coordination approach, let alone an exact evaluation based approach, is not effective in addressing the poor coordination issue in existing TTP methods. Then, we propose a human designed coordination heuristic that makes changes to collection plans during exploration of cyclic tours. We also propose another human designed coordination heuristic that explicitly exploits the cyclic tours in item selections during collection plan exploration. We further propose a machine learning based coordination heuristic that captures characteristics of the two human designed coordination heuristics. Our proposed coordination heuristics  explore potentially better TTP solutions than the approaches exhibiting poor coordination. We empirically evaluate the effectiveness of our proposed approaches. On a set of benchmark problems, our proposed approaches help our coordination based TTP solver significantly outperform existing state-of-the-art TTP solvers. Our TTP solver is named Cooperation Coordination (CoCo) and it is available from \url{https://github.com/majid75/CoCo}.

We note that this paper thoroughly extends our previous preliminary work \cite{namazi2019pgch} that has presented an early version of our human designed cyclic tour exploration coordination heuristic. In this paper, we have considerably revised the human designed heuristic. Furthermore, we have designed two cyclic tour exploration coordination heuristics: ({\it i}) a local search based heuristic and ({\it ii}) a machine learning based heuristic. Next, we have designed a coordination based collection plan heuristic. Moreover, we have described our proposed approaches more formally and in greater details.

We continue the paper as follows.
Section~\ref{sec:prelim} covers preliminaries.
Section~\ref{sec:related} explores related work.
Section~\ref{sec:framework} describes the search framework used.
Section~\ref{sec:proposed} describes the proposed coordination heuristics.
Section~\ref{sec:experiments} presents the experimental results.
Section~\ref{sec:conclusion} presents the conclusions.

\section{Preliminaries}
\label{sec:prelim}

We formally define TSP, KP, and TTP. We describe the neighbourhood operators \textsf{2OPT} for TSP and \textsf{BitFlip} for KP. We also define the prefix minimum and suffix maximum functions to help describe TTP coordination heuristics.

\subsection{Travelling Salesman Problem}\label{sec:TSP}

Assume a set {$C = \{1,\ldots,n\}$} of $n>1$ {\em cities}. The {\em distance} between each two cities {$c\ne c'$} is {$d(c,c') = d(c',c)$}. In TSP, a salesman starts travelling from city 1 and visits each other city exactly once and returns back to city 1. The salesman thus completes a non-overlapping cyclic tour through all cities. For a given set $C$ of cities, assume $t = \langle t_0, t_1, \ldots, t_n\rangle$ is a {\em cyclic tour} with $t_0 = t_n = 1$, and $t_k=c$ iff $t(c)=k$, where $c \in C \setminus \{1\}$ is a city and $k\in[1,n-1]$ is a {\em position}. No city in $C \setminus \{1\}$ can be visited more than once in a cyclic tour $t$. So we have $t_k \ne t_{k'}$ for any $k \ne k'$ where $k,k'\in[1,n-1]$. Given a cyclic tour $t$, the {\em total distance} travelled by the salesman is $D(t) = \sum_{k= 0}^{k<n}d(t_{k},t_{k+1})$.

\begin{definition}[TSP]
Given a set $C$ of cities, distance $d(c,c')$ between each pair of cities $c\ne c'$, find a cyclic tour $t$ for a salesman such that the objective total distance $D(t)$ is minimised. Note that the objective could be the total travelling time if the travelling speed is constant between each two cities.
\end{definition}

\newpage
Given a cylic tour $t$, a {\em tour segment} $t[b,e] = \langle t_{b}, \ldots, t_{e}\rangle$ of {\em length} $|t[b,e]| = e-b+1$ with $0 < b < e< n$ comprises cities in $t$ between positions $b$ and $e$ both inclusive. In TSP, a tour segment reversal operator {2OPT}~\cite{croes1958method} is often used in generating a neighbouring tour from a given tour.

\begin{definition}[{2OPT}]
Given a cyclic tour $t$ and positions $b$ and $e$ such that $0 < b < e < n$, a {2OPT}$(t,b,e)$ operator reverses the tour segment $t[b,e]$ of length $|e-b+1|$. So {2OPT} essentially reverses the order of cities between positions {$b$} and {$e$} to produce a new tour~{$t'$}. Thus, {$t'_{b+k} = t_{e-k}$} is obtained for {$0 \leq k \leq e - b$} taking $\mathcal{O}(e-b)$ time. Any other city at position $k \not\in [b,e]$ remains at the same position, i.e. $t'_k = t_k$.
\end{definition}

\begin{lemma}[{2OPT} for TSP]
Given a cyclic tour $t$ in TSP, a {2OPT}$(t,b,e)$ operator produces a new cyclic tour $t'$ for which computing $D(t') = D(t) + d(t_{b-1},t_{e})+ d(t_{b}, t_{e + 1}) -d(t_{b-1},t_{b}) - d(t_{e},t_{e+1})$ takes $\mathcal{O}(1)$ time, when $D(t)$ is already known.
\end{lemma}

\ignore{ 

\begin{proof}
The distance between cities at positions $b - 1$ and $b$ in $t$ is replaced by the distance between cities at positions $b-1$ and $e$ in $t$. Similarly, the distance between cities at positions $e$ and $e + 1$ in $t$ is replaced by the distance between cities at positions $b$ and $e+1$ in $t$. The distance of any other city at position $k : b < k < e$ from its succeeding or preceding city remains the same.
\end{proof}
}

\subsection{Knapsack Problems}\label{sec:KP}

Assume a set {$I = \{1,\ldots,m\}$} of {$m>0$} {\em items}. Each item has {\em weight} {$w_i> 0$} and {\em profit} {$\pi_i > 0$}. Assume $p = \langle p_1, p_2, \ldots, p_m \rangle \equiv \{ i:p_i = 1\}$ is a {\em collection plan} with $p_i \in \{0,1\}$ for each item $i$, where $p_i = 1$ means $i$ is a {\em collected item} and $p_i = 0$ means $i$ is an {\em uncollected item}. Assume the knapsack has {\em weight capacity} $W>0$. For a given collection plan $p$, the {\em total weight} of the knapsack is $W(p) = \sum^{i=m}_{i=1}w_ip_i$, the {\em knapsack constraint} is $K(p)\equiv W(p) \leq W$, and the {\em total profit} of the collected items is $P(p) = \sum_{i=1}^{i=m}\pi_ip_i$.

\begin{definition}[KP]
Given a set $I$ of items with weight $w_i$ and profit $\pi_i$ for each item $i$ and also the knapsack capacity $W$, find a collection plan $p$ such that the objective total profit $P(p)$ is maximised subject to the knapsack constraint $K(p) \equiv W(p) \leq W$.
\end{definition}

In KP, an item selection operator {BitFlip}~\cite{polyakovskiy2014comprehensive} is often used in generating a neihbouring collection plan from a given one.

\begin{definition}[{BitFlip}]
Given a collection plan $p$ and an item $i$, a {BitFlip}$(p,i)$ operator flips $p_i$ from $0$ to $1$ or vice versa to produce a new collection plan $p'$ taking $\mathcal{O}(1)$ time.
\end{definition}

\begin{lemma}[{BitFlip} for KP]
Given a collection plan $p$ in KP, a {BitFlip}$(p,i)$ operator produces a new collection plan $p'$ with $P(p') = P(p) + \pi_i\times(p'_i - p_i)$. Here, computation of $P(p')$ takes $\mathcal{O}(1)$ time when $P(p)$ is already known.
\end{lemma}

For convenience of exposition and for the sake of formality, below we define pick and unpick operations for collection plans in KP.

\begin{definition}[Pick and Unpick]
Given a collection plan $p$ in KP, {\em picking} an item $i$ is when $p_i = 0$ and {BitFlip}$(p,i)$ is applied. Moreover, {\em unpicking} an item $i$ is when $p_i = 1$ and {BitFlip}$(p,i)$ is applied on a collection $p$.
\end{definition}

\subsection{Travelling Thief Problems}\label{sec:TTP}

We start with all the notations and terminologies used for TSP and KP in Sections~\ref{sec:TSP} and \ref{sec:KP} respectively. However, the salesman in TSP is viewed as the thief in TTP, the items in KP are scattered over the cities in TTP, and the thief travels around to collect the items. Moreover, the travelling speed in TTP gets slower as the theif collects items and the knapsack gets heavier.

 Assume, in TTP, each item {$i$} is collected from a city $l_i$ and a city $c$ has a set $I(c) = \{i : l_i = c\}$ of items. However, the designated city 1, where the cyclic tour $t$ of the thief starts from and ends at, arguably does not have any item since such an item could be collected without any travelling. So, $l_i > 1$ for any item $i$ and $I(1) = \{\}$. A TTP {\em solution} $\langle t, p\rangle$ comprises a cyclic tour $t$ and a collection plan $p$. An item $i$ in a city $l_i$ has a position $t(l_i)$ in a cyclic tour $t$.

Assume the thief in TTP rents a knapsack of {\em weight capacity} $W>0$ at a {\em renting rate} of $R>0$ per unit of time. For a given collection plan $p$, also assume the total weight of the items collected from city {$c$} is $w_p(c) = \sum_{i \in I(c)}w_ip_i$. Further, assume that $w_{t, p}(k) = \sum^{k'=k}_{k'=0}w_p(t_{k'})$ denotes the weight of the knapsack after collecting items from cities up to position $k$ in the tour $t$ using a collection plan $p$ for a TTP solution $\langle t, p\rangle$. Assume a {\em speed function} $s(w) = s_{\max} - \frac{w}{W} \times (s_{\max} - s_{\min})$ for the current knapsack weight $w \leq W$, where the given maximum and minimum speed limits of the thief are $s_{\max}$ and $s_{\min}$ respectively with $s_{\max}\geq s_{\min}$. So for a TTP solution $\langle t, p\rangle$, the thief travels from city $t_k$ to $t_{k+1}$ with the knapsack weight $w_{t,p}(k)$ and with a travelling speed $s_{t,p}(k) = s(w_{t,p}(k))$. Moreover, the travelling time up to the position $k$ in the cyclic tour $t$ is $\tau_{t,p}(k) = \sum_{k'= 0}^{k'<k}d(t_{k'},t_{k'+1})/s_{t,p}(k')$ and the {\em total travelling time} is $T(t,p) = \tau_{t,p}(n) = \sum_{k = 0}^{k<n}d(t_{k},t_{k+1})/s_{t,p}(k)$. Hence, the {\em total renting cost} of the knapsack is $R(t,p) = R\times T(t,p)$, and so the {\em net profit} is $N(t,p) = P(p) - R(t,p)$. In TTP, we have to maximise the objective $N(t,p)$ over all possible cyclic tours $t$ and all possible collection plans $p$ subject to the knapsack constraint $K(p)\equiv W(p) \leq W$.

\begin{definition}[TTP]
Given a set $C$ of cities, a set $I$ of items, distance $d(c,c')$ between each pair of cities $c\ne c'$, weight $w_i$ and profit $\pi_i$ for each item $i$ available in city $l_i$, the knapsack capacity $W$, the knapsack renting rate $R$, a speed function $s(w)$ with $s_{\max}$ and $s_{\min}$ as the maximum and the minimum speeds respectively, find a solution $\langle t, p\rangle$ comprising a cyclic tour $t$ and a collection plan $p$ such that the objective $N(t,p)$ is maximised subject to the knapsack constraint $K(p)$.
\end{definition}

\figurename~\ref{fig:TTP} shows a TTP example, a solution, and the objective computation.

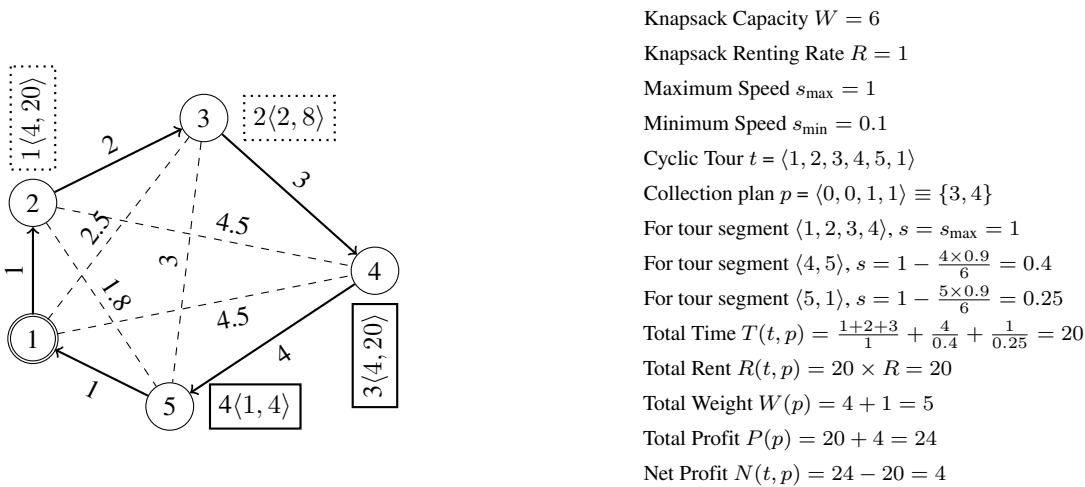
\begin{figure}[!b]
\begin{tabular}{lr}
\begin{minipage}{0.45\textwidth}
\begin{tikzpicture}[scale=0.9]
    \node[draw,circle,double] (1) at (0,0) {1};
    \node[draw,circle] (2) at (0,2) {2};
    \node[draw,circle] (3) at (2.5,3.25) {3};
    \node[draw,circle] (4) at (5,1) {4};
    \node[draw,circle] (5) at (2,-1) {5};

    \node[draw,thick,dotted,rotate=90] at ($(2) + (0,1.25)$) {$1\langle 4, 20\rangle$};
    \node[draw,thick,dotted] at ($(3) + (1.25,0)$) {$2\langle 2, 8\rangle$};
    \node[draw,thick,rotate=90] at ($(4) + (0,-1.25)$) {3$\langle 4, 20\rangle$};
    \node[draw,thick] at ($(5) + (1.25,0)$) {4$\langle 1, 4\rangle$};

    \draw[->,thick] (1) -- node [above,sloped] {1} (2);
    \draw[->,thick] (2) -- node [above,sloped] {2} (3);
    \draw[->,thick] (3) -- node [above,sloped] {3} (4);
    \draw[->,thick] (4) -- node [below,sloped] {4} (5);
    \draw[->,thick] (5) -- node [below,sloped] {1} (1);

    \draw[dashed] (1) -- node [above right,sloped,pos=0.33] {2.5} (3);
    \draw[dashed] (1) -- node [below right,sloped] {4.5} (4);
    \draw[dashed] (2) -- node [above right,sloped] {4.5} (4);
    \draw[dashed] (2) -- node [above,sloped] {1.8} (5);
    \draw[dashed] (3) -- node [above,sloped] {3} (5);
\end{tikzpicture}
\end{minipage}
&
\begin{minipage}{0.55\textwidth}
\begin{footnotesize}
Knapsack Capacity $W = 6$\\
Knapsack Renting Rate $R = 1$\\
Maximum Speed $s_\textrm{max} = 1$\\
Minimum Speed $s_\textrm{min} = 0.1$\\
Cyclic Tour $t$ = $\langle 1,2,3,4,5,1 \rangle$\\
Collection plan $p$ = $\langle 0, 0, 1, 1\rangle \equiv \{3, 4\}$\\
For tour segment $\langle 1, 2, 3, 4\rangle$, $s = s_\textrm{max} = 1$\\
For tour segment $\langle 4,5\rangle$, $s = 1 - \frac{4\times 0.9}{6} = 0.4$\\
For tour segment $\langle 5,1\rangle$, $s = 1 - \frac{5\times 0.9}{6} = 0.25$\\
Total Time $T(t,p) = \frac{1 + 2 + 3}{1} + \frac{4}{0.4} + \frac{1}{0.25} = 20$\\
Total Rent $R(t,p) = 20 \times R = 20$\\
Total Weight $W(p) = 4 + 1 = 5$\\
Total Profit $P(p) = 20 + 4 = 24$\\
Net Profit $N(t,p) = 24 - 20 = 4$
\end{footnotesize}
\end{minipage}
\end{tabular}
\caption{Left: a TTP instance having 5 cities (circles) and 4 items (rectangles) and a TTP solution with the travelled cyclic path (solid lines) and collected items (solid rectangles). Eacy city has a city index. City 1 (double circle) is the designated city to start from and end to. Each item has an item index and a tuple for weight and profit. Lines have distances as labels. Dashed lines are not in the travelled path and dotted rectangles are for items not collected. Right: other required parameters of the TTP instances along with the calculation  of the net profit for the TTP solution.\label{fig:TTP}}

\end{figure}

Although we do not claim any contribution, we prove relations of TTP with TSP and KP and show that TTP is NP-Hard. For this, we show how TSP and KP could be reduced to TTP. Note that there are many ways to get the reductions, but we just show one example for each case.

\begin{lemma}[TSP Reduction]\label{lem:TSPTTP}
Solving a TSP is equivalent to solving a TTP when for the speed function, $s_{\max} = s_{\min}$ and the knapsack weight capacity $W\geq \sum_{i=1}^{i=m}w_i$, i.e. the knapsack is sufficiently large to hold all items.
\end{lemma}

\begin{proof}
With $s_{\max} = s_{\min}$, the travelling speed becomes constant. With $W \geq \sum_{i=1}^{i=m}w_i$, all items must be collected for the maximum profit. So the collection plan $p$ has no impact on the cyclic tour $t$.
\end{proof}

\begin{lemma}[KP Reduction]\label{lem:KPTTP}
Solving a KP is equivalent to solving a TTP when distance $d(c,c')$ is the same for any two cities $c \ne c'$ and for the speed function, $s_{\max} = s_{\min}$ resulting into a constant speed.
\end{lemma}

\begin{proof}
When the distance $d(c,c')$ is the same for any two cities $c \ne c'$, and the speed is always a constant during the tour, the total travelling time is always the same. So the cyclic tour $t$ has no impact on the collection plan $p$.
\end{proof}

As per Lemmas~\ref{lem:TSPTTP} and \ref{lem:KPTTP}, TSP and KP are both special cases of TTP. So we have the following lemma, which is also mentioned in~\cite{mei2014improving}.

\begin{lemma}[TTP Complexity]
TTP is NP-Hard since TSP and KP are NP-Hard.
\end{lemma}

\ignore{ 
\begin{proof}
As per Lemmas~\ref{lem:TSPTTP} and \ref{lem:KPTTP}, TSP and KP are both special cases of TTP.
\end{proof}

}

We define TSP and KP components of a TTP when respectively the KP and the TSP components are left unchanged.

\begin{definition}[TSPC]\label{def:TSPC}
Given a TTP and a particular collection plan $p$, find a cyclic tour $t$ so that the TTP objective is maximised. This is the TSP component of TTP or in short TSPC.
\end{definition}

\begin{definition}[KPC]\label{def:KPC}
Given a TTP and a particular cyclic tour $t$, find a collection plan $p$ so that the TTP objective is maximised. This is the KP component of TTP or in short KPC.
\end{definition}

We also show that solving the TSP or KP component of a TTP is not equivalent to solving a standalone TSP or KP respectively.

\begin{lemma}[TSPC] 
TSPC is NP-Hard and is not equivalent to TSP.
\end{lemma}

\begin{proof}
For the first part: TSP easily reduces to TSPC with a constant speed function having $s_{\max} = s_{\min}$. For the second part: assume the speed depends on the knapsack weight. Although the collection plan is unchanged, reordering cities may also change the item collection order. This implies the travelling speed and the travelling time even between the same pair of cities might change. This means the collection plan could still affect exploration of cyclic tours in TSPC.
\end{proof}

The lemma below is provided in~\cite{polyakovskiy2017packing}.

\begin{lemma}[KPC] \label{lem:KPCComplexity}
KPC is NP-Hard and is not equivalent to KP.
\end{lemma}

\ignore{ 
\begin{proof}
For the first part: KP easily reduces to KPC when the distance and the speed between any two cities are the same. For the second part: assume the speed depends on the knapsack weight. Although the cyclic tour is unchanged, changing the collection plan might change the travelling speed and the travelling time even between the same pair of cities resulting into changes in the total renting cost. This means the cyclic tour could still affect exploration of collection plans in KPC.
\end{proof}

}

Above two lemmas show that just using standalone TSP and KP solvers to solve TSPC and KPC will not work. The reason is still the mutual interdependence of TSPC and KPC within TTP.

We adapt {2OPT} and {BitFlip} operators to TSPC and KPC respectively.

\begin{lemma}[{2OPT} for TSPC]
Given a TTP solution $\langle t, p\rangle$ with $w_{t,p}(k)$ and $\tau_{t,p}(k)$ for all positions, a {2OPT}$(t,b,e)$ operator produces a new cyclic tour $t'$ for which computing $N(t',p)$ needs $\mathcal{O}(e - b)$ time.
\end{lemma}

\begin{proof}
Given $w_{t,p}(k)$ and $\tau_{t,p}(k)$ for all positions in $t$, for each position $k \in [b,e+1]$ in $t'$ first we have to compute the knapsack weight $w_{t',p}(k)$, the travelling speed $s_{t',p}(k-1)$, and the travelling time up to each position $\tau_{t',p}(k)$. Then, the new total travelling time and the new objective values are computed as $T(t',p)=T(t,p)+\tau_{t',p}(e+1)-\tau_{t,p}(e+1)$ and $N(t',p) =P(p)-R \times T(t',p)$, respectively.
\end{proof}

\begin{lemma}[{BitFlip} for KPC]
Given a TTP solution $\langle t, p\rangle$ with $w_{t,p}(k)$ and $\tau_{t,p}(k)$ for all positions, {BitFlip}$(p,i)$ produces a new collection plan $p'$ for which computing $N(t,p')$ needs $\mathcal{O}(n - t(l_i))$ time.
\end{lemma}

\begin{proof}
In addition to computing $P(p')$ in $\mathcal{O}(1)$ time, for all positions $k\in[t(l_i),n-1]$ in $t$, we have to compute the knapsack weight $w_{t,p'}(k)$, the travelling speed $s_{t,p'}(k)$, and the travelling time up to each position $\tau_{t,p'}(k+1)$. Then, the new total travelling time and the new objective values are computed as $T(t,p')=\tau_{t,p'}(n)$ and $N(t,p') = P(p')-R \times T(t,p')$, respectively.
\end{proof}

\section{Related Work}
\label{sec:related}

TTP was introduced in~\cite{bonyadi2013travelling} and later many benchmark instances were given in~\cite{polyakovskiy2014comprehensive}. Depending on whether cities and items are dynamically made available for visiting or collection, TTP is of two types: dynamic and static. For dynamic TTP, we refer to a recent article in~\cite{sachdeva2020dynamic}. In this paper, we mainly deal with static TTP solving:  all cities must be visited and all items are available all the time. The thief decides whether particular cities are to be visited first or particular items are to be collected.
Existing TTP solvers can be grouped into 5~main categories: ({\it i})~constructive methods, ({\it ii})~fixed-tour methods, ({\it iii})~cooperative methods,
({\it iv})~full encoding methods, and ({\it v})~hyper-heuristic methods.
We give an overview of each category below. For further details, we also refer the reader to a recent review article \cite{wagner2017case}.

\subsection{Constructive Methods}

In constructive methods, an initial cyclic tour is generated (TSP) using the
Chained Lin-Kernighan heuristic~\cite{applegate2003chained}. The cyclic tour is then kept unchanged while collection plans are generated (KPC) using item scores based on their weight, profit, and position
in the cyclic tour. This category includes greedy approaches such as Simple Heuristic~\cite{polyakovskiy2014comprehensive}, Density-Based Heuristic~\cite{bonyadi2014socially}, Insertion~\cite{mei2014improving}
and PackIterative~\cite{faulkner2015approximate}. These approaches have been used in restart-based algorithms such as S5~\cite{faulkner2015approximate}
and in the initialisation phase of other methods.

\subsection{Fixed-Tour Methods}

In fixed-tour methods, after generating an initial cyclic tour (TSP) using constructive methods, an exact or an approximate method is used to find a collection plan (KPC). Exact methods \cite{wu2017exact,polyakovskiy2017packing} using dynamic programming or mixed integer programming approaches can find the best collection plan for every given cyclic tour. However, these methods can not solve large instances in a reasonable time. Approximate methods \cite{polyakovskiy2014comprehensive,faulkner2015approximate,maity2020efficient} iteratively improve the collection plan by using the {BitFlip} operator on one or more items in each iteration. Approximate methods can solve large instances in a reasonable time although they do not guarantee to find the best collection plan for a given cyclic tour.

\subsection{Cooperative Methods}

Cooperative methods are iterative approaches based on the cooperational coevolution approach \cite{potter1994cooperative}. After generating an initial TTP solution using a constructive or a fixed-tour method, the cyclic tours and the collection plans are explored by two separate search modules for TSPC and KPC. These two search modules are executed by a meta-optimiser in an interleaving fashion so that their interdependent nature is somewhat considered \cite{wagner2017case}. Some well-known cooperative methods are Cooperative Solver
(CoSolver)~\cite{bonyadi2014socially},
CoSolver with 2OPT and Simulated Annealing (CS2SA)~\cite{el2016population},
CS2SA with offline instance-based parameter tuning (CS2SA*)~\cite{el2018efficiently} and CoSolver with reverse order item selection (RWS)~\cite{zhang2021solving}.
Moreover, a surrogate assisted cooperative solver \cite{namazi2020surrogate} approximates the final TTP objective value for any given initial TSP tour without finding the final solution; based on the approximation,  non-promising initial solutions are discarded and thus more solutions are considered within a given time budget.

\subsection{Full-Encoding Methods}

Full-encoding methods consider the problem as a whole. Well-known full-encoding methods include a Memetic Algorithm with Two-stage Local Search (MATLS)~\cite{mei2014improving}, a Memetic algorithm with Edge-Assembly and 2-Points crossover operators (MEA2P)~\cite{wuijts2019investigation}, a swarm intelligence algorithm~\cite{wagner2016stealing} based on max-min ant system~\cite{stutzle2000max}, a memetic Algorithm with {2OPT} and {BitFlip} search~\cite{el2016population}, another memetic algorithm with joint {2OPT} and {BitFlip}~\cite{el2017local} such that 
 {BitFlip} is used on just one item each time a {2OPT} operator is used on cyclic tours, and an evolutionary algorithm using typical TSP and KP operators but maintaining quality solutions over epochs \cite{nikfarjam2022use}.

Overall full-encoding methods do not perform well beyond a few hundred cities and a few thousand items due to search space explosion.

\subsection{Hyper-Heuristic Methods}

In hyper-heuristic based methods, genetic programming (GP) is usually used to generate or select low level heuristics for cyclic tour or collection plan exploration. One GP method \cite{mei2015heuristic} generates two packing heuristics for collection plans. An individual in each generation is a tree with internal nodes being simple arithmetic operators and leaf nodes are numerical parameters of a given TTP. Other GP methods \cite{martins2017hseda,el2018hyperheuristic} learn how to select a sequence of low level heuristics for TSPC or KPC. Another GP method \cite{martins2017hseda} uses Baysian networks with low level heuristics as networks of individuals in each generation. Yet another GP method \cite{el2018hyperheuristic} has trees as individuals in each generation with internal nodes as functions and low level heuristics as leaf nodes. A recent random and reward based hyperheuristic method \cite{ali2020hyper} uses 23 operators and 4 ways to choose from 23 operators but evaluates the method only on 9 problem instances. Overall hyper-heuristic methods do not perform well beyond a few hundred cities and a few thousand items since the search space becomes very large.

\section{TTP Search Framework}
\label{sec:framework}

As we see from the TTP literature, existing TTP methods have very little to no explicit coordination between selection decisions made for cyclic tours and collection plans. In this paper, we propose coordination based methods for TTP. Our proposed approaches for TSPC select moves that explore cyclic tours and collection plans in a coordinated fashion and explicitly based on their potential mutual effects. Also, our proposed approach for KPC selects marginally profitable items to explore collection plans with respect to the cyclic tour selected earlier. We embed our coordination based approaches within {2OPT} and along with {BitFlip} operators to be used in exploring cyclic tours and collection plans. Our proposed coordination based approaches thus improve the effectiveness of the search for TTP solutions.

Note that our proposed approaches could also be viewed as cooperative approaches since our algorithm also does move between cyclic tour exploration and collection plan exploration in an interleaving fashion. Moreover, our proposed approaches for TSPC are also like full-encoding methods since they make changes to both cyclic tours and collection plans at the same time.

\begin{algorithm}[!b]
\caption{Main Function in the Proposed TTP Search Framework}
\label{algo:MainSearchFramework}
\begin{algorithmic}
\Function{{TTPS}}{$C,I,d, W, R, s,s_{\max},s_{\min}$}
\State {\bf parameters:}
\State \qquad $C$: a set of $n$ cities
\State \qquad $I$: a set of $m$ items
\State \qquad $d(c,c')$: distance between cities $c \ne c'$
\State \qquad $W$: knapsack weight capacity
\State \qquad $R$: knapsack renting rate
\State \qquad $s(w)$: speed function for weight $w$
\State \qquad $s_{\max}$: maximum speed
\State \qquad $s_{\min}$: minimum speed
\State {\bf returns:} $\langle t,p\rangle$: cyclic tour, collection plan
\State
\State $\langle t_\ast,p_*\rangle \gets \emptyset$ \LeftComment{best solution}
\While{not timeout} \LeftComment{restart in each lap}
    \State $t\gets \Call{{ChainedLinKernTour}}{~}$
    \State $p\gets \Call{{InitCollectionPlan}}{t}$
    \While{not timeout}
        \State $N_\textsf{BS}\gets N(t,p)$ \LeftComment{before search, in first iteration}
        \State $N_\textsf{BS}\gets N_\textsf{KP}$ \LeftComment{before search, in next iterations}
        \State $N_\textsf{BS}\gets N(t,p)$ \LeftComment{before search}      
        \State $\langle t,p\rangle \gets \Call{TSPS}{t,p}$
        \State $N_\textsf{TSP}\gets N(t,p)$ \LeftComment{after TSP search}
		\State $p \gets \Call{KPS}{t,p,1,n-1}$
        \State $N_\textsf{KP}\gets N(t,p)$ \LeftComment{after KP search}
        \If{$N_\textsf{BS}=N_\textsf{KP}$}
            \State {\bf break}
        \EndIf
    \EndWhile
    \If{$N(t,p) > N(t_*,p_*)$}
    \State $\langle t_*,p_*\rangle \gets \langle t,p\rangle$ \LeftComment{new best solution}
    \EndIf
\EndWhile
\State \Return {$\langle t_*, p_*\rangle$}
\EndFunction
\end{algorithmic}%
\end{algorithm}

Algorithms~\ref{algo:MainSearchFramework} and \ref{algo:AuxSearchFramework} describe the TTP search framework that we use in evaluating our proposed coordination based approaches. The search framework is similar to the cooperational coevolution approach~\cite{potter1994cooperative,el2018efficiently}. It has three main functions: {TTPS}, {TSPS}, and {KPS}. The framework allows customisation of its various parts to facilitate development of TTP search methods with or without coordination.

Below we list the abbreviations used in the proposed search framework.

~

\begin{tabular}{rl}
\textbf{TTPS  } & The main TTP search function in Algorithm~\ref{algo:MainSearchFramework} \\
\textbf{TSPS  } & The TSP component search function in Algorithm~\ref{algo:AuxSearchFramework} \\
\textbf{KPS   } & The KP component search function in Algorithm~\ref{algo:AuxSearchFramework} \\
\textbf{NOCH  } & No coordination heuristic in Section~\ref{subsec:baseline} \\
\textbf{SGCH  } & Search guided coordination heuristic in Section~\ref{subsec:sgch} \\
\textbf{PGCH  } & Profit guided coordination heuristic in Section~\ref{subsec:pgch} \\
\textbf{CISH  } & Coordinated item selection heuristic in Section~\ref{subsec:cish} \\
\textbf{IPR   } & Item profitability ratio defined in Section~\ref{subsec:character} \\
\textbf{LCIPR } & The lowest collected item profitability ratio defined in Section~\ref{subsec:character} \\
\textbf{HUIPR } & The highest uncollected item profitability ratio defined in Section~\ref{subsec:character} \\
\textbf{SBFS } & Standard bit-flip search in Section~\ref{subsec:baseline} \\
\textbf{MBFS } & Marginal bit-flip search in Section~\ref{subsec:cish} \\
\textbf{NLBC } & Non-linear binary classifer in Section~\ref{subsec:lgch} \\ 
\textbf{LGCH } & Learning guided coordination heuristic in Section~\ref{subsec:lgch}
\end{tabular}

\subsection{Function {TTPS}}

Function~{TTPS} in Algorithm~\ref{algo:MainSearchFramework} has two loops, one inside another. The outer loop runs for a given timeout limit. Each iteration of the outer loop is a restarting of the search from scratch. Inside the outer loop, first an initial cyclic tour $t$ and an initial collection plan $p$ for $t$ are generated. Function~{ChainedLinKernTour} generates the initial cyclic tour using Chained Lin-Kernighan heuristic~\cite{applegate2003chained}. Then, Function~{InitCollectionPlan} generates the initial collection plan taking the best of the solutions returned by PackIterative~\cite{faulkner2015approximate} and Insertion~\cite{mei2014improving} methods. Once a complete TTP solution $\langle t, p\rangle$ is thus obtained, the inner loop then refines that solution in an iterative fashion. In each iteration of the inner loop, Functions {TSPS} and {KPS} are invoked in an interleaving fashion to improve the cyclic tour and the collection plan. The inner loop terminates when the objective value does not change between two successive iterations.

\begin{algorithm}[!tb]
\caption{Other Functions in the Proposed TTP Search Framework}
\label{algo:AuxSearchFramework}
\begin{algorithmic}
\Function{{TSPS}}{$t,p$}
\State $\langle t_\diamond,p_\diamond\rangle \gets \langle t,p\rangle$ \LeftComment{best solution}
\Repeat{} \LeftComment{main loop}
    \State $N' \gets N(t,p)$
    \For{$b \gets 1$ {\bf to} $n-2$}
        \For{{\bf each} $t_{e} \in \textsf{DelaTriNeighb}[t_{b}]$}
            \If{$b < e < n$}
                \State $t' \gets \Call{2OPT}{t,b,e}$
                \State$p' \gets\!$ \Call{{CoordHeu}}{$t,p,t',b,e$}
                \If{$N(t',p')>N(t_\diamond,p_\diamond)$}
                    \State $\langle t_\diamond,p_\diamond\rangle \gets \langle t',p'\rangle$
                \EndIf
            \EndIf
        \EndFor
    \EndFor
    \State $\langle t,p\rangle \gets \langle t_\diamond,p_\diamond\rangle$
\Until{$N(t,p) < N'\cdot(1 + \alpha)$} \LeftComment{$\alpha = 0.01\%$}
\State \Return $\langle t, p\rangle$
\EndFunction
\State
\Function{{KPS}}{$t,p,b,e$}
    \State $I' \gets$ \Call{{SelectItemsSubset}}{$t, p,b,e$}    \State \Call{{MarkAllItemsUnchecked}}{$I'$}
    \While{$\lnot \Call{{AllItemsChecked}}{I'}$}
        \State $i \gets \Call{{RandomUncheckedItem}}{I'}$
        \State \Call{{MarkItemChecked}}{$i$}
        \State $p' \gets \Call{{BitFlip}}{p,i}$ {\bf when} $K(p')$
        \If{$N(t,p') > N(t,p)$}
            \State $ p \gets p'$
            \State $I' \gets$ \Call{{SelectItemsSubset}}{$t, p,b, e$}
            \State \Call{{MarkAllItemsUnchecked}}{$I'$}
        \EndIf
    \EndWhile
    \State \Return $p$
\EndFunction
\end{algorithmic}%
\end{algorithm}

\subsection{Function {TSPS}}

Function~{TSPS} in Algorithm~\ref{algo:AuxSearchFramework} is a steepest ascent hill-climbing method. Inside the main loop, from the current solution $\langle t, p\rangle$, a new solution $\langle t', p'\rangle$ is generated using the  neighbourhood operator {2OPT} and the coordination function {CoordHeu} for each tour segment $t[b,e]$, where $b \in [1,n-2]$ and $t_e$ is in the precomputed Delaunay triangulation~\cite{delaunay1934sphare} neighbourhood \textsf{DelaTriNeighb} array of $t_{b}$. The best solution among the newly generated solutions that are better than the current solution is accepted as the current solution for the next iteration of the main loop. Note that the main loop of each invocation of the function continues as long as the improvement in the objective value is at least $\alpha\%$ with respect to the objective value computed at the starting of the loop \cite{dueck1993new}. Here, $\alpha$ essentially controls when to switch from the TSP component to the KP component. After initial experiments, we set $\alpha = 0.01$.

Notice that in Function {TSPS}, after calling the operator {2OPT}, there is a calling of the coordination function {CoordHeu}. We know Operator {2OPT} makes changes only to the cyclic tour. When no change in the collection plan is sought after Function~{2OPT}, Function~{CoordHeu} is defined to be returning just $p$ as $p'$. However, in this paper, considering coordination between TSP and KP components, we design alternative coordination functions to be used as Function~{CoordHeu}. We later describe the alternative functions.

\subsection{Function {KPS}}

Function~{KPS} in Algorithm~\ref{algo:AuxSearchFramework} starts with an initial subset $I'$ of items selected by Function~{SelectItemsSubset} based on a given tour segment $t[b,e]$ in a solution $\langle t, p\rangle$. The loop in Function~{KPS} runs until for all of the items in $I'$, {BitFlip} has been applied without any improvement in the objective, since the latest change in the collection plan. In each iteration of the loop, one previously unchecked item $i$ from $I'$ is randomly checked and $p_i$ is flipped using {BitFlip}($p,i$). The change in $p_i$ is accepted if it improves the objective. Note that every time a change in $p$ is thus accepted, $I'$ is computed again by Function~{SelectItemsSubset} and all items in the new $I'$ are marked unchecked. This in essence restarts the KP search within the same loop. Functions {MarkAllItemsUnchecked}, {AllItemsChecked}, {RandomUncheckedItem}, {MarkItemChecked} are respectively for marking all items in $I'$ unchecked, testing whether all items in $I'$ are checked already, selecting an unchecked item $i$ from $I'$ randomly, and marking a selected item $i$ as checked.

For Function~{SelectItemsSubset}, we could typically use all items in the given tour segment $t[b,e]$, or just some of them. Considering coordination between the TSP and the KP component, in this paper, we later propose strategies to select a subset of items from a tour segment.

\subsection{Baseline Solver Version}
\label{subsec:baseline}

Our baseline TTP solver has no explicit coordination between TSP and KP components. As shown in Algorithm~\ref{algo:Baseline}, for Function {CoordHeu},  we use Function \Call{NoCoordHeu}{$t,p,t',b,e$} that just returns $p$ making no change at all, and for Function $\Call{SelectItemsSubset}{t,p,b,e}$, we use Function $\Call{SelectTourSegmentItems}{t,p,b,e}$ that just returns $I(t[b,e])$, i.e. the set of all items available in the tour segment $t[b,e]$.
For convenience, in discussing the experimental results, we denote the approach using Function~{NoCoordHeu} by NOCH. Note that Function~{TSPS} with Function~{NoCoordHeu} is almost the same as the method used for solving the TSP component in~\cite{el2018efficiently}. Also, note that Function~{KPS} with {SelectTourSegmentItems} is called the {\em standard bit-flip search} (SBFS) \cite{polyakovskiy2014comprehensive,faulkner2015approximate} algorithm for solving the KP component in TTP.

\begin{algorithm}[!tb]
\caption{Implementing Baseline Solver on the Search Framework}
\label{algo:Baseline}
\begin{algorithmic}
    \Function{NoCoordHeu}{$t,p,t',b,e$}
        \State\LeftComment{defines \Call{CoordHeu}{$t,p,t',b,e$}}
        \State \Return $p$s
    \EndFunction
    \Function{SelectTourSegmentItems}{$t,p,b,e$}
        \State\LeftComment{defines \Call{SelectItemsSubset}{$t,p,b,e$}}
        \State \Return $I(t[b,e])$
    \EndFunction
\end{algorithmic}
\end{algorithm}

\section{Proposed Coordination Approaches}
\label{sec:proposed}

We give a motivating example to show how coordination helps evaluate a cyclic tour better in TTP. We also characterise Operator~{2OPT} to find the reasons behind its poor coordination behaviour. We develop our coordination based heuristics for TTP on top of the search framework in Algorithms~\ref{algo:MainSearchFramework} and \ref{algo:AuxSearchFramework}. We develop three alternative approaches to be used within Function~{TSPS} and one alternative approach to be used within Function~{KPS}. The three coordination approaches to be used to define Function {CoordHeu} within Function~{TSPS} are local search based, human designed intuitive, and machine learning models. The other coordination approach to be used within Function~{KPS} is a strategy to select items by Function {SelectItemsSubset}.

\subsection{Observing Coordination Effect after {2OPT}}

In Function~{TTPS} in Algorithm~\ref{algo:MainSearchFramework}, Function~{TSPS} and Function~{KPS} are invoked in an interleaving fashion. In the baseline algorithm in Section~\ref{subsec:baseline}, after Operator~{2OPT} is called in Function~{TSPS}, Function~{NoCoordHeu} is used for Function~{CoordHeu}. This means no change in collection plan is made after changing the cyclic tour. The example below shows such an approach results in incorrect or misleading evaluations of the TTP solutions by Function {TSPS}.

\begin{figure}[!b]
\begin{tabular}{lr}
\begin{minipage}{0.47\textwidth}
\begin{tikzpicture}[scale=1]
    \node[draw,circle,double] (1) at (0,0) {1};
    \node[draw,circle] (2) at (0,2) {2};
    \node[draw,circle] (3) at (2.5,3.25) {3};
    \node[draw,circle] (4) at (5,1) {4};
    \node[draw,circle] (5) at (2,-1) {5};

    \node[draw,thick,dotted,rotate=90] at ($(2) + (0,1.00)$) {\footnotesize$1\langle 4, 20\rangle$};
    \node[draw,thick,dotted] at ($(3) + (1,0)$) {\footnotesize $2\langle 2, 8\rangle$};
    \node[draw,thick,rotate=90] at ($(4) + (0,-1.00)$) {\footnotesize $3\langle 4, 20\rangle$};
    \node[draw,thick] at ($(5) + (1,0)$) {\footnotesize $4\langle 1, 4\rangle$};

    \draw[dashed] (1) -- node [above,sloped] {\footnotesize 1} (2);
    \draw[<-,thick] (2) -- node [above,sloped] {\footnotesize 2} (3);
    \draw[<-,thick] (3) -- node [above,sloped] {\footnotesize 3} (4);
    \draw[dashed] (4) -- node [below,sloped] {\footnotesize 4} (5);
    \draw[->,thick] (5) -- node [below,sloped] {\footnotesize 1} (1);

    \draw[dashed] (1) -- node [above right,sloped,pos=0.33] {\footnotesize 2.5} (3);
    \draw[->,thick] (1) -- node [below right,sloped] {\footnotesize 4.5} (4);
    \draw[dashed] (2) -- node [above right,sloped] {\footnotesize 4.5} (4);
    \draw[->,thick] (2) -- node [above,sloped] {\footnotesize 1.8} (5);
    \draw[dashed] (3) -- node [above,sloped] {\footnotesize 3} (5);
\end{tikzpicture}
\end{minipage}
&
\begin{minipage}{0.47\textwidth}
\begin{tikzpicture}[scale=1]
    \node[draw,circle,double] (1) at (0,0) {1};
    \node[draw,circle] (2) at (0,2) {2};
    \node[draw,circle] (3) at (2.5,3.25) {3};
    \node[draw,circle] (4) at (5,1) {4};
    \node[draw,circle] (5) at (2,-1) {5};

    \node[draw,thick,rotate=90] at ($(2) + (0,1.00)$) {\footnotesize$1\langle 4, 20\rangle$};
    \node[draw,thick,dotted] at ($(3) + (1,0)$) {\footnotesize $2\langle 2, 8\rangle$};
    \node[draw,thick,dotted,rotate=90] at ($(4) + (0,-1.00)$) {\footnotesize $3\langle 4, 20\rangle$};
    \node[draw,thick] at ($(5) + (1,0)$) {\footnotesize $4\langle 1, 4\rangle$};

    \draw[dashed] (1) -- node [above,sloped] {\footnotesize 1} (2);
    \draw[<-,thick] (2) -- node [above,sloped] {\footnotesize 2} (3);
    \draw[<-,thick] (3) -- node [above,sloped] {\footnotesize 3} (4);
    \draw[dashed] (4) -- node [below,sloped] {\footnotesize 4} (5);
    \draw[->,thick] (5) -- node [below,sloped] {\footnotesize 1} (1);

    \draw[dashed] (1) -- node [above right,sloped,pos=0.33] {\footnotesize 2.5} (3);
    \draw[->,thick] (1) -- node [below right,sloped] {\footnotesize 4.5} (4);
    \draw[dashed] (2) -- node [above right,sloped] {\footnotesize 4.5} (4);
    \draw[->,thick] (2) -- node [above,sloped] {\footnotesize 1.8} (5);
    \draw[dashed] (3) -- node [above,sloped] {\footnotesize 3} (5);
\end{tikzpicture}
\end{minipage}
\\
\\
\begin{minipage}{0.47\textwidth}
\begin{footnotesize}
Without Coordination\\
$t'$ = $\langle 1,4,3,2,5,1 \rangle$, $p$ = $\langle 0, 0, 1, 1\rangle$ = $\{3, 4\}$\\
For tour segment $\langle 1, 4\rangle$, $s = s_\textrm{max} = 1$\\
For tour seg $\langle 4,3,2,5\rangle$, $s = 1 - \frac{4\times 0.9}{6} = 0.4$\\
For tour seg $\langle 5,1\rangle$, $s = 1 - \frac{5\times 0.9}{6} = 0.25$\\
$T(t',p) = \frac{4.5}{1} + \frac{3 + 2 + 1.8}{0.4} + \frac{1}{0.25} =25.5$\\
$N(t',p) = 24 - 25.5 = -1.5$
\end{footnotesize}
\end{minipage}
&
\begin{minipage}{0.47\textwidth}
\begin{footnotesize}
With Coordination\\
$t'$ = $\langle 1,4,3,2,5,1 \rangle$, $p'$ = $\langle 1, 0, 0, 1\rangle$ = $\{1, 4\}$\\
For tour segment $\langle 1, 4, 3, 2\rangle$, $s = s_\textrm{max} = 1$\\
For tour segment $\langle 2,5\rangle$, $s = 1 - \frac{4\times 0.9}{6} = 0.4$\\
For tour segment $\langle 5,1\rangle$, $s = 1 - \frac{5\times 0.9}{6} = 0.25$\\
$T(t',p') = \frac{4.5 + 3 + 2}{1} + \frac{1.8}{0.4} + \frac{1}{0.25} =18$\\
$N(t',p') = 24 - 18 = 6$
\end{footnotesize}
\end{minipage}
\end{tabular}
\caption{From the scenario in \figurename~\ref{fig:TTP} with $t = \langle 1,2,3,4,5,1\rangle$, $p = \langle 0,0,1,1\rangle$, and $N(t,p)= 4$, (left) only {2OPT} is applied on $t$ to get $t' = \langle 1,4,3,2,5,1\rangle$ and (right) after {2OPT} is applied, $p$ is also changed to $p'=\langle 1,0,0,1\rangle$.  \label{fig:coordination}}

\end{figure}

Consider the TTP example in \figurename~\ref{fig:TTP} and the solution comprising cyclic tour $t = \langle 1, 2, 3, 4, 5, 1\rangle$ and collection plan $p = \langle 0,0,1,1\rangle$ having the objective value $N(t,p) = 4$. When Operator $\Call{2OPT}{t,1,3}$ is applied on the cyclic tour $t$ to reverse the tour segment $\langle 2,3,4\rangle$ to $\langle 4,3,2\rangle$, the resultant cyclic tour is $t' = \langle 1,4,3,2,5,1\rangle$. \figurename~\ref{fig:coordination} (left) shows that if the collection plan $p$ is not changed when $t$ changes to $t'$, the objective value $N(t',p) = -1.5$ is used to evaluate $t'$. In this case, there is no explicit coordination between the cyclic tour and the collection plan. Then, \figurename~\ref{fig:coordination} (right) shows that if the collection plan $p$ is also changed to $p' = \langle 1,0,0,1\rangle$ after $t$ is changed to $t'$, the objective value $N(t',p') = 6$ is used to evaluate $t'$. There is coordination here between the cyclic tour and the collection plan.
This example clearly shows that the potential of a cyclic tour is better reflected when the collection plan is also adjusted with the cyclic tour and so coordination is needed. To have a clearer perspective, with $N(t',p) = -1.5$, the resultant tour $t'$ could be easily rejected during search while with $N(t',p') = 6$, the same resultant tour $t'$ could be easily accepted. With an interleaving approach of invoking Functions~{TSPS} and {KPS}, existing TTP methods thus do not properly evaluate generated TTP solutions and thus suffer from not having a proper search direction. In this paper, we argue that for better coordination between two TTP components, the quality of each cyclic tour or the collection plan should be evaluated along with the best possible corresponding collection plan or the cyclic tour and not against only the current collection plan or cyclic tour. Our arguments above are equally applicable to both TTP components. However, in this paper, we mainly evaluate cyclic tours against the best possible collection plans. This is because the 2OPT operator used in Function~{TSPS} can make changes to many cities in large tour segments while {BitFlip} operator used in Function~{KPS} makes changes to only one item in the collection plan, and we put more emphasis on the large changes. For an operator making huge change in Function~{KPS}, one could also evaluate  collection plans against the best possible cyclic tours.

Below we define the quality of a cyclic tour and prove its time complexity.

\begin{definition}[Cyclic Tour Quality] The quality $Q(t) = \max_pN(t,p)$ of a cyclic tour $t$ in TTP is the maximum objective value $N(t,p)$ over all possible collection plans $p$.
\end{definition}

\begin{lemma}[Cyclic Tour Quality]
Computing quality $Q(t)$ for a cyclic tour $t$ in TTP is NP-Hard.
\end{lemma}

\begin{proof}
Computing $Q(t)$ is essentially solving KPC and so is NP-Hard as per Lemma~\ref{lem:KPCComplexity}.
\end{proof}

From the above lemma, it is clear that invoking a {\em complete search} for collection plans for each and every tour segment reversal for a given cyclic tour is not feasible within a given timeout limit. So we need an incomplete search or a heuristic.

\subsection{Local Search Based Coordination}\label{subsec:sgch}

Since computing $Q(t)$ for $t$ is very hard, as shown above, we want to obtain an estimate of $Q(t)$ for a given $t$. For this, in this paper, we propose to invoke a local search based incomplete approach. We name our proposed approach as Search Guided Coordination Heuristic (SGCH) for TTP.

\paragraph{SGCH Implementation} Algorithm~\ref{algo:SGCH} shows the implementation of our proposed SGCH approach on top of the search framework. For Function~{SelectItemsSubset} in Function~{KPS}, we define Function~{SelectTourSegmentItems} to be returning $I(t[b,e])$ and
for Function~{CoordHeu} in Function~{TSPS}, we define Function~{SearchGuidedCoordHeu} to be returning the collection plan produced by Function~{KPS} by exploring items $I(t'[b,e])$. Notice that Function~{KPS} is called twice for SGCH with the same definition of Function~{SelectTourSegmentItems}: once in Function~{TTPS} for $t[1,n-1]$, i.e. for the entire tour, and again in Function~{SearchGuidedCoordHeu} called from Function~{TSPS} for each reversed tour segment $t'[b,e]$.

\begin{algorithm}[!tb]
\caption{Implementing SGCH on the Search Framework}
\label{algo:SGCH}
\begin{algorithmic}
    \Function{SearchGuidedCoordHeu}{$t,p,t',b,e$}
        \State\LeftComment{defines \Call{CoordHeu}{$t,p,t',b,e$}}
        \State \Return \Call{KPS}{$t',p,b,e$}
    \EndFunction
    \Function{SelectTourSegmentItems}{$t,p,b,e$}
        \State\LeftComment{defines \Call{SelectItemsSubset}{$t,p,b,e$}}        \State \Return $I(t[b,e])$
    \EndFunction
\end{algorithmic}
\end{algorithm}

\subsection{Characterising {2OPT} Coordination Behaviour}
\label{subsec:character}

We characterise the coordination behaviour of Operator~{2OPT} using item profitability ratio (IPR).  Below we formally define IPR.

\begin{definition}[Item Profitability Ratio] For an item $i$, the {\em item profitability ratio} (IPR) {$r_i = \pi_i/w_i$}.
An item {$i$} is {\em more profitable} than item {$i'$}, if {$r_i > r_{i'}$}, or if {$\pi_i > \pi_{i'}$} when {$r_i = r_{i'}$}.
\end{definition}

Greedy constructive KP heuristics typically collect items in non-increasing order of IPR. Constructive TTP heuristics also exhibit similar trends. To describe these trends, we need two functions. For a given sequence of numbers and a given position, one of the functions return the smallest number from the beginning up to the given position while the other returns the largest number from the ending down to the given position.

\begin{definition}[Prefix Minimum]\label{def:prefixmin}
Given a sequence $s=\langle s_1, s_2, \ldots, s_n\rangle$ of $n$ numbers $s_k$ with $k\in [1,n]$, the {\em prefix minimum} function $\Pi$ is defined by $\Pi(s, k) = \min(\Pi(s, k-1), s_k)$ when $1<k\leq n$ and $\Pi(s,1) = s_1$. Using the definition, we also get the {\em prefix minimum sequence} $s'=\Pi(s) = \langle \Pi(s,1), \Pi(s,2), \ldots, \Pi(s,n) \rangle$ for a given sequence $s$ in $O(n)$ time. For example, the prefix minimum sequence $s'$ is $\Pi(s)=\langle 9, 6, 6, 4, 4, 4\rangle$ when the given sequence $s$ is $\langle 9, 6, 8, 4, 5, 7\rangle$.
\end{definition}

\begin{definition}[Suffix Maximum]\label{def:suffixmax}
Given a sequence $s=\langle s_1, s_2, \ldots, s_n\rangle$ of $n$ numbers $s_k$ with $k\in [1,n]$, the {\em suffix maximum} function $\Omega$ is defined by $\Omega(s, k) = \max(s_k,\Omega(s, k+1))$ when $1\leq k < n$ and $\Omega(s,n) = s_n$. Using the definition, we also get the {\em suffix maximum sequence} $s''=\Omega(s) = \langle \Omega(s,1), \Omega(s,2), \ldots, \Omega(s,n) \rangle$ for a given sequence $s$ in $O(n)$ time. For example, the suffix maximum sequence $s''$ is $\Omega(s)=\langle 9, 8, 8, 7, 7, 7\rangle$ when the given sequence $s$ is $\langle 9, 6, 8, 4, 5, 7\rangle$.
\end{definition}

\paragraph{IPR Trends in TTP} As is already mentioned above, constructive KP heuristics exhibit a non-increasing trend in IPR. In TTP, items are scattered over cities and item collection order is restricted by city visiting order in the cyclic tour. Therefore, constructive greedy TTP methods such as PackIterative~\cite{faulkner2015approximate} and Insertion~\cite{mei2014improving}
use IPRs along with distances of the respective cities from the end of the cyclic tour in constructing the collection plan. So a monotonous non-increasing trend in IPRs of collected items is not expected in TTP solutions. {\em However, given a cyclic tour, within each city, we can reasonably expect items are collected in non-increasing order of IPR, unless there is not enough space for a highly profitable but heavy item}. This could be a key guideline to get collection plans in TTP. Below we define the lowest collected IPR (LCIPR) and the highest uncollected IPR (HUIPR) for each city in a TTP solution.

\begin{definition}[Lowest Collected IPR] Given a TTP solution $\langle t, p\rangle$, for a city $t_k$ at position $k$ in the cyclic tour $t$, the lowest collected IPR is $L(t,p,k) = \min_{i \in I(t_k) \land p_i = 1} r_i$. We then define a series of LCIPRs as $L(t,p) = \langle L_1, L_2,\ldots,L_{n-1}\rangle$ where $L_k = L(t,p,k)$. Using Definition~\ref{def:prefixmin}, we further define a prefix minimum function $\Pi(L(t,p), k)$ and a prefix minimum sequence $\Pi(L(t,p))$.
\end{definition}

\begin{definition}[Highest Uncollected IPR] Given a TTP solution $\langle t, p\rangle$, for a city $t_k$ at position $k$ in the cyclic tour $t$, the highest uncollected IPR is $H(t,p,k) = \max_{i \in I(t_k) \land p_i = 0} r_i$. We then define a series of HUIPRs as $H(t,p) = \langle H_1, H_2,\ldots,H_{n-1}\rangle$ where $H_k = H(t,p,k)$. Using Definition~\ref{def:suffixmax}, we further define a suffix maximum function $\Omega(H(t,p),k)$ and a suffix maximum sequence $\Omega(H(t,p))$.
\end{definition}

\paragraph{LCIPR and HUIPR Trends in TTP} \figurename~\ref{Eil76_1} (left) shows the LCIPR sequence $L(t,p)$ and the HUIPR sequence $H(t,p)$ for a TTP solution returned by PackIterative~\cite{faulkner2015approximate} for a benchmark instance {\it eil76\_n750\_uncorr\_10.ttp}. Clearly, $L(t,p)$ and $H(t,p)$ do not exhibit any monotonous trends. However, $L(t,p)$ does exhibit an overall decreasing trend from city positions low to high and $H(t,p)$ does exhibit an overall increasing trend from city positions high to low. Notice that the prefix minimum sequence $\Pi(L(t,p))$ and the suffix maximum sequence $\Omega(H(t,p))$ in effect capture the two overall trends respectively. Moreover, both of these two trend lines are monotonous, although one is in the forward direction and the other is in the backward direction.

\begin{figure*}[!b]
  \centering
  \begin{minipage}{0.49\textwidth}
    \centering
    \includegraphics[width=\textwidth]{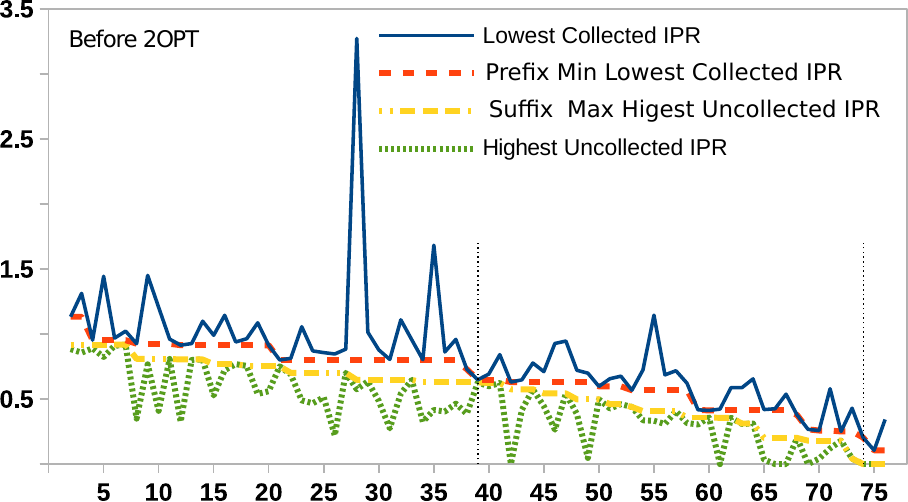}
  \end{minipage}
  \hfill
  \begin{minipage}{0.49\textwidth}
    \centering
    \includegraphics[width=\textwidth]{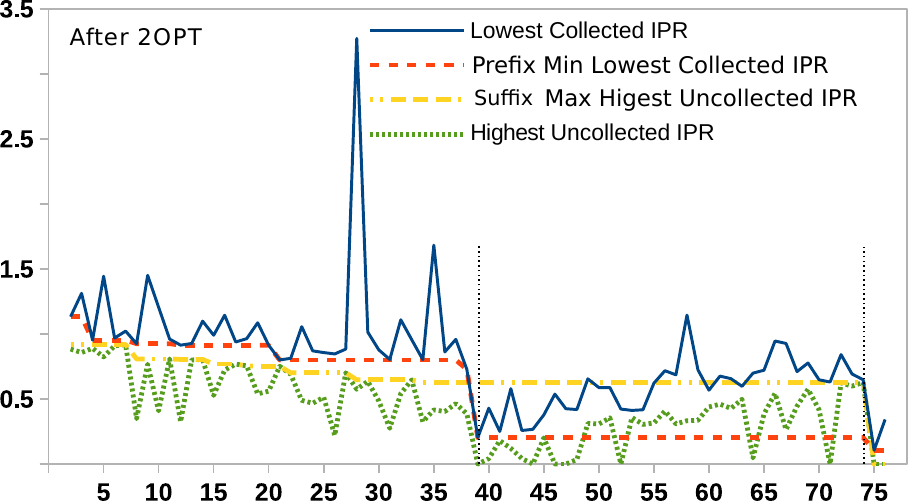}
  \end{minipage}
  \caption
    {City position in a tour (x-axis) vs IPR (y-axis) for (left) a TTP solution with objective value {$77544.88$} as obtained by the PackIterative method for a benchmark instance {\it eil76\_n750\_uncorr\_10.ttp} and for (right)
    the solution with objective value {$72151.46$} as obtained after the application of operator {2OPT} on the PackIterative generated solution for the same TTP instance on the cities between positions 39 and 74, keeping the collection plan unchanged meaning considering no coordination.}
  \label{Eil76_1}
  
\end{figure*}

\paragraph{Disruptions in Trends by {2OPT}} On the TTP solution $\langle t,p\rangle$ shown in \figurename~\ref{Eil76_1} (left), if we apply operator $\Call{2OPT}{t, 39, 74}$ to reverse the cities in the tour segment between positions $39$ and $74$ keeping the collection plan $p$ unchanged, we get a new solution $\langle t',p\rangle$ that is shown in \figurename~\ref{Eil76_1} (right). Notice that the {2OPT} operator affects the prefix minimum sequence $\Pi(L(t',p))$ and the suffix maximum sequence $\Omega(H(t',p))$ in the reversed tour segment $t'$ for $39 \leq k \leq 74$. Further notice that the objective value 72151.46 of the resultant solution $\langle t',p\rangle$ is smaller than the objective value 77544.88 of the solution $\langle t,p\rangle$. This means the resultant solutions in such cases would be mostly rejected by the search algorithm. The degradation of the objective value by the {2OPT} operator is because in the resultant tour, less profitable items are collected in the cities furthest from the end of the tour causing more travelling time. As shown before in \figurename~\ref{Eil76_1} (left), this was not the case in the solution before application of the {2OPT} operator. So this is somewhat clear that the {2OPT} operator results in deviation from the typical trends of $\Pi(L(t,p))$ and $\Omega(H(t,p))$ in TTP.

\subsection{Human Designed Intuitive Coordination}\label{subsec:pgch}

Although the local search based approach mentioned before is a way to obtain an estimate of quality $Q(t')$ for a generated cyclic tour $t'$, invoking the local search method for each generated cyclic tour $t'$ would be costlier.
In this paper, we design a heuristic approach to obtain a modified collection plan $p'$ that is used to get the estimated quality value for the generated cyclic tour $t'$.
We name our proposed approach as Profit Guided Coordination Heuristic (PGCH) for TTP.
The proposed approach aims to fix the disruptions in the trends in the resulting collection plan produced by Operator {2OPT}.

\begin{figure*}[!b]
  \centering
  \begin{minipage}{0.49\textwidth}
    \centering
    \includegraphics[width=\textwidth]{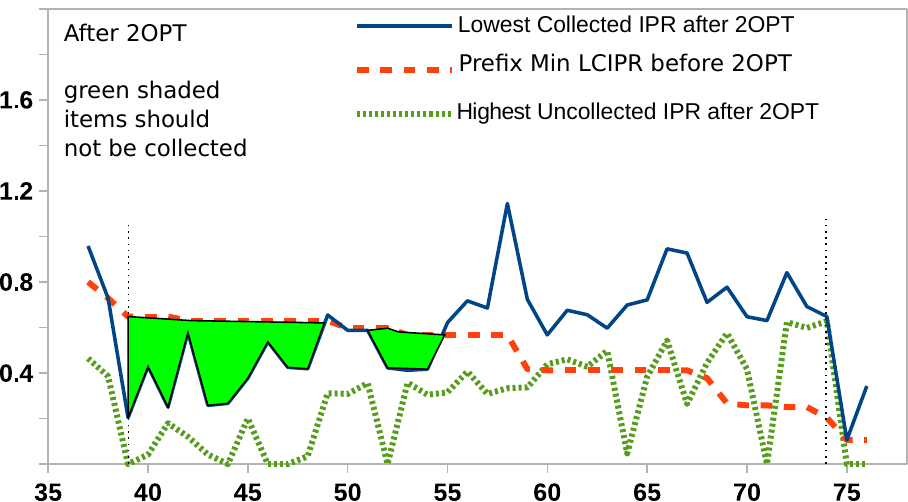}
  \end{minipage}
  \hfill
  \begin{minipage}{0.49\textwidth}
    \centering
    \includegraphics[width=\textwidth]{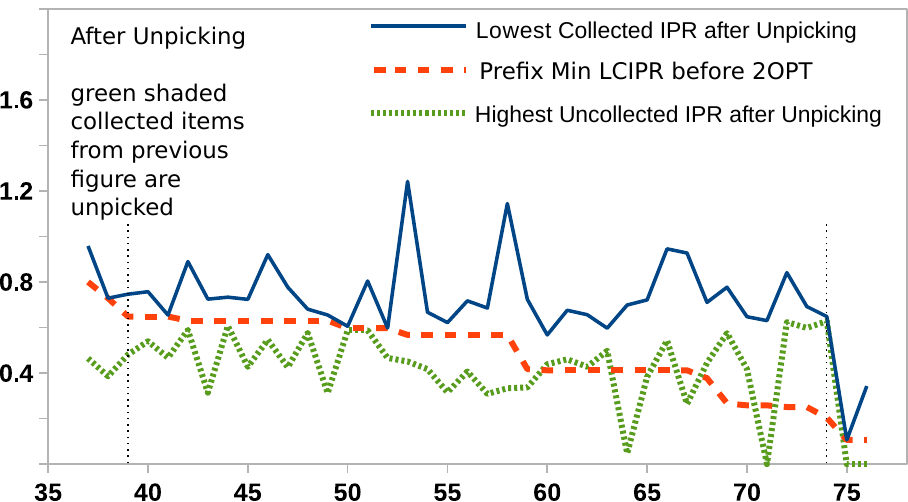}
  \end{minipage}
  \\[1em]
  \begin{minipage}{0.49\textwidth}
    \centering
    \includegraphics[width=\textwidth]{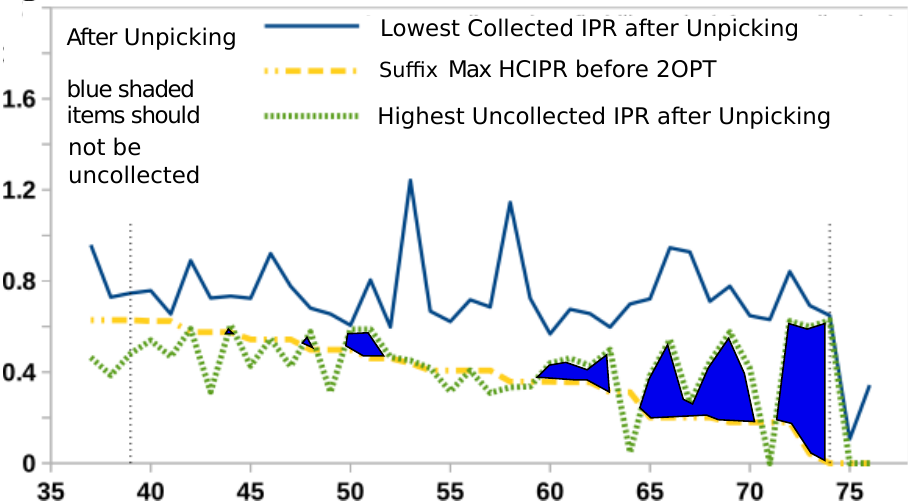}
  \end{minipage}
  \hfill
  \begin{minipage}{0.49\textwidth}
    \centering
    \includegraphics[width=\textwidth]{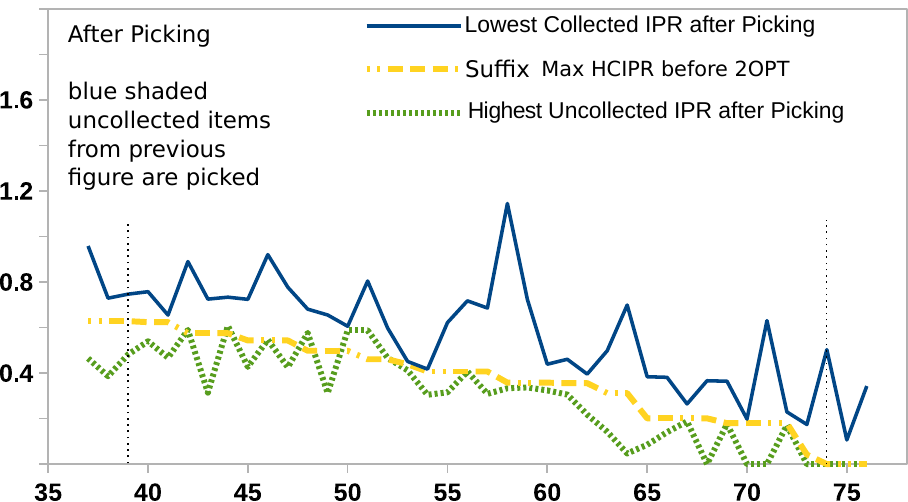}
  \end{minipage}
  \caption
    { City positions in a tour (x-axis) vs IPR (y-axis) when items in the solution shown in \figurename~\ref{Eil76_1} (right) are picked or unpicked. Top-Left: green shaded regions denote collected items that should be unpicked. Top-Right: green shaded collected items in the Top-Left figure are now unpicked. Bottom-Left: blue shaded regions denote uncollected items that should be picked. Bottom-Right: blue shaded uncollected items in the Bottom-Left figure are picked.
    }
  \label{Eil76_2}
  
\end{figure*}

\paragraph{Fixing Trends after {2OPT} using PGCH} After applying {2OPT} as shown in \figurename~\ref{Eil76_1} (right), we fix the deviations of trends in the resultant $\Pi(L(t',p))$ and $\Omega(H(t',p))$ by using $\Pi(L(t,p))$ and $\Omega(H(t,p))$ from \figurename~\ref{Eil76_1} (left) as references. As per PGCH, at any city position $k$ in the reversed tour segment, a collected item $i$ having IPR $r_i$ below $\Pi(L(t,p),k)$ should be unpicked. Such less profitable collected items are shown as green shaded regions in \figurename~\ref{Eil76_2} (Top-Left) and the result of unpicking such less profitable items and so a changed collection plan $\bar{p}$ are shown in \figurename~\ref{Eil76_2} (Top-Right). Further, as per PGCH, at any city at position $k$ in the reversed tour segment, an uncollected item $i$ having IPR $r_i$ above $\Omega(H(t,p),k)$ should be picked. Such more profitable uncollected items in the changed collection plan $\bar{p}$ in \figurename~\ref{Eil76_2} (Top-Right) are shown as blue shaded regions in \figurename~\ref{Eil76_2} (Bottom-Left) and the result of picking such more profitable items and so a further changed collection plan $p'$ are shown in \figurename~\ref{Eil76_2} (Bottom-Right). Nevertheless,  \figurename~\ref{Eil76_2} (Top-Right) and (Bottom-Right) altogether show that such unpicking and picking of the items in the shaded regions would fix the trends of $\Pi(L(t',p))$ and $\Omega(H(t',p))$ by changing the collection plan $p$ to $\bar{p}$ first and then ultimately to $p'$ and thus having $\Pi(L(t',p'))$ and $\Omega(H(t',p'))$. \figurename~\ref{Eil76_3} (left) and (right) respectively show the solutions $\langle t',p\rangle$ and $\langle t',p'\rangle$ with respective objective values $72151.46$ and $78252.18$ while the objective value for $\langle t,p\rangle$ is $77544.88$. So the solution $\langle t',p'\rangle$ could be accepted by the search while $\langle t',p\rangle$ could be rejected.

\paragraph{Overall Comments on PGCH} While the above example shows PGCH helps improve the objective value, this is in general not true. This is because, as mentioned earlier, finding the best collection plan $p'$ to compute the quality value $Q(t')$ of the generated cyclic tour $t'$ is eventually an NP-Hard problem, whereas PGCH is just an approximation heuristic. In fact, PGCH might result into a decrease in the objective value when ({\it i}) distances between cities, not just positions as in PGCH, also affect the objective value, ({\it ii}) all low profitable items in earlier positions might not necessarily be unpicked, ({\it iii}) not enough high profitable items are available in the later positions in the reversed tour segment, or they may not have been picked by PGCH. Regardless of underestimation or overestimation of the objective value, the above positive example is just to make a point that PGCH after {2OPT} helps better evaluate the potential of a changed cyclic tour $t'$ than what just {2OPT} alone does.

\begin{figure*}[!b]
  \centering
  \begin{minipage}{0.49\textwidth}
    \centering
    \includegraphics[width=\textwidth]{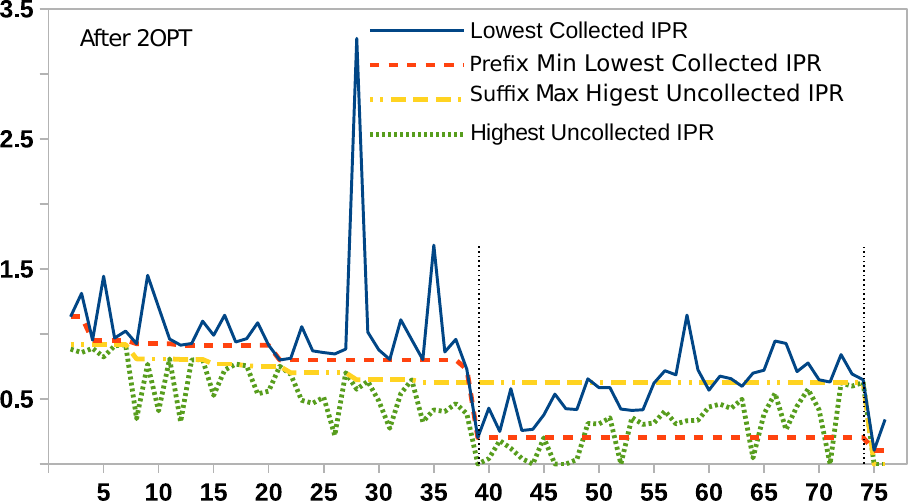}
  \end{minipage}
  \hfill
  \begin{minipage}{0.49\textwidth}
    \centering
    \includegraphics[width=\textwidth]{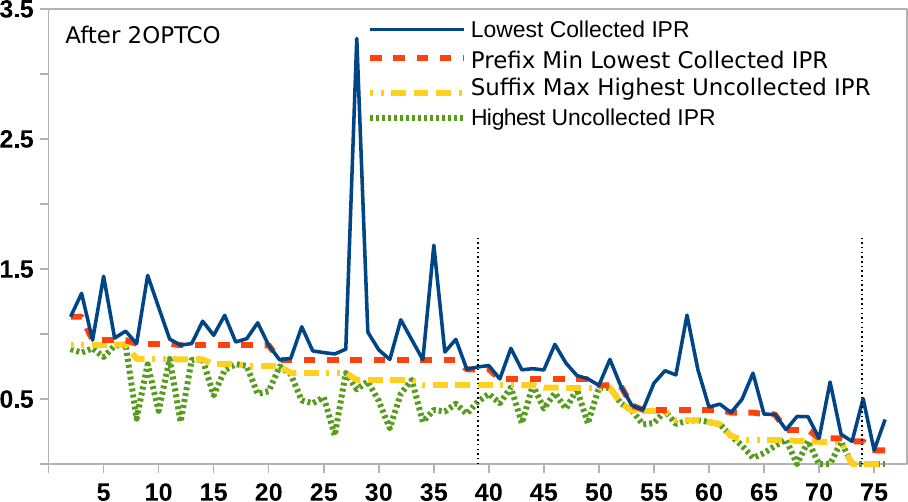}
  \end{minipage}
  \caption
    {
    City positions in a tour (x-axis) vs IPR (y-axis) (left) when  Operator~{2OPT} has just been applied to $t$ keeping $p$ unchanged and (right) when $p$ is changed using PGCH after applying Operator {2OPT}. The Left figure is the same as \figurename~\ref{Eil76_1} (right) and the Right figure is the same as \figurename~\ref{Eil76_2} (Bottom-Right) but cities outside the reversed segment and the prefix minimum of LCIPR are also shown. The objective 0value for the Left solution is $72151.46$ and that for the Right solution is $78252.18$ while that for the solution in \figurename~\ref{Eil76_1} (left) before applying {2OPT} is $77544.88$.
    }
  \label{Eil76_3}
    
\end{figure*}

\paragraph{PGCH Implementation} Algorithm~\ref{algo:PGCH} shows the implementation of our PGCH approach. In the first loop, collected items that have IPR below $\Pi(L(t,p,k))$ are unpicked and in the second loop, uncollected items that have IPR above $\Omega(H(t,p,k))$ are picked.

\begin{algorithm}[!tb]
\caption{Implementing PGCH on the Search Framework}
\label{algo:PGCH}
\begin{algorithmic}
    \Function{ProfitGuidedCoordHeu}{$t,p,t',b,e$}
        \State\LeftComment{defines \Call{CoordHeu}{$t,p,t',b,e$}}
        \State $p' \gets p$
        \For{$b \leq k \leq e$} \LeftComment{unpick forward}
            \For{$i \in I(t'_k)$}
                \If{$p'_i = 1 \land r_i < \Pi(L(t,p,k))$}
                    \State $p'_i \gets 0$
                \EndIf
            \EndFor
        \EndFor
        \For{$e \geq k \geq b$} \LeftComment{pick backward}
            \For{$i \in I(t'_k)$}
                \If{$p'_i = 0 \land r_i > \Omega(H(t,p,k))$}
                    \State $p'_i \gets 1$ {\bf when} $K(p')$ \LeftComment{Knapsack constraint}
                \EndIf
            \EndFor
        \EndFor
        \State \Return $p'$
    \EndFunction
\end{algorithmic}
\end{algorithm}

\begin{lemma}[PGCH]
Given a TTP solution $\langle t', p\rangle$ obtained after applying Operator $\Call{2OPT}{t,b,e}$ on solution $\langle t,p\rangle$, Algorithm~\ref{algo:PGCH} computes $p'$ and so a new solution $\langle t',p'\rangle$. Computing $p'$ and then $N(t',p')$ takes $O(n - b + |I(t'[b,e])|)$ time in total.
\end{lemma}

\begin{proof} For each application of Algorithm~\ref{algo:PGCH}, the two loops run for $O(|I(t'[b,e])|)$ times. Then, to compute $N(t',p')$, as an approximation of $Q(t')$, for each city from positions $b$ to $n-1$ in $t'$, we need to compute the knapsack weight and the travelling speed, which needs $O(n-b)$ time.
Thus, PGCH and computation of $N(t',p')$ takes $O(n-b + |I(t'[b,e])|)$ time in total.
\end{proof}

Note PGCH implementation requires computation of $\Pi(L(t,p))$ and $\Omega(H(t,p))$ for the current solution $\langle t,p\rangle$ in the beginning of each iteration of the main loop in Function {TSPS}. Below we show the time complexity of this.

\begin{lemma}[Trend Lines] Computing $\Pi(L(t,p))$ and $\Omega(H(t,p))$ for any solution $\langle t,p\rangle$ requires $O(n+m)$ time.
\end{lemma}

\begin{proof}
Based on definitions~\ref{def:prefixmin} and \ref{def:suffixmax}, computing these sequences for any solution $\langle t,p\rangle$ needs considering all items in $O(m)$ time and considering all cities in $O(n)$ time. So, the total needed time is $O(n+m)$.
\end{proof}

\subsection{Coordination Based Item Selection}
\label{subsec:cish}

In Function {KPS} when called from {TTPS} in Algorithm~\ref{algo:MainSearchFramework}, Function $\Call{SelectItemsSubset}{t,p,b,e}$ is by default defined by Function {SelectTourSegmentItems} that returns all items in $I(t[1,n-1])$, i.e. all items in $I$, in an uncoordinated fashion. As mentioned before, it is called the standard bit-flip search (SBFS). However, SBFS leads to an unguided exploration of the collection plans. In this paper, we present a targeted form of bit-flip search and name it {\em marginal bit-flip search} (MBFS). MBFS restricts the items to be explored using the cyclic tour in a coordinated fashion. We call our proposed approach for selection of the items to be explored as Coordinated Item Selection Heuristic (CISH) and we define $\Call{SelectItemsSubset}{t,p,b,e}$ as such.

Before going into further details, let us define marginally collected and uncollected items in a given tour segment for a given TTP solution. The marginally collected items have the lowest collected IPRs at the cities where the prefix minimum sequence changes as we move from low positions to the high positions in the cyclic tour. Similarly, the marginally uncollected items have the higest uncollected IPRs at the cities where the suffix maximum sequence changes as we move from high positions to low positions.

\begin{definition}[Marginally Collected Item]\label{def:margincolleced}
Given a TTP solution $\langle t,p\rangle$ and a tour segment $t[b,e]$, an item $i \in I(t[b,e])$ is a {\em marginally collected item} if there exists $k: i \in I(t_k)$ such that $r_i = L(t,p,k) = \Pi(L(t,p),k)$ and there exists no $k': b\leq  k'< k$ such that $L(t,p,k') = L(t,p,k)$.
\end{definition}

\begin{definition}[Marginally Uncollected Item]\label{def:marginuncolleced}
Given a TTP solution $\langle t,p\rangle$ and a tour segment $t[b,e]$, an item $i \in I(t[b,e])$ is a {\em marginally uncollected item} if there exists $k: i \in I(t_k)$ such that $r_i = H(t,p,k) = \Omega(H(t,p),k)$ and there exists no $k': k < k' \leq e$ such that $H(t,p,k') = H(t,p,k)$.
\end{definition}

Our CISH approach, in case of using MBFS in Function {KPS} considers unpicking only marginally collected items and picking only marginally uncollected items. Algorithm~\ref{algo:CISH} shows Function~$\Call{SelectMarginalItems}{t,p,b,e}$, that implements the CISH approach. This function returns at most one arbitrarily selected marginally collected item and at most one arbitrarily selected marginally uncollected item from each city in the tour segment $t[b,e]$, even though multiple marginally collected items or multiple marginally uncollected items could exist in a city.

\begin{algorithm}[!tb]
\caption{Implementing CISH for using MBFS in Function KPS}
\label{algo:CISH}
\begin{algorithmic}
    \Function{SelectMarginalItems}{$t,p,b,e$}
        \State\LeftComment{defines \Call{SelectItemsSubset}{$t,p,b,e$}}
        \State $I_\textsf{collected} \gets $ \{ at most one marginally collected item
		\State\qquad\qquad\qquad\qquad for each position $k \in [b,e]$ \}
        \State $I_\textsf{uncollected} \gets $ \{ at most one marginally uncollected item
		\State\qquad\qquad\qquad\qquad for each position $k \in [b,e]$ \}
        \State $I_\textsf{marginal} \gets I_\textsf{collected} \cup I_\textsf{uncollected}$
        \State \Return $I_\textsf{marginal}$
    \EndFunction
\end{algorithmic}
\end{algorithm}

We now analyse the time complexity of applying Operator~{BitFlip} in case of using MBFS in Function {KPS} on a marginally collected or uncollected item and subsequently update the prefix minimum and suffix maximum sequences.

\begin{lemma}
Applying the {BitFlip} operator on a marginally collected or uncollected item requires $O(n+m)$ time to recompute $\Pi(L(t,p))$ and $\Omega(H(t,p))$ and thus update marginally collected and uncollected items.
\end{lemma}

\begin{proof}
The proof is obtained from Definitions \ref{def:prefixmin}, \ref{def:suffixmax}, \ref{def:margincolleced}, and \ref{def:marginuncolleced}.
\end{proof}

Note that Function~{SelectMarginalItems} could not be used when Function~\Call{KPS}{$t',p,b,e$} is called from Function {SearchGuidedCoordHeu} as part of SGCH implementation in Algorithm~\ref{algo:SGCH}. The reason is after applying Operator~{2OPT}, calling Function~{SelectMarginalItems} do not find any marginally collected or uncollected items in $I(t'[b,e])$ since the prefix minimum $\Pi(L(t',p))$ and suffix maximum $\Omega(H(t',p))$ sequences, as shown in \figurename~\ref{Eil76_1} (right), do not exhibit any changes within the reversed tour segment. As such, for SGCH, we could at best use Function~{SelectTourSegmentItems} as is shown in Algorithm~\ref{algo:SGCH}.

\subsection{Machine Learning Based Coordination}\label{subsec:lgch}

As shown in \figurename~\ref{Eil76_1} (left), the prefix minimum and suffix maximum sequences roughly demarcate collected and uncollected items creating non-linear demarcation lines. As such, from a number of generated example TTP solutions for the same given TTP instance, a properly trained non-linear binary classifier (NLBC) could learn classification of an item as collected or uncollected at a given position of its city in a given cyclic tour and the training could even be online and instance specific. After {2OPT} in Function {TSPS} in Algorithm~\ref{algo:AuxSearchFramework}, we then can use the trained NLBC to decide which item in the reversed segment is to be collected and which one is not to be. This essentially replaces the local search based or human designed intuitive coordination with machine learning based coordination. Nevertheless, we name this proposed approach as Learning Guided Coordination Heuristic (LGCH).

Given a typical timeout limit of $10$ minutes to solve each problem instance, as is used as standard in evaluation of TTP methods, it is difficult to perform online training of NLBC models within the timeout limit before using them during search for the rest of the left-out time. We still perform instance specific online training within the timeout limit. However, after running preliminary experiments, we keep the learning effort as low as practical and set required parameter values as deemed appropriate. Below we describe the NLBC models, their training procedures, and their use during search.

\paragraph{Training and Validation Examples}

For a given problem instance, we generate $\dfrac{30}{\max_c|I(c)|}$ solutions to be used in training and half of that number of solutions to be used in validation. The training and validation solutions are generated by using Chained Lin-Kernighan heuristic~\cite{applegate2003chained} for cyclic tours followed by PackIterative~\cite{faulkner2015approximate} and Insertion~\cite{mei2014improving} for collection plans for the cyclic tours. Keeping the generated cyclic tours unchanged, only the generated collection plans are further improved by running our proposed MBFS algorithm and the improved collection plans are actually used in training and validation of the neural network. In this way, our learning model captures the characteristics of the initialisation and improvement of the collection plan by the MBFS algorithm. Then, we use the learning model to define Function {CoordHeu} to be used within Function {TSPS} in Algorithm~\ref{algo:AuxSearchFramework}. Nevertheless, the actual input to the NLBC models are the {\em normalised item profitability ratio} $\textsf{nipr}(i)=\dfrac{r_i}{\max_{i'}r_{i'}}$ for an item $i$ and its {\em normalised position  $\textsf{np}(i)=\dfrac{t(l_i)}{n}$} in a cyclic tour $t$ of a TTP solution $\langle t, p\rangle$. On the other hand the actual output of the NLBC models are $p_i$ denoting whether an item $i$ is collected or not in the collection plan $p$ of the same TTP solution $\langle t,p\rangle$. To be more specific, input features $\textsf{nipr}(i)$ and $\textsf{np}(i)$ of each item $i$ is fed to the NLBC model at a time and $p_i$ is predicted for the same item. So the training examples comprise all items in all training solutions. However, each pair $\langle \textsf{nipr}(i), \textsf{np}(i)\rangle$ can appear in multiple solutions. So we take only unique such pairs from the generated training solutions and use the collection state $p_i$ with the highest frequency over all the corresponding solutions in training the NLBC models.
Conceptually, an NLBC model would make an overall prediction of whether an item should be collected or not when its city is in a certain position in a possible cyclic tour for the given TTP instance.

\paragraph{Neural Networks as NLBCs} We use neural networks to represent NLBCs. For just two inputs $\textsf{nipr}(i)$ and $\textsf{np}(i)$ and one output $p_i$ for any given item $i$, we could think of simpler statistical models. We choose neural networks because the output for a given $i$ not only does explicitly depend on just $\textsf{nipr}(i)$ and $\textsf{np}(i)$ but also implicitly depends on the inputs features for the other $i$ values. In our view, the neural networks through their weights are a promising means to accumulate the implicit dependencies over $i$ values and generalise over problem instances. As for using more input features, we have tried to incorporate distance, but in our preliminary experiments, the city positions appeared to be more promising than the distances of the cities from City 1 in the forward or backward direction. Having discussed this, we acknowledge that further experiments with various machine learning techniques and input features are necessary to make any more meaningful conclusion in this regard. We further emphasise that our main focus in developing LGCH is to show that a machine learning approach could effectively capture the characteristics of our human designed coordination heuristics. Of course better performing machine learning approaches could be developed and we consider that to be out of scope of the paper.

\paragraph{Neural Network Architecture and Training} \figurename~\ref{fig:neuralnet} shows the architecture of the neural network. It has three layers: one input layer, one hidden layer and one output layer. The first two layers have $\ln m$ neurons each, where $m$ is the number of items. The last layer has only one neuron. We use the rectified linear unit (ReLU) as the activation function in the neurons in the first two layers and the sigmoid activation function in the neuron in the last layer.
We use the feed forward neural network architecture in the mlpack C++ library~\cite{curtin2023mlpack} with its default optimiser~\cite{curtin2021ensmallen}.
We train the same neural network architecture $10$ times to get $10$ separately trained models for each TTP instance.
We then take the best trained model in terms of the number of the correctly classified pairs of $\langle \textsf{nipr}(i), \textsf{np}(i)\rangle$ for the validation examples.
Henceforth, we refer to the best trained neural network as the neural network $\mathcal{N}$ and use $\mathcal{N}(\textsf{nipr}(i),\textsf{np}(i)) = p_i$ to denote its prediction $p_i$ made for the pair $\langle \textsf{nipr}(i),\textsf{np}(i)\rangle$.

 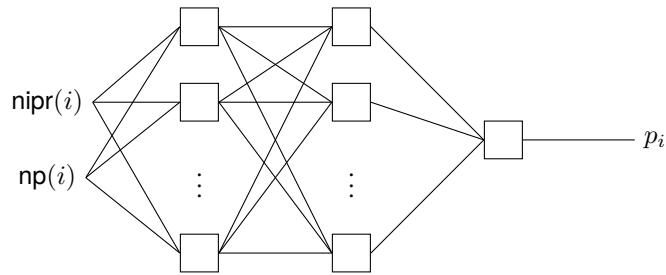
\begin{figure}[!b]
 \centering
 \begin{tikzpicture}[scale=1.0]
  \node (inputr) at (0,0.5) {$\textsf{nipr}(i)$};
  \node (inputrho) at (0,-0.5) {$\textsf{np}(i)$};
  \node[draw,minimum width=0.5cm,minimum height=0.5cm] (node11) at (2,1.5) {};
  \node[draw,minimum width=0.5cm,minimum height=0.5cm] (node12) at (2,0.5) {};
  \node (node13) at (2,-0.5) {\vdots};
  \node[draw,minimum width=0.5cm,minimum height=0.5cm] (node14) at (2,-1.5) {};
  \node[draw,minimum width=0.5cm,minimum height=0.5cm] (node21) at (4,1.5) {};
  \node[draw,minimum width=0.5cm,minimum height=0.5cm] (node22) at (4,0.5) {};
  \node (node23) at (4,-0.5) {\vdots};
  \node[draw,minimum width=0.5cm,minimum height=0.5cm] (node24) at (4,-1.5) {};
  \node[draw,minimum width=0.5cm,minimum height=0.5cm] (node31) at (6,0) {};
  \node (output) at (8,0) {$p_i$};

  \draw[-] (inputr.east) -- (node11.west);
  \draw[-] (inputr.east) -- (node12.west);
  \draw[-] (inputr.east) -- (node14.west);
  \draw[-] (inputrho.east) -- (node11.west);
  \draw[-] (inputrho.east) -- (node12.west);
  \draw[-] (inputrho.east) -- (node14.west);
  \draw[-] (node11.east) -- (node21.west);
  \draw[-] (node11.east) -- (node22.west);
  \draw[-] (node11.east) -- (node24.west);
  \draw[-] (node12.east) -- (node21.west);
  \draw[-] (node12.east) -- (node22.west);
  \draw[-] (node12.east) -- (node24.west);
  \draw[-] (node14.east) -- (node21.west);
  \draw[-] (node14.east) -- (node22.west);
  \draw[-] (node14.east) -- (node24.west);
  \draw[-] (node21.east) -- (node31.west);
  \draw[-] (node22.east) -- (node31.west);
  \draw[-] (node24.east) -- (node31.west);
  \draw[-] (node31.east) -- (output.west);
 \end{tikzpicture}
 \caption{The neural network architecture used in representing NLBCs in our proposed LGCH. The architecture has 3 layers with two input features $\textsf{nipr}(i) = \dfrac{r_i}{\max_{i'}r_{i'}}$ and $\textsf{np}(i) = \dfrac{t(l_i)}{n}$ and one output $p_i$. The first two layers have $\ln(m)$ neurons, where $m$ is the number of items.}
 \label{fig:neuralnet}
 
 \end{figure}

\paragraph{LGCH Implementation Using Neural Network Predictions} For Function~{CoordHeu} in Function~{TSPS} in Algorithm~\ref{algo:AuxSearchFramework}, we define Function $\Call{LearningGuidedCoordHeu}{t,p,t',b,e}$ to be returning $p'$ where $p'_i = \mathcal{N}(\textsf{nipr}(i),\textsf{np}(i))$ for $i \in I(t'[b,e])$ and $p'_i = p_i$ for $i\not\in I(t'[b,e])$. Note that because of the knapsack constraint, the precise implementation, as shown in Algorithm~\ref{algo:LGCH}, needs unpicking of all items followed by picking of the items predicted to be collected in the reversed tour segment.

 \begin{algorithm}[!tb]
\caption{Implementing LGCH on the Search Framework}
\label{algo:LGCH}
\begin{algorithmic}
    \Function{LearningGuidedCoordHeu}{$t,p,t',b,e$}
        \State\LeftComment{defines \Call{CoordHeu}{$t,p,t',b,e$}}
        \State $p' \gets p$
        \For{$i \in I(t'[b,e])$} \LeftComment{any order}
            \If{$p'_i = 1$}
                \State $p'_i \gets 0$
            \EndIf
        \EndFor
        \For{$i \in I(t'[b,e])$} \LeftComment{any order}
            \State $\textsf{NIPR} \gets \dfrac{r_i}{\textsf{max}_{i'}(r_i')}$
            \State $\textsf{NP} \gets \dfrac{t'(l_i)}{n}$
            \If{$\mathcal{N}(\textsf{NIPR}, \textsf{NP}) = 1$}
                \State $p'_i \gets 1$ {\bf when} $K(p')$ \LeftComment{Knapsack constraint}
            \EndIf
        \EndFor
        \State \Return $p'$
    \EndFunction
\end{algorithmic}
\end{algorithm}

\begin{algorithm}[!tb]
\caption{Implementing LGCH Efficiently on the Search Framework}
\label{algo:boundary}
\begin{algorithmic}
    \Function{LearningGuidedCoordHeu}{$t,p,t',b,e$}
        \State\LeftComment{defines \Call{CoordHeu}{$t,p,t',b,e$}}
        \State $p' \gets p$
        \For{$i \in I(t'[b,e])$} \LeftComment{unpick any order}
            \If{$p'_i = 1~\land~r_i < \mathcal{B}[t'(l_i)]$}
                \State $p'_i \gets 0$
            \EndIf
        \EndFor
        \For{$e \geq k \geq b$} \LeftComment{pick backward} 
            \For{$i \in I(t'_k)$}
                \If{$p'_i = 0~\land~r_i \geq \mathcal{B}[k]$}
                    \State $p'_i \gets 1$ {\bf when} $K(p')$
                \EndIf
            \EndFor
        \EndFor
        \State \Return $p'$
    \EndFunction
    \State
    \Function{computeBPRs}{$\mathcal{N}$}
    \State $\mathcal{B}:$ BPR for each position $k$
    \State $\mathcal{R} \gets$ sort all unique $r_i$ in increasing order
    \For{$0 < k < n$}
        \State $\mathcal{B}[k] \gets \Call{computeBPR}{\mathcal{N}, \mathcal{R}, k}$
    \EndFor
    \Return $\mathcal{B}$
    \EndFunction
    \State
    \Function{computeBPR}{$\mathcal{N}, \mathcal{R}, k$}
    \State $\textsf{NP} \gets \dfrac {k}{n}$
    \State ${low} \gets 0$
    \State ${high} \gets |\mathcal{R}| - 1$
    \While{${low} \leq {high}$} \LeftComment{perform binary search}
        \State ${mid} \gets ({low} + {mid})/2$
        \State $\textsf{NIPR} \gets \dfrac{\mathcal{R}[mid]}{\textsf{max}_i(r_i)}$

        \If{\Call{$\mathcal{N}$}{$\textsf{NIPR},\textsf{NP}$} = 1}
            \State ${high} \gets {mid} - 1$
        \Else~ \LeftComment{{\bf if}{~\Call{$\mathcal{N}$}{$\textsf{NIRP},\textsf{NP}$} = 0}}
            \State ${low} \gets {mid} + 1$
        \EndIf
    \EndWhile
    \If{${low} = |\mathcal{R}|$}
        \State \Return $\textsf{max}_i(r_i) + 1$
    \Else
        \State \Return $\mathcal{R}[{low}]$
    \EndIf
    \EndFunction
\end{algorithmic}
\end{algorithm}

\paragraph{Reusing Neural Network Predictions} Since for an item, the neural network $\mathcal{N}$ only needs $\textsf{nipr}(i)$ and $\textsf{np}(i)$, we can actually make predictions for all items and all positions beforehand and store them. This would certainly save the time required to recompute the predictions for the same items and the same positions over and over again. In fact, our preliminary experiment shows that recomputation of the predictions becomes costlier since each call of $\mathcal{N}$ is arguably compute intensive. However, a straightforward approach to store all predictions for all items for all positions require $O(nm)$ memory and more importantly needs $O(nm)$ calls of the costlier computation of $\mathcal{N}$. In this paper, we propose an alternative strategy to store only one profitability ratio for each position and thus taking only $O(n)$ memory and $O(n\log_2 m)$ calls of $\mathcal{N}$. The idea is to store the profitability ratio, called the {\em boundary profitability ratio (BPR)} that approximately demarcates the collected items from the uncollected items in a given position. The idea is again based on the previously mentioned key guideline for TTP that at a given position, more profitable items are more likely to be collected. The notions of lowest collected IPR and highest uncollected IPR are relevant in this context. However, instead of two such IPRs, we rather use one BPR in this case. Nevertheless, Algorithm~\ref{algo:boundary} shows our implementation of computing BPR for each position. In Function~{computeBPRs} in Algorithm~\ref{algo:boundary}, we first sort profitability ratios of all items in a non-decreasing order and store only unique values in increasing order in $\mathcal{R}$. Then, for each position $k$, we store in $\mathcal{B}[k]$ the value returned by Function~\Call{computeBPR}{$\mathcal{N},\mathcal{R},k$} that runs binary search to find the profitability ratio below which items are not collected at position $k$. Next, we redefine Function $\Call{LearningGuidedCoordHeu}{t,p,t',b,e}$ to be returning $p'$ where for $i \in I(t'[b,e])$, $p'_i = 1$ when $r_i \geq \mathcal{B}[t'(l_i)]$ and $p'_i = 0$ when $r_i < \mathcal{B}[t'(l_i)]$, and for $i \not\in I(t'[b,e])$, $p'_i = p_i$. Note that because of the knapsack constraint, the precise implementation, as shown in Algorithm~\ref{algo:boundary}, needs unpicking of all items followed by picking of the items predicted to be collected in the reversed tour segment. Also, note that after computing BPRs in $\mathcal{B}$, we no longer need the neural network $\mathcal{N}$ and BPRs are sufficient for the purpose of our machine learning guided coordination heuristic.

After computing BPRs using the trained neural network only to replace the same neural network with the computed BPRs, the same question could come again whether we could use a simpler machine learning model. Considering the scope of this work, we leave the quest for finding a better machine learning model for future. However, we make a particular note that an explicit metric to determine the BPRs is not known to us and we have just relied on the implicit power of a neural network for this.

\section{Experiments}
\label{sec:experiments}

We describe the benchmark instances that we use in our experiments. We also discuss the experiment settings and evaluation metrics. Then, we compare various versions of our proposed solver. Finally, we compare our proposed solver with existing state-of-the-art TTP solvers.

\subsection{Benchamrk TTP Instances}

TTP solvers are typically evaluated using the benchmark instances introduced in~\cite{polyakovskiy2014comprehensive}. Each TTP benchmark instance has been generated based on the following things:

\begin{itemize}
\itemsep0ex
\item A symmetric TSP instance with 51 to 85900 cities as taken from TSPLIB \cite{reinelt1991tsplib}. While generating the benchmark instances, the number of cities has been used in determining the total number of items.
\item A set $I(c)$ of 1, 3, 5 or 10 items for each city $c$. So each TTP instance has $m = (n-1) \times |I(c)|$ items. Note that $\max_cI(c)$ is used in generating training solutions for LGCH in Section~\ref{subsec:lgch}.
\item Weights and profits of all items are ({\it i}) bounded and strongly correlated, or ({\it ii}) uncorrelated but weights are similar for all items, or ({\it iii}) fully uncorrelated. 

\item A knapsack with a weight capacity indicator ranging from 1 to 10, where larger indicator means larger knapsack capacity (not the knapsack capacity itself) \cite{polyakovskiy2014comprehensive}.
\end{itemize}

Note that the exact TTP instances used in our experiments are downloaded from \url{https://cs.adelaide.edu.au/~optlog/CEC2014COMP_InstancesNew/}. These instances have from $76$ to $33810$ cities and from $75$ to $338090$ items with the knapsack capacity from $5780$ to $153960049$ unit of weight. These instances are divided into $3$ categories \cite{el2018efficiently,el2016population}. Below we briefly describe the three categories.

\begin{itemize}
\itemsep1ex
\item {\bf CatA:} The knapsack weight capacity is relatively small. There is only one item in each city. The weights and profits of the items are bounded and strongly correlated.
\item {\bf CatB:} The knapsack weight capacity is moderate. There are 5 items in each city. The weights and profits of the items are uncorrelated. The weights of all items are similar.
\item {\bf CatC:} The knapsack weight capacity is high. There are 10 items in each city. The weights and profits of the items are uncorrelated.
\end{itemize}

As shown below, there are $20$ TTP instances in each of the above three categories. The TTP instance names in each category are based on the names of the same TSP instances that are used in generating the TTP instances. For each TSP instance, three TTP instances are generated for three categories, just by changing the item distribution as discussed above in the description of the categories. Notice that the numbers of cities appear in the names of the instances. Depending on the categories, the numbers of items are various multiples of the numbers of cities.

\begin{multicols}{5}
\begin{enumerate}
\itemsep0ex
\item {eil76}
\item {kroA100}
\item {ch130}
\item {u159}
\item {a280}
\item {u574}
\item {u724}
\item {dsj1000}
\item {rl1304}
\item {fl1577}
\item {d2103}
\item {pcb3038}
\item {fnl4461}
\item {pla7397}
\item {rl11849}
\item {\sf\small usa13509}
\item {brd14051}
\item {d15112}
\item {d18512}
\item {pla33810}
\end{enumerate}
\end{multicols}

We analyse the performance of the solvers on each category, but for overall analyses, we also use all $60$ instances from the three categories altogether.
In the charts, unless mentioned otherwise, performance on the instances is plotted in the order CatA, CaB, CatC of categories and within each category in the order of the instances as shown above.
Notice that the order of instances in this way within each category is roughly in the order of their sizes.

\subsection{Settings}

We run each solver version on each TTP instance $10$ times, each time with a standard timeout of $10$ minutes.
For each run in all experiments, we ensure a new initial cyclic tour is generated using the Chained Lin-Kernighan heuristic~\cite{applegate2003chained} whenever an initial cyclic tour is needed in each run or in each restart in a run. We run all experiments on the high performance computing cluster Gowonda with a 2~GB memory limit and an Intel Xeon CPU X5650 running at 2.66~GHz on each machine.

To measure performance differences across solvers,
we use the relative deviation index~(RDI)~\cite{kim1996simulated}
for each solver on each TTP instance.
RDI for a given solver on a given TTP instance is defined as $\dfrac{N_\textsf{mean} - N_\textsf{min}}{N_\textsf{max} - N_\textsf{min}}\times 100$ where $N_\textsf{max}$ and $N_\textsf{min}$ are respectively the maximum and minimum $N(t,p)$ over all runs over all solver versions that we run for the respective experiment and $N_\textsf{mean}$ is the mean over all $10$ runs of the same solver. Note that the larger the RDI value of a solver version, the better its performance. While we use RDI values to present our main results, we do include in the appendix $N_\textsf{max}$, $N_\textsf{min}$, and $N_\textsf{mean}$ along with  $N_\textsf{stddev}$, and $N_\textsf{median}$ for each solver for each instance where $N_\textsf{stddev}$ and $N_\textsf{median}$ are the standard deviation and median of $N(t,p)$ values over the $10$ runs of the solver on the instance.

We use Wilcoxon Signed Rank Test with 95\% confidence interval and also 95\% Confidence Interval plots to show the significance of differences in the performances of various solvers and versions.

We use line charts to compare instance specific performances of various solvers. The line charts have the problem instances on the x-axis. The problem instances are sorted on the number of cities within each category. We have noted before that the number of items in each instance depends on the number of cities. Nevertheless, it is in general difficult to find a well-justified order of the instances in terms of hardness even when the numbers of cities and items increase in TTP and as such we do not intend to find any obvious trend.  Given that no trend is intended among the problem instances, one could think of using bar charts in such cases. We do not use bar charts because with large numbers of data points, the bodies of the bars make extracting information from the peaks of the bars difficult by matching the same type bars.

\subsection{Comparison of Proposed Solver Versions}

In Algorithm~\ref{algo:AuxSearchFramework}, in Function {TSPS}, we have four ways to define Function {CoordHeu}:  {NoCoordHeu}, {SearchGuidedCoordHeu}, {ProfitGuidedCoordHeu}, and {LearningGuidedCoordHeu}, which are respectively denoted by NOCH, SGCH, PGCH, and LGCH. Further, in Algorithm~\ref{algo:AuxSearchFramework}, in Function {TSPS}, Function {KPS} can be run in two ways: standard bit-flip search and marginal bit-flip search, which are respectively denoted by SBFS and MBFS. Note that MBFS uses coordinated item selection heuristic (CISH) to limit BitFlip operators only on the marginal items.
So we denote our proposed solver version by $X+Y$, where $X \in \{\textrm{NOCH}, \textrm{SGCH}, \textrm{PGCH}, \textrm{LGCH}\}$ and $Y \in \{\textrm{SBFS}, \textrm{MBFS}\}$.
For example, NOCH+SBFS denotes a solver version having NOCH and SBFS, and is the baseline version as described in Section~\ref{subsec:baseline}.

\subsubsection{Overall Effectiveness of MBFS Approach}

\figurename~\ref{fig:pgch-cish} shows that NOCH+MBFS outperforms NOCH+SBFS and PGCH+MBFS outperforms NOCH+SBFS. The differences are clear in the large instances in all three categories albeit some mixed performances by PGCH+MBFS and PGCH+SBFS in small CatB and CatC instances and in CatA instances that all have comparatively small number of items. Exploring
only marginally collected and uncollected items using CISH inside MBFS approach allows more focused exploration and more efficient utilisation of the limited time budget. In the
instances having fewer items, the restriction however excessively reduces the search space and narrows down the chance of finding better solutions due to lack of diversity. Nevertheless, we compute p-values of Wilcoxon Signed Rank Test on the RDI values of all $60$ instances. The p-value for NOCH+MBFS and NOCH+SBFS is $0.00001$ while that for PGCH+MBFS and PGCH+SBFS is $0.0012$. So at 95\% confidence level, we conclude our MBFS approach statistically significantly improves the performance over SBFS.

\begin{figure}[!b]
    
    \centering
    \includegraphics[width=\textwidth,height=0.33\textwidth]{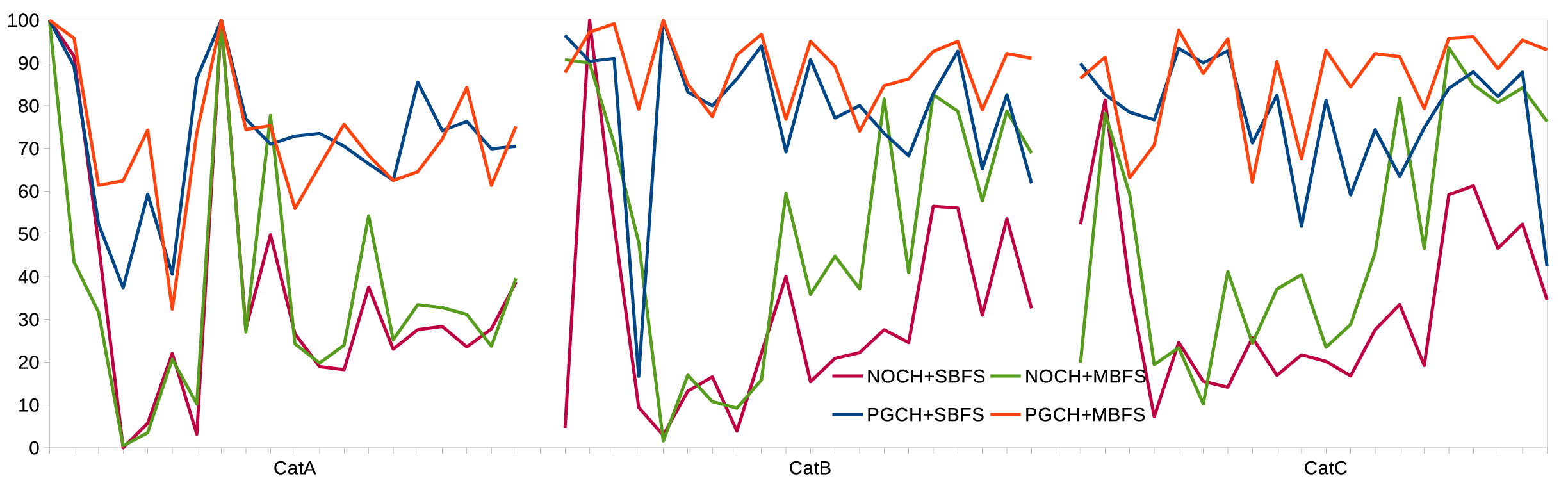}
    \caption{RDI values obtained (y-axis) on problem instances (x-axis) by various versions of our proposed solver to show the effectiveness of MBFS over SBFS and PGCH over NOCH, and also the interaction of MBFS and PGCH} \label{fig:pgch-cish}
    
\end{figure}

\subsubsection{Overall Effectiveness of PGCH Approach}

\figurename~\ref{fig:pgch-cish} shows that overall PGCH+SBFS outperforms NOCH+SBFS and PGCH+MBFS outperforms NOCH+MBFS. The differences are very clear and large in almost all instances in all three categories. We compute p-values of Wilcoxon Signed Rank Test on the RDI values of all $60$ instances. The p-value for PGCH+SBFS and NOCH+SBFS is $0.00001$ and that for PGCH+MBFS and NOCH+MBFS is also $0.00001$. So at 95\% confidence level, we conclude our PGCH approach statistically very significantly improves the performance over the NOCH approach.

\subsubsection{Learning Details of LGCH Approach}

\tablename~\ref{tab:LGCHdetails} shows the learning details of the LGCH approach. The training times in the table include the time spent in generation of training and validation solutions, sorting and selecting unique $\langle \textsf{nipr}(i), \textsf{np}(i) \rangle$ pairs from generated solutions, training $10$ neural networks, and finally computing boundary profitability ratios (BPRs) to be used in Algorithm~\ref{algo:boundary}. Given the timeout of $10$ minutes for each TTP instance, notice that the maximum training time needed is about 4 minutes and is in the largest CatA instance. Within the same category, training time increases with the increase of the problem size. For the same TSP instance, the training time decreases from CatA to CatB to CatC. This is because to keep the number of input $\langle \textsf{nipr}(i), \textsf{np}(i) \rangle$ pairs to the neural network almost the same for all three categories, we generate more training and testing solutions in CatA than in CatB and CatC (30, 6 and 3 solutions for training and 15, 3 and 2 solutions for validation in CatA, CatB, and CatC respectively). In the table, we also show the percentage of unique $\langle \textsf{nipr}(i), \textsf{np}(i) \rangle$ pairs with respect to the total number of pairs found in the example collection plans. As problem size increases, the percentage of unique pairs arguably increases. Nevertheless, the average accuracy values of the neural networks for the training and validation $\langle \textsf{nipr}(i), \textsf{np}(i) \rangle$ pairs are very high (above 95\%). The mean validation accuracy over the $60$ instances is very slightly better than the mean training accuracy. The p-value of the Wilcoxon Signed Rank Test is 0.0455 for the training and the validation accuracy values and so the difference is still statistically significant at 95\% confidence level.

\begin{table}[!tb]
\caption{Training time in seconds, \% of unique pairs among all pairs $\langle \textsf{nipr}(i), \textsf{np}(i) \rangle$ in training and validation solutions, \% average training accuracy, and \% average validation accuracy over the $10$ neural networks trained in LGCH}
\label{tab:LGCHdetails}
\centering
\setlength{\tabcolsep}{1pt}
\begin{tabular}{|l|rrrr|rrrr|rrrr|}\hline
TTP         & \multicolumn{4}{c|}{CatA} & \multicolumn{4}{c|}{CatB} & \multicolumn{4}{c|}{CatC}\\\cline{2-13}
Problem      & Train  & Unique  & Train     & Valid. & Train  & Unique  & Train     & Valid. & Train  & Unique  & Train     & Valid.\\
Instance & Time      & Pair \%    & Acc \%  & Acc  \% & Time      & Pair \%    & Acc \%  & Acc \% & Time      & Pair \%    & Acc \%  & Acc \%\\\hline
eil76	&	9.45	&	31.76	&	95.52	&	95.59	&	2.92	&	48.53	&	98.58	&	98.53	&	1.61	&	76.80	&	97.88	&	98.17	\\
kroA100	&	85.53	&	24.47	&	97.65	&	97.54	&	8.51	&	72.88	&	99.04	&	99.31	&	4.34	&	92.59	&	98.92	&	98.89	\\
ch130	&	31.71	&	39.74	&	96.68	&	96.64	&	1.99	&	42.38	&	98.55	&	98.66	&	1.65	&	70.10	&	98.45	&	98.76	\\
u159	&	30.54	&	19.04	&	97.81	&	97.75	&	2.80	&	36.03	&	99.25	&	99.15	&	1.99	&	60.51	&	98.75	&	98.88	\\
a280	&	5.62	&	35.34	&	97.47	&	97.40	&	2.43	&	78.64	&	99.21	&	99.16	&	3.29	&	83.58	&	98.75	&	98.60	\\
u574	&	17.78	&	33.53	&	98.81	&	98.77	&	4.84	&	63.07	&	99.02	&	98.89	&	5.29	&	92.26	&	98.19	&	98.15	\\
u724	&	16.36	&	57.21	&	98.21	&	98.25	&	5.70	&	51.97	&	99.28	&	99.36	&	5.58	&	98.60	&	98.55	&	98.63	\\
dsj1000	&	52.53	&	67.89	&	98.42	&	98.51	&	15.61	&	94.53	&	99.02	&	99.10	&	4.98	&	99.84	&	98.01	&	97.77	\\
rl1304	&	31.88	&	49.38	&	98.40	&	98.58	&	10.80	&	83.32	&	99.04	&	99.10	&	8.19	&	83.96	&	98.19	&	98.14	\\
fl1577	&	53.16	&	70.87	&	98.49	&	98.49	&	16.46	&	90.26	&	99.05	&	99.06	&	12.27	&	90.85	&	97.65	&	97.92	\\
d2103	&	34.86	&	51.49	&	98.89	&	98.83	&	7.57	&	64.17	&	99.41	&	99.34	&	9.69	&	99.39	&	97.30	&	97.44	\\
pcb3038	&	44.65	&	89.79	&	98.76	&	98.83	&	14.26	&	97.18	&	99.41	&	99.46	&	12.82	&	98.94	&	97.74	&	97.76	\\
fnl4461	&	70.63	&	95.01	&	98.81	&	98.81	&	19.75	&	98.98	&	99.18	&	99.36	&	13.32	&	99.75	&	98.51	&	98.40	\\
pla7397	&	83.89	&	94.32	&	97.99	&	98.02	&	27.02	&	98.30	&	97.53	&	97.40	&	31.74	&	99.25	&	97.72	&	98.14	\\
rl11849	&	129.84	&	97.65	&	99.01	&	99.04	&	49.32	&	99.54	&	99.44	&	99.45	&	64.70	&	99.90	&	98.74	&	98.70	\\
usa13509	&	116.14	&	95.87	&	98.26	&	98.26	&	42.18	&	98.60	&	97.81	&	97.91	&	59.85	&	99.23	&	97.50	&	97.76	\\
brd14051	&	126.42	&	97.53	&	98.87	&	98.87	&	49.04	&	99.64	&	98.88	&	99.04	&	46.92	&	99.83	&	98.01	&	98.12	\\
d15112	&	160.17	&	97.81	&	98.90	&	98.87	&	65.32	&	99.72	&	99.13	&	99.15	&	71.33	&	99.94	&	98.18	&	98.15	\\
d18512	&	137.85	&	97.82	&	98.74	&	98.75	&	60.98	&	99.71	&	98.89	&	98.95	&	98.41	&	99.93	&	98.10	&	98.10	\\
pla33810	&	233.37	&	98.12	&	98.79	&	98.81	&	101.91	&	99.77	&	99.18	&	99.29	&	144.13	&	99.96	&	98.65	&	98.62	\\\hline
\end{tabular}
\end{table}

\subsubsection{Overall Comparison of PGCH, SGCH, and LGCH}

In \figurename~\ref{fig:tsps-mbfs}, we compare RDI values obtained by PGCH, SGCH, LGCH, and NOCH on all $60$ instances from three categories. For these solver versions, we use MBFS since it has already been shown to be better than SBFS.
The solver versions compared are respectively PGCH+MBFS, SGCH+MBFS, LGCH+MBFS, and NOCH+MBFS.
We see that PGCH and LGCH make huge improvement over NOCH. However, SGCH performs worse than NOCH.
The reason is running KPS for every tour segment reversal, even when KPS is restricted only to the reversed segment,
takes huge time and consequently within a given timeout of 10 minutes,
not much of the TTP search space is explored. Note that we include SGCH in this comparison mainly to show that a simple local search based coordination approach does not work well in TTP.
Nevertheless, among other heuristics, PGCH appears to be performing slightly better than LGCH. With the p-value of $0.00782$ of Wilcoxon Signed Rank Test,
the difference in the performances of PGCH and LGCH is also statistically significant at 95\% confidence level.
This result is very interesting as we can see that the machine learning based algorithm LGCH has learnt almost up to the level of the human designed PGCH heuristic.

\begin{figure}[!b]
    \centering
    \includegraphics[width=\textwidth,height=0.33\textwidth]{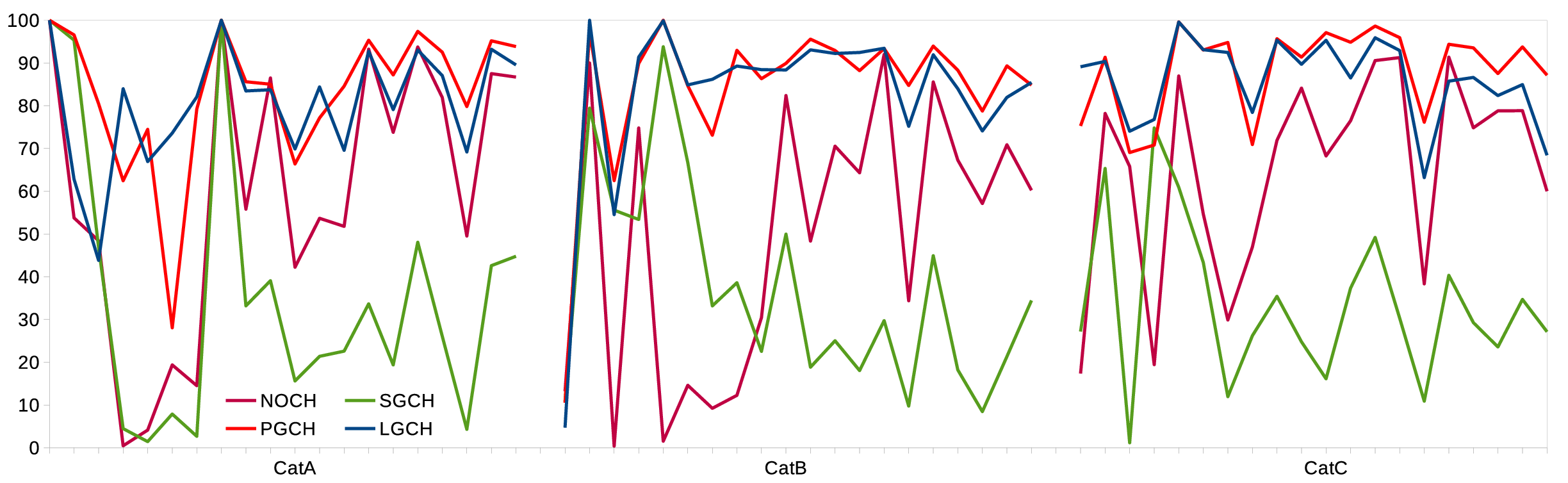}
    \caption{RDI values obtained (y-axis) on problem instances (x-axis) by various versions of our proposed solver to show the comparison of SGCH, PGCH, LGCH, and NOCH when MBFS is used with all of them.} \label{fig:tsps-mbfs}
    
\end{figure}

\figurename~\ref{fig:restarts} shows the numbers of restarts in Function {TTPS} in Algorithm~\ref{algo:MainSearchFramework} in each instance when SGCH, PGCH, LGCH, and NOCH are used along with MBFS. We see that SGCH performs the least numbers of restarts since it spends huge time in running {KPS} for each tour segment reversal in {TSPS}. The low numbers of restarts also indicate low diversity in terms of the search space exploration. Notice that PGCH and LGCH performs very similar numbers of restarts in all instances. The numbers of restarts performed by NOCH are very similar to those performed by PGCH or LGCH in small instances in each category, but is quite larger in large instances. NOCH is arguably faster than PGCH or LGCH and so help explore more of the search space by restarting more number of times. However, with a poor evaluation of the generated cyclic tours, NOCH eventually does not result into better RDI values.

\begin{figure}[!b]
    
    \centering
    \includegraphics[width=\textwidth,height=0.2\textwidth]{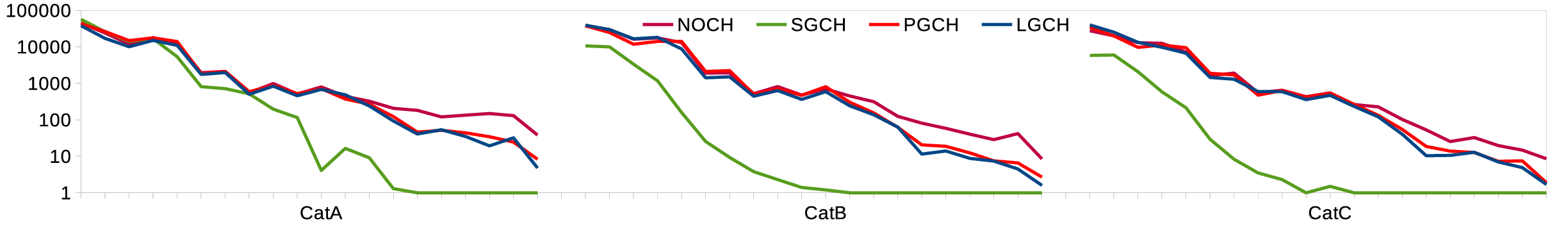}
    \caption{Numbers of restarts (y-axis) in Function {TTPS} in Algorithm~\ref{algo:MainSearchFramework} in problem instances (x-axis) by various versions of our proposed solver to show the comparison of SGCH, PGCH, LGCH, and NOCH when MBFS is used with them.} \label{fig:restarts}
\end{figure}

\newpage
\subsection{Further Analysis of PGCH and LGCH over NOCH}

\begin{figure}[!b]
    \centering
    \includegraphics[width=\textwidth,height=0.2\textwidth]{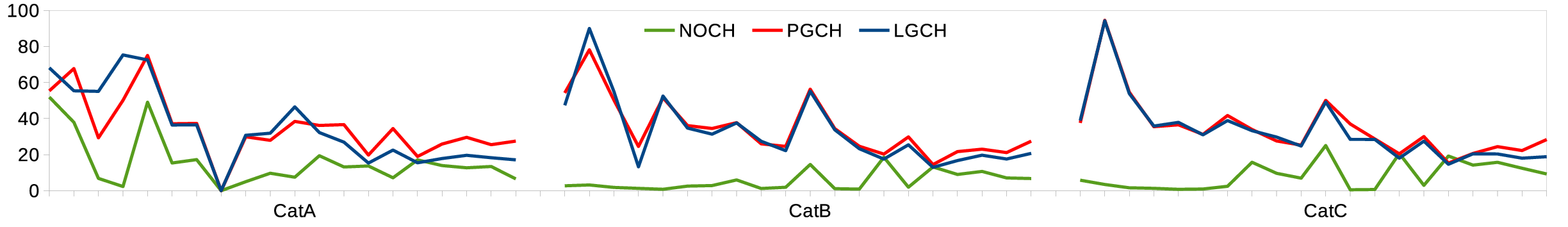}
    \caption{Mean of relative lengths (y-axis) of tour segments reversed by {2OPT} and accepted by the search algorithm in Function~{TSPS} over 10 runs when NOCH, PGCH, and LGCH are used along with MBFS on problem instances (x-axis)} \label{fig:pgch-seg-length}
    
\end{figure}

\begin{figure}[!b]
    \centering
    \includegraphics[width=\textwidth,height=0.2\textwidth]{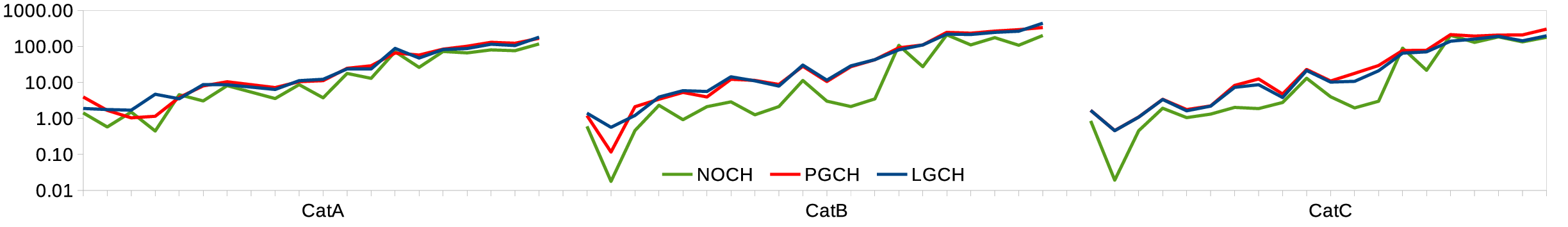}
    \caption{Mean of numbers of accepted application of {2OPT} per restarts in Function {TTPS} over 10 runs when NOCH, PGCH, and LGCH are used along with MBFS on problem instances (x-axis)}
    \label{fig:pgch-seg-accepted}
    
\end{figure}

To investigate the huge difference in the performance of PGCH and LGCH from NOCH, we observe the reverse tour segments generated, evaluated, and accepted during search. \figurename~\ref{fig:pgch-seg-length} shows the mean relative lengths $\dfrac{|t[b,e]|}{n}\times 100$ of the tour segments reversed by {2OPT} and accepted by the search algorithm in Function~{TSPS} over 10 runs when used with NOCH, PGCH and LGCH along with MBFS. Moreover, \figurename~\ref{fig:pgch-seg-accepted} shows the mean numbers of the tour segments of which mean lengths have been shown in \figurename~\ref{fig:pgch-seg-length}. From these two figures, we see that the use of PGCH and LGCH has resulted in the acceptance of notably larger tour segments reversed by {2OPT} operator and also in larger numbers than what NOCH has resulted in. In the absence of a coordination heuristic, as shown in \figurename~\ref{Eil76_3}, the quality values of the cyclic tours
produced by {2OPT} are not estimated properly and thus the reversed tour segments are rejected by the search algorithm. Arguably, this happens even at an worsened level when reversed tour segments are large in sizes and more in numbers. In contrast, when a coordination heuristic such as PGCH or LGCH is used, the quality values of the reversed tour segments are more properly estimated and as we see from \figurename{s}~\ref{fig:pgch-seg-length} and \ref{fig:pgch-seg-accepted}, larger reversed tour segments are accepted in larger numbers and thus we have obtained higher objective values at the end. This explains the advantage of PGCH and LGCH over NOCH. Notice that both in \figurename{s}~\ref{fig:pgch-seg-length} and \ref{fig:pgch-seg-accepted}, PGCH and LGCH are very close in most TTP instances, except in \figurename~\ref{fig:pgch-seg-length} in CatA instances. However, CatA instances have only one item in each city. As such picking or unpicking the only available item in each city by mistake as a classification error of the neural network is harder to compensate than the classification errors in CatB and CatC.

\begin{figure}[!b]
    
    \centering
    \includegraphics[width=\textwidth,height=0.33\textwidth]{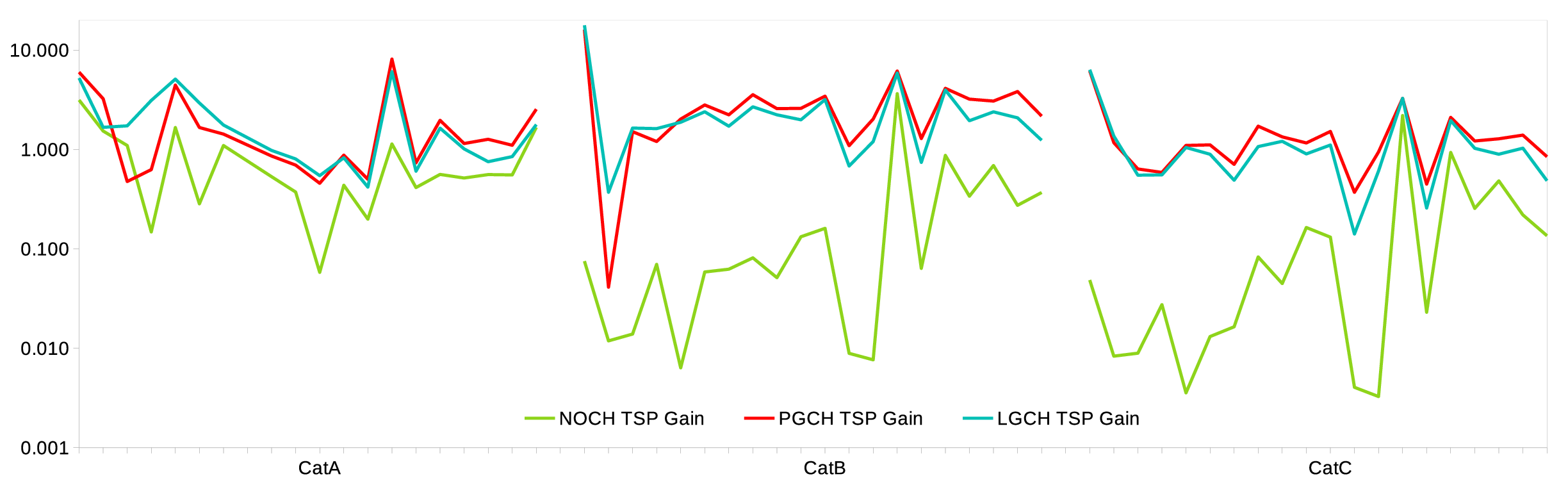}
    \caption{Mean objective gains (y-axis) $G_\textsf{TSP}$ by {TSPS} per restarts in {TTPS} over 10 runs of  Algorithm~\ref{algo:MainSearchFramework}, when NOCH, PGCH, and LGCH are used along with MBFS on problem instances (x-axis)}
    \label{fig:pgch-obj-gains}
\end{figure}

In Algorithm~\ref{algo:MainSearchFramework} in Function {TTPS}, we compute the objective values $N_\textsf{BS}$ and $N_\textsf{TSP}$ respectively before and after running {TSPS}. We then compute means of objective gains $G_\textsf{TSP} = \dfrac{N_\textsf{TSP} - N_\textsf{BS}}{N_\textsf{BS}} \times 100$ over all iterations of the outer loop in {TTPS} over all $10$ runs of Algorithm~\ref{algo:MainSearchFramework} for each instance. The mean objective gains are shown in \figurename~\ref{fig:pgch-obj-gains}. We see that, in $G_\textsf{TSP}$, in most cases, PGCH is better than LGCH, which is better than NOCH. The performance difference of NOCH from that of PGCH or LGCH is arguably huge in CatB and CatC instances. This is explainable as PGCH or LGCH is targeted to improve evaluation of the cyclic tours produced by {2OPT} and the better evaluation results in accepting longer and more tour segment reversals and hence better $G_\textsf{TSP}$ values.

\subsection{Comparison with Existing TTP Solvers}

We compare our proposed MBFS with a simulated annealing method in terms of the performance improvement in the KP component while we use PGCH with both. We then compare our PGCH+MBFS and LGCH+MBFS solvers with other existing state-of-the-art TTP methods.

\subsubsection{Comparison of MBFS with Simulated Annealing Search}

We compare our hill-climbing based MBFS algorithm with a simulated annealing search (SAS) algorithm \cite{el2018efficiently} for the KP component of TTP. The SAS algorithm defines Function \Call{KPS}{$t,p,1,n-1$} to be called in Function {TTPS} in Algorithm~\ref{algo:MainSearchFramework}. For Function {TSPS}, we use PGCH in this case with both MBFS and SAS. We compute means of objective gains $G_\textsf{KP} = \dfrac{N_\textsf{KP} - N_\textsf{TSP}}{N_\textsf{TSP}} \times 100$ over all iterations of the outer loop in {TTPS} over all $10$ runs of Algorithm~\ref{algo:MainSearchFramework} for each instance. \figurename~\ref{fig:mbfs-sas-kps} shows that MBFS is very slightly better than SAS in mean $G_\textsf{KP}$ values. However, the difference is statistically not significant with p-value $0.27572$ of Wilcoxon Signed Rank Test at 95\% confidence level. 

\begin{figure}[!b]
    
    \centering
    \includegraphics[width=\textwidth,height=0.2\textwidth]{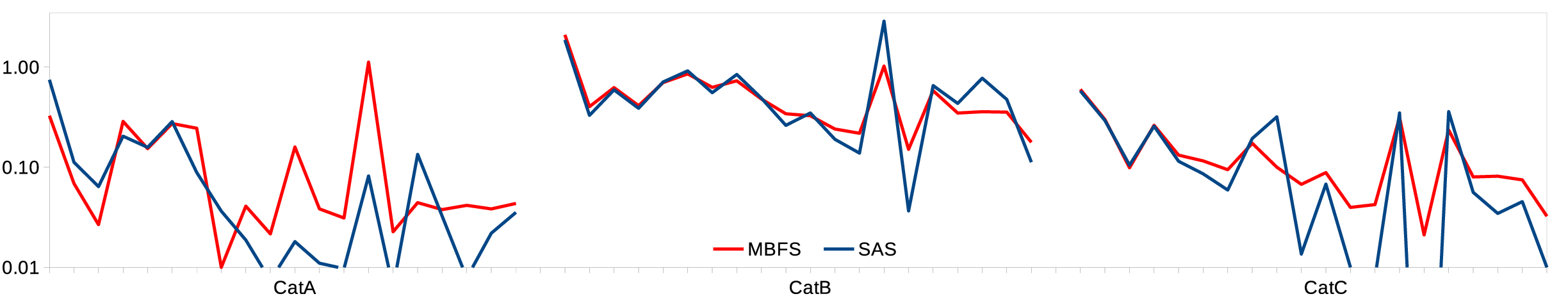}
    \caption{Mean objective gains (y-axis) $G_\textsf{KP}$ over all iterations of the outer loop of {TTPS} over 10 runs of  Algorithm~\ref{algo:MainSearchFramework}, when MBFS and SAS are used along with PGCH on problem instances (x-axis)} \label{fig:mbfs-sas-kps}
    
\end{figure}

\begin{figure}[!b]
    
    \centering
    \includegraphics[width=\textwidth,height=0.2\textwidth]{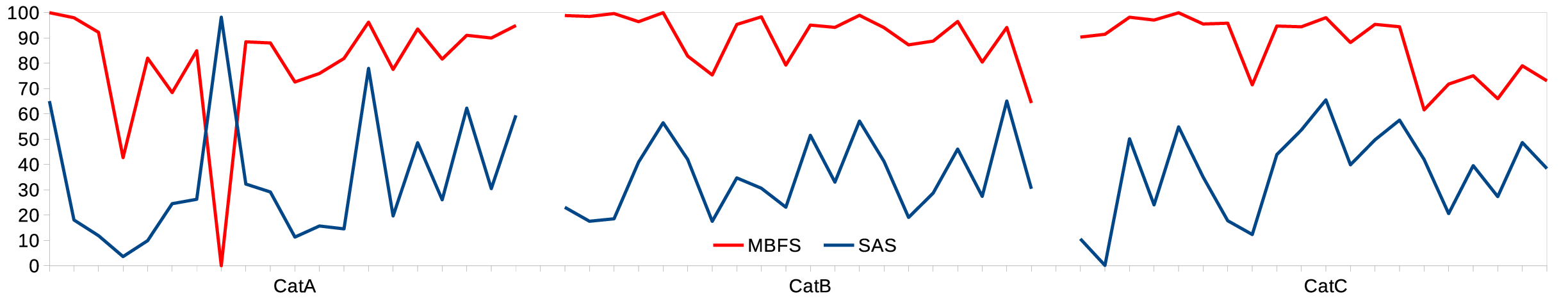}
    \caption{RDI values obtained (y-axis) on problem instances (x-axis) when MBFS and SAS are used along with PGCH } \label{fig:mbfs-sas-rdi}
    
\end{figure}

\begin{figure}[!b]
    
    \centering
    \includegraphics[width=\textwidth,height=0.2\textwidth]{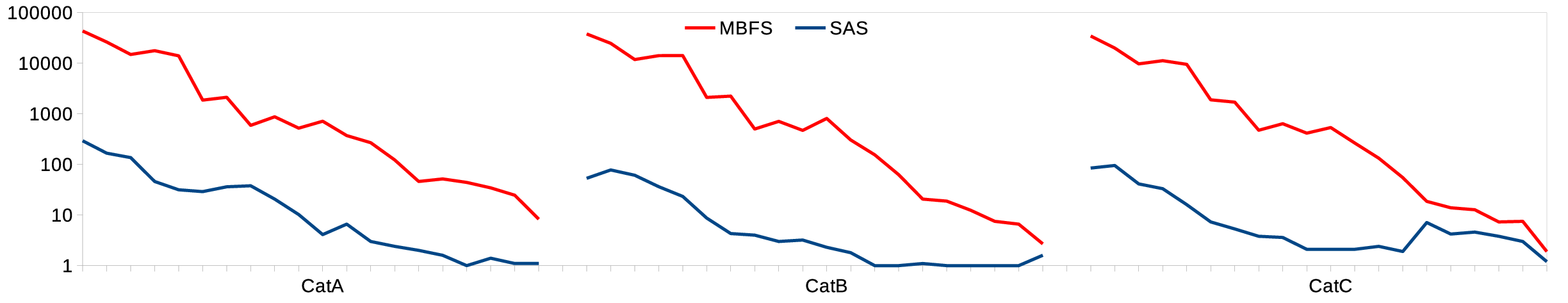}
    \caption{Numbers of restarts (y-axis) on problem instances (x-axis) when MBFS and SAS are used along with PGCH} \label{fig:mbfs-sas-restarts}
    
\end{figure}

Interestingly, as per \figurename~\ref{fig:mbfs-sas-rdi}, RDI values obtained by using MBFS are significantly higher than those obtained by using SAS. The p-value of the Wilcoxon Signed Rank Test is $0.00001$. To understand this apparent anomaly, in \figurename~\ref{fig:mbfs-sas-restarts}, we compare numbers of restarts i.e. the numbers of iterations the outer loop in Function {TTPS} in Algorithm~\ref{algo:MainSearchFramework} runs with
MBFS or SAS, along with PGCH in {TSPS} of course. We see that MBFS leads to a huge numbers of restarts compared to what SAS leads to. This indicates that via more restarts, MBFS leads to greater diversity and eventually better RDI values while SAS spends time in the simulated annealing process and does not get good RDI values. We further reason that with targeted search, MBFS converges quickly to local optima and thus resorts to restarts more often while SAS solely depends on diminishing probabilities of accepting worse solutions and thus get out of local optima. Notice that the numbers of restarts get lower with the increase in the problem size. This is because in large problems, arguably only fewer or even no restarts could take place within a limited timeout of $10$ minutes.

\begin{table}  
\centering
\caption{Comparison of RDI values obtained by the proposed CoCoP and CoCoL solvers and those obtained by MATLS, S5, and CS2SA*. Emboldened values denote the best performers.}
\label{tab:sota-comparison}
\setlength{\tabcolsep}{4.5pt}
\begin{footnotesize}
\begin{tabular}{|l|rrrrr|rrrrr|rrrrr|}\hline
Problem & \multicolumn{5}{|c|}{CatA} & \multicolumn{5}{|c|}{CatB} & \multicolumn{5}{|c|}{CatC}\\\cline{2-16}
Instance	&	MATLS	&	S5	&	CS2SA*	&	CoCoP & CoCoL	&	MATLS	&	S5	&	CS2SA*	&	CoCoP & CoCoL	&	MATLS	&	S5	&	CS2SA*	&	CoCoP & CoCoL	\\\hline

eil76 & 72.2 & \textbf{100.0} & 15.4 & \textbf{100.0} & \textbf{100.0} & 95.4 & 80.3 & 20.1 & 95.7 & \textbf{99.4} & \textbf{95.5} & 86.9 & 74.5 & 91.3 & 90.7 \\ 
kroA100 & 47.1 & 69.9 & 9.9 & \textbf{85.8} & 81.6 & 49.0 & 92.2 & 12.1 & \textbf{97.7} & 95.1 & 67.9 & 71.1 & 33.2 & \textbf{99.8} & 94.6 \\ 
ch130 & 57.6 & 88.2 & 45.7 & \textbf{96.7} & 96.5 & 92.7 & 94.6 & 19.1 & \textbf{98.8} & \textbf{98.8} & 66.5 & 96.1 & 20.4 & 95.6 & \textbf{97.3} \\ 
u159 & 74.9 & 79.0 & 57.1 & 88.5 & \textbf{91.5} & 18.1 & 89.7 & 17.7 & 93.0 & \textbf{95.3} & 16.0 & 55.0 & 50.0 & \textbf{82.2} & 80.7 \\ 
a280 & 42.5 & 88.9 & 52.4 & \textbf{97.4} & 91.2 & 31.9 & 31.3 & 54.2 & \textbf{100.0} & 99.9 & 69.0 & 98.3 & 41.4 & \textbf{100.0} & 99.7 \\ 
u574 & 57.3 & 77.1 & 33.2 & 84.4 & \textbf{93.5} & 53.1 & 43.5 & 22.1 & 92.6 & \textbf{95.7} & 92.0 & 93.8 & 63.8 & \textbf{98.8} & 96.1 \\ 
u724 & 66.4 & 83.5 & 44.2 & 95.4 & \textbf{96.4} & 13.9 & 44.3 & 36.5 & \textbf{97.5} & 91.8 & 56.6 & 58.3 & 25.0 & 85.7 & \textbf{92.9} \\ 
dsj1000 & 85.7 & 3.1 & \textbf{100.0} & \textbf{100.0} & \textbf{100.0} & 66.7 & 69.1 & 34.8 & 96.1 & \textbf{98.2} & 84.5 & 88.3 & 42.3 & 98.3 & \textbf{96.1} \\ 
rl1304 & 23.9 & 90.3 & 16.8 & \textbf{98.5} & 94.6 & 23.7 & 54.9 & 31.4 & \textbf{94.3} & 91.7 & 72.2 & 75.1 & 39.2 & \textbf{99.3} & 96.5 \\ 
fl1577 & 65.4 & 95.1 & 29.7 & 95.0 & \textbf{96.6} & 68.6 & 78.4 & 45.9 & \textbf{96.2} & \textbf{96.2} & 79.2 & 84.7 & 42.4 & \textbf{92.5} & 90.1 \\ 
d2103 & 1.8 & 82.9 & 62.3 & \textbf{93.7} & 92.5 & 39.5 & 67.2 & 53.7 & \textbf{96.6} & 96.3 & 27.9 & 48.3 & 17.9 & \textbf{95.8} & 85.9 \\ 
pcb3038 & 33.0 & 89.6 & 23.8 & 95.2 & \textbf{96.3} & 60.7 & 70.4 & 60.9 & \textbf{97.6} & 95.8 & 69.9 & 79.3 & 71.2 & \textbf{97.4} & 86.6 \\ 
fnl4461 & 32.7 & 87.4 & 5.1 & \textbf{96.7} & 91.4 & 22.1 & 39.4 & 34.4 & \textbf{86.9} & 82.2 & 70.2 & 68.1 & 63.3 & \textbf{96.5} & 91.6 \\ 
pla7397 & 78.2 & 95.6 & 43.0 & \textbf{98.1} & 97.1 & 74.9 & 83.2 & 49.9 & \textbf{97.7} & 93.3 & 66.5 & 76.7 & 49.1 & \textbf{95.2} & 89.3 \\ 
rl11849 & 32.7 & 89.9 & 8.9 & \textbf{98.4} & 94.2 & 20.9 & 25.2 & 32.4 & \textbf{87.4} & 67.5 & 53.5 & 43.4 & 44.9 & 68.8 & \textbf{75.1} \\ 
usa13509 & 57.1 & 94.4 & 23.2 & \textbf{95.5} & 93.9 & 57.0 & 65.4 & 58.4 & \textbf{90.5} & 82.8 & 83.0 & 83.3 & 81.8 & 93.5 & \textbf{95.9} \\ 
brd14051 & 28.1 & 88.8 & 11.3 & 94.1 & \textbf{95.3} & 72.1 & 76.8 & 75.3 & \textbf{93.2} & 92.7 & 47.6 & 49.3 & 54.0 & \textbf{80.6} & 70.3 \\ 
d15112 & 25.2 & 78.3 & 13.2 & 91.7 & \textbf{95.4} & 14.7 & 28.4 & 62.5 & 79.1 & \textbf{80.9} & 11.1 & 27.0 & 63.8 & \textbf{89.3} & 88.8 \\ 
d18512 & 60.7 & 92.7 & 18.7 & \textbf{97.7} & 95.5 & 73.9 & 76.5 & 74.2 & \textbf{92.9} & 89.9 & 34.3 & 31.9 & 55.5 & \textbf{83.9} & 59.3 \\ 
pla33810 & 27.7 & 87.0 & 19.1 & 93.2 & \textbf{94.0} & 70.1 & 77.5 & 42.5 & \textbf{96.8} & 89.9 & 74.3 & 56.2 & 40.0 & \textbf{93.1} & 84.4 \\ 

\hline
\end{tabular}
\end{footnotesize}
\end{table}  

\subsubsection{Comparison with MATLS, S5, and CS2SA* Solvers}
\label{sec:sotaComparison}

We name our final TTP solver as Cooperative Coordination (CoCo) and based on the experimental results presented so far, we obtain two CoCo versions. These two versions are PGCH+MBFS and LGCH+MBFS, and for the rest of the paper, we respectively name them as CoCoP and CoCoL.

We compare our CoCoP and CoCoL solvers with three existing state-of-the-art TTP solvers such as MATLS~\cite{mei2014improving}, S5~\cite{faulkner2015approximate} and
CS2SA*~\cite{el2018efficiently}. CS2SA* is selected because our TTP search framework in Algorithms~\ref{algo:MainSearchFramework} and \ref{algo:AuxSearchFramework} is similar to its cooperational coevoluation approach. MATLS and S5 are selected due to their salient performance reported in~\cite{wagner2017case}.
The source code for CS2SA* and MATLS has been obtained from the corresponding authors. We have reconstructed S5 ourselves and S5 does not have any parameters to be tuned.

\paragraph{CS2SA* and Recent Descendants}
After CS2SA*~\cite{el2018efficiently}, two further TTP methods \cite{maity2020efficient,zhang2021solving} have been reported.
Below are several observations about these methods.

\begin{itemize}
    \item {\bf CS2SA*~\cite{el2018efficiently}:} It is reported in~\cite{wuijts2019investigation} that CS2SA* and its precursors incorrectly present the objective values by taking the rounded values of the distances between cities. This is different from the definition of TTP benchmark instances \cite{polyakovskiy2014comprehensive}. As such this makes CS2SA* incomparable with other TTP methods. \cite{wagner2017case} reports more issues with the precursor of CS2SA*. Further to these, while investigating the source code of CS2SA*, we have observed that it uses the same stored high quality TSP tour in each run and mainly focuses on improving the collection plan. This partially explains why its precursor \cite{el2016population} somewhat misleadingly concludes that the KP component of the TTP is more critical compared to the TSP component for optimisation while our effort in the TSP component shows otherwise. Nevertheless, using the same TSP tour in each run of a TTP method does not conform to the standard practice in empirical evaluation of methods that have stochasticity in decision making. For a fair comparison in this paper, when we run CS2SA* in our experiments, we compute the objective values correctly and also use different TSP tour in each run.
    \item {\bf A CS2SA* descendant \cite{maity2020efficient}:} This method follows the same incomparable empirical evaluation style of CS2SA*~\cite{el2018efficiently}. This method more explicitly shows that it is a fixed tour method. Moreover, its evaluation is based on only 9 benchmark instances. As such, we do not compare our proposed TTP solvers with this method.
    \item {\bf Another CS2SA* descendant \cite{zhang2021solving}:} This method follows the same incomparable experiment setup as CS2SA*~\cite{el2018efficiently} does. Unfortunately, its source code is not available. Moreover, while making an attempt to reconstruct this method and to run as we do with CS2SA*, we could not find necessary details in its corresponding published article.
    The pseudocode is unclear and appears to have issues that include ({\it i}) by definition item scores cannot be negative but pseudocode has conditions on that, ({\it ii}) the loop does not terminate unless the knapsack is full but practically it might be partially filled, and ({\it iii}) items are sorted by their scores but are picked mainly in the order of the cities. As a result of all these, we do not compare our proposed TTP solvers with this method.
\end{itemize}

\tablename~\ref{tab:sota-comparison} shows the RDI values obtained by CoCoP, CoCoL, MATLS, S5, and CS2SA* solvers. From the table, we see that CoCoP performs better than CoCoL. Both CoCoP and CoCoL outperform the other three solvers in almost all problem instances in all three categories. Moreover, S5 performs the third best but with a big difference with CoCoP and CoCoL while CS2SA* is the worst performer. The 95\% confidence interval plots of the RDI values in \figurename~\ref{fig:ConfInterval} also shows the statistical significance of the performance differences. More specifically, the p-value for Wilcoxon Signed Rank Test on the RDI values obtained by CoCoP and S5 is 0.00001 and by CoCoL and S5 is also the same. These indicate very highly significant differences. The overlapping intervals of CoCoP and CoCoL shows that their performance difference is statistically not significant. Nevertheless, \tablename{s}~\ref{tab:sotacata}, \ref{tab:sotacatb}, and \ref{tab:sotacatc} in the appendix provide further details on the objective values obtained by various solvers.

\begin{figure}[!b]
\centering
\includegraphics[width=0.6\textwidth]{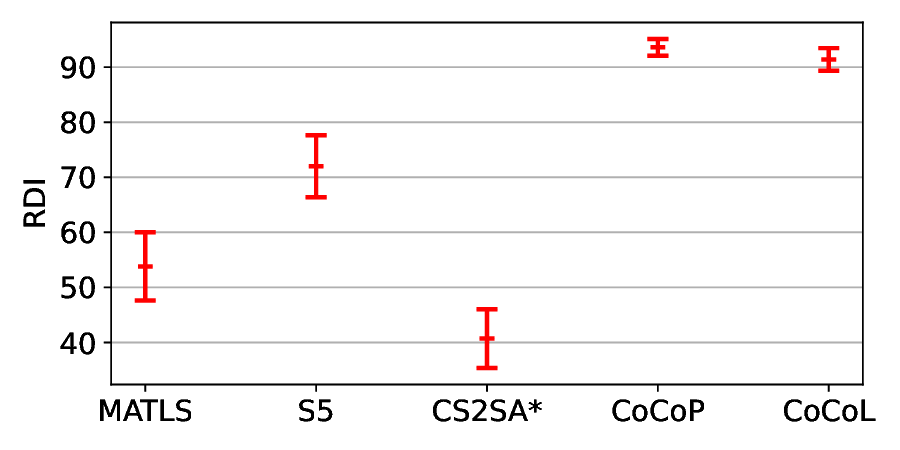}
\caption{95\% confidence intervals for our proposed CoCo solver and existing state-of-the-art TTP solvers such as MATLS, S5, CS2SA*, and CoCo. Overlapping confidence intervals mean the performance differences are not significant.}
\label{fig:ConfInterval}

\end{figure}

\begin{figure}[!b]
\centering
\includegraphics[width=0.6\textwidth]{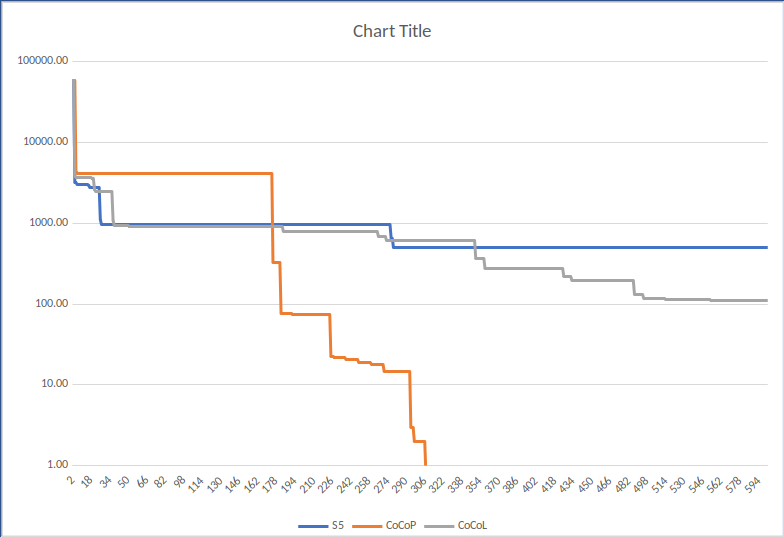}
\caption{Sample changes in best objectives (y-axis)  in each second (x-axis) by the best performing three solvers S5, CoCoP and CoCoL on CatC pla33810 instance. For better visual representation, plotted values are actually the maximum objective value obtained by any of the three solvers minus the objective value obtained by the respective solvers at the respective timepoints. Moreover, the logarithmic scale in the y-axis is used. So the lower the better in the chart although TTP is by definition a maximisation problem.}
\label{fig:BestSolChanges}

\end{figure}

\figurename~\ref{fig:BestSolChanges} shows that in sample runs of S5, CoCoP and CoCoL on CatC pla33810 instance, CoCoP makes good progress before getting into the flat region. CoCoL shows a better trend than S5 but is worse than CoCoP. The difference in CoCoP and CoCoL is that  CoCoL's search relies on the pattern learnt from its training solutions which are arguably not very high quality and so its prediction does not help much when already further better solutions are found over time.

\tablename~\ref{tab:long-comparison} shows the performances of the three best solvers S5, CoCoP, and CoCoL, when the timeout is 1 hour instead of the standard of 10 minutes. We see that the three solvers perform similarly with the longer timeout as they do with the shorter timeout. This shows the consistency of their performance over the time horizon.

\begin{table}[!tbp]
\centering
\caption{Comparison of RDI values obtained by S5, CoCoP, and CoCoL solvers when 1-hour timeout is used instead of standard 10-minute timeout; all other settings remain the same. Emboldened values denote the best performers among four solvers.}
\label{tab:long-comparison}
\setlength{\tabcolsep}{4.5pt}
\begin{small}
\begin{tabular}{|l|rrr|rrr|rrr|}\hline
Problem & \multicolumn{3}{|c|}{CatA} & \multicolumn{3}{|c|}{CatB} & \multicolumn{3}{|c|}{CatC}\\\cline{2-10}
Instance	&	S5	&	CoCoP & CoCoL	&	S5	&	CoCoP & CoCoL	&	S5	&	CoCoP & CoCoL	\\\hline
eil76	&	\bf 100.0	&	\bf 100.0	&	\bf 100.0	&	89.0	&	92.2	&	\bf  96.0	&	84.1	&	\bf 100.0	&	37.0	\\
kroA100	&	44.1	&	95.9	&	\bf 97.3	&	29.6	&	\bf 89.2	&	50.1	&	86.7	&	\bf 99.2	&	97.4	\\
ch130	&	4.0	&	\bf 85.0	&	78.5	&	0.0	&	\bf 90.3	&	86.5	&	62.1	&	73.6	&	\bf 81.1	\\
u159	&	4.4	&	34.0	&	\bf 64.9	&	4.0	&	\bf 100.0	&	96.5	&	32.8	&	\bf 99.9	&	\bf 99.9	\\
a280	&	0.9	&	\bf 48.1	&	46.2	&	1.8	&	\bf 100.0	&	99.9	&	0.0	&	\bf 100.0	&	\bf 100.0	\\
u574	&	9.8	&	30.1	&	\bf 71.8	&	2.2	&	81.4	&	\bf 92.0	&	9.0	&	\bf 71.8	&	52.9	\\
u724	&	7.0	&	76.4	&	\bf 81.3	&	7.3	&	94.5	&	\bf 97.2	&	7.2	&	\bf 87.8	&	78.5	\\
dsj1000	&	5.1	&	\bf 100.0	&	\bf 100.0	&	3.9	&	\bf 88.7	&	84.1	&	26.7	&	\bf 65.6	&	58.3	\\
rl1304	&	19.9	&	\bf 70.3	&	55.7	&	14.6	&	\bf 80.7	&	80.6	&	28.2	&	76.7	&	\bf 81.3	\\
fl1577	&	67.6	&	66.7	&	\bf 70.7	&	13.9	&	\bf 66.7	&	63.4	&	24.3	&	31.2	&	\bf 58.2	\\
d2103	&	12.9	&	68.7	&	\bf 74.6	&	14.6	&	92.1	&	\bf 94.2	&	9.1	&	\bf 88.3	&	83.6	\\
pcb3038	&	9.6	&	55.0	&	\bf 66.4	&	15.0	&	88.3	&	\bf 89.4	&	16.3	&	\bf 81.2	&	71.2	\\
fnl4461	&	12.4	&	\bf 78.6	&	63.1	&	6.8	&	71.4	&	\bf 77.3	&	16.8	&	\bf 83.8	&	70.6	\\
pla7397	&	11.2	&	\bf 76.0	&	66.7	&	26.4	&	\bf 92.0	&	83.6	&	22.9	&	\bf 86.3	&	77.4	\\
rl11849	&	14.5	&	\bf 77.9	&	67.6	&	10.1	&	75.2	&	\bf 82.4	&	13.9	&	\bf 74.2	&	66.6	\\
usa13509	&	30.0	&	\bf 72.2	&	23.6	&	14.5	&	\bf 91.2	&	84.1	&	16.5	&	\bf 92.9	&	76.4	\\
brd14051	&	8.9	&	\bf 84.7	&	62.7	&	12.6	&	\bf 83.7	&	77.4	&	18.4	&	\bf 82.2	&	69.6	\\
d15112	&	13.5	&	\bf 74.5	&	64.1	&	9.5	&	\bf 82.8	&	79.7	&	12.6	&	\bf 85.1	&	77.7	\\
d18512	&	20.0	&	\bf 76.6	&	63.6	&	17.1	&	\bf 84.1	&	81.7	&	18.0	&	\bf 71.4	&	46.8	\\
pla33810	&	15.4	&	\bf 61.8	&	53.7	&	19.8	&	\bf 78.4	&	70.4	&	27.4	&	\bf 79.9	&	54.4	\\

\hline
\end{tabular}
\end{small}
\end{table}

\subsubsection{Comparison with a Recent Solver MEA2P}
\label{sec:mea2pComparison}

We compare our proposed best performing CoCoP solver with a recent TTP solver named MEA2P~\cite{wuijts2019investigation}. MEA2P is a steady state Memetic algorithm with Edge-Assembly~\cite{nagata2006new} and Two-Points crossover (EAX) operators. Like a number of other solvers \cite{wagner2016stealing,mei2015heuristic,martins2017hseda,el2018hyperheuristic}, MEA2P is targeted to solve small TTP instances. For its initial population, MEA2P generates 50 solutions, each with a random cyclic tour and an empty collection plan. Then, in each of its 2500 iterations, MEA2P generate a new solution by combining two randomly selected solutions using the edge-assembly crossover operator~\cite{nagata2006new} on the cyclic tours and the two-point crossover operator on the collection plans. The initial solutions and the subsequently generated combined solutions are improved using a local search method that uses {2OPT}~\cite{croes1958method},
node insertion~\cite{faulkner2015approximate},
bit-flip~\cite{polyakovskiy2014comprehensive,faulkner2015approximate}
and item exchange~\cite{mei2016investigation} moves in an interleaving fashion.

MEA2P demands heavy computation time particularly in large problems. Therefore, for a meaningful comparison, instead of running for $10$ minutes, we run both MEA2P and CoCoP with a termination criterion of 2500 restarts for each TTP instance. Also, we use only the 8 small instances from each of the three categories. For large instances MEA2P takes hours and days.
As we see, this experiment setting is different from the settings in other earlier experiments presented in this paper.

\begin{table}  
\caption{Comparison of average execution times and RDI values of MEA2P and CoCoP on 8 instances in each category.}
\label{tab:mea2p}
\centering
\setlength{\tabcolsep}{2pt}
\begin{small}
\begin{tabular}{|l|r||rr|rr||rr|rr||rr|rr|}\hline
TTP     &  \% Unique     & \multicolumn{4}{c||}{CatA} & \multicolumn{4}{c||}{CatB} & \multicolumn{4}{c|}{CatC}\\\cline{3-14}
Problem            & CLK Init       & \multicolumn{2}{c|}{Avg Time}  & \multicolumn{2}{c||}{RDI}       &  \multicolumn{2}{c|}{Avg Time} & \multicolumn{2}{c||}{RDI} & \multicolumn{2}{c|}{Avg Time} & \multicolumn{2}{c|}{RDI}\\\cline{3-14}
Instance	& Solutions	&	MEA2P	&	CoCoP	&	MEA2P	&	CoCoP	&	MEA2P	&	CoCoP	&	MEA2P	&	CoCoP	&	MEA2P	&	CoCoP	&	MEA2P	&	CoCoP	\\\hline
eil76	&	3.3	&	45s	&	27s	&	43.0	&	72.0	&	2.3m	&	29s	&	\bf 97.0	&	32.0	&	4.7m	&	40s	&	\bf 100.0	&	48.9	\\
kroA100	&	1.2	&	93.6s	&	49s	&	\bf 94.4	&	0.0	&	5m	&	51s	&	\bf 99.4	&	50.4	&	10m	&	1.1m	&	\bf 100.0	&	0.5	\\
ch130	&	4.9	&	3m	&	1.7m	&	\bf 84.1	&	54.3	&	9.8m	&	1.6m	&	\bf 92.9	&	15.8	&	20.3m	&	2.1m	&	\bf 79.7	&	39.6	\\
u159	&	4.4	&	7.2m	&	1.2m	&	\bf 71.2	&	0.0	&	16.8m	&	1.3m	&	\bf 78.0	&	5.8	&	31.1m	&	1.7m	&	\bf 100.0	&	4.8	\\
a280	&	80.3	&	31.8m	&	1.7m	&	\bf 54.4	&	7.0	&	1.3h	&	1.9m	&	56.4	&	\bf 73.9	&	2h	&	2.7m	&	73.4	&	\bf 100.0	\\
u574	&	65.1	&	4.5h	&	10.6m	&	\bf 61.1	&	10.4	&	12.9h	&	10.4m	&	\bf 40.0	&	1.9	&	25.2h	&	12.6m	&	47.6	&	\bf 73.1	\\
u724	&	90.2	&	6.5h	&	9.2m	&	30.9	&	\bf 89.1	&	1.2d	&	10.4m	&	44.7	&	\bf 52.6	&	2.1d	&	13m	&	48.5	&	\bf  78.5	\\
dsj1000	&	87.3	&	5.9h	&	39.3m	&	9.8	&	\bf 100.0	&	3.4d	&	44.3m	&	\bf 40.9	&	30.2	&	6.3d	&	44.1m	&	58.3	&	\bf 81.0	\\\hline
\end{tabular}
\end{small}
\end{table}  

\tablename~\ref{tab:mea2p} shows the average execution times and the RDI values obtained by MEA2P and CoCoP on 8 small instances. Moreover, \tablename{s}~\ref{tab:mea2pcata}, \ref{tab:mea2pcatb}, and \ref{tab:mea2pcatc} in the appendix provide further details on the execution times and the objective values obtained by the two solvers. Nevertheless, from these tables, we see that MEA2P runs in the scale of hours and days while CoCoP runs in the scale of seconds and minutes. Overall, MEA2P takes a number of times the execution time of CoCoP. In RDI values, MEA2P achieves very good performance in small instances while CoCoP achieves so in large instances. We further investigate the reasons behind such performance.
MEA2P is a population based algorithm that aims to maintain diversity by starting from random solutions, keeping a number of solutions in its population, and using combination operators.
In small instances, MEA2P affords the time to explore the search space to a large extent and obtains better objective values.
However, CoCoP is a single-solution based search algorithm that depends on Chained Lin-Kernighan (CLK) heuristic~\cite{applegate2003chained} for initial cyclic tours,
and PackIterative~\cite{faulkner2015approximate} and Insertion~\cite{mei2014improving} methods for initial collection plans.
So the greater diversity needs to come from the search restart or from the initial solution generators.
In \tablename~\ref{tab:mea2p} Column 2 (title ``\% Unique CLK Init Solutions''), we show the relative unique initial cyclic tours found by the CLK heuristic.
These numbers essentially help us explain that CoCoP performs better when CLK generates large numbers of unique
initial cyclic tours, which is more usual in large instances than in small ones.

\paragraph{Comments on a Recent Method Presented in \cite{nikfarjam2022use}}

For convenience, we use NNN to refer to the recent TTP method presented in \cite{nikfarjam2022use}. Upon careful consideration, we do not compare the proposed method with NNN. There is considerable overlap between NNN and MEA2P~\cite{wuijts2019investigation}, and as such a comparison against NNN appears redundant. Furthermore, the results obtained by NNN do not appear to have compelling advantages. The detailed reasons are further discussed below.

\begin{enumerate}
\item NNN and MEA2P are both evolutionary algorithms. Both use EAX crossover operators on tours to generate neighbour TTP solutions. The only difference between the two methods is that NNN keeps the current generation in a structured form while MEA2P uses a flat one-dimensional form.

\item NNN neither provides comparisons with the most relevant MEA2P method nor does it cite MEA2P, even though the two methods are ostensibly very similar. Moreover, NNN uses problem instances that are mostly different from what MEA2P uses. Looking at the common instances, MEA2P performs better on a280 instances, while NNN performs better on eil51 instances. Based on the presented results, it is unclear whether NNN is actually better than MEA2P. In this paper, we have already shown MEA2P performs better than the proposed method on small instances, while the proposed method is better on large instances.

\item NNN uses dynamic programming for KP but only for tiny instances with at most 280 cities. In contrast, our benchmark instances have the numbers of cities in the range of 76 to 33810. For larger instances (maximum 4461 cities), NNN uses a bitflip local search method instead of dynamic programming. This indicates that dynamic programming does not scale up in large problem instances. Indeed, the paper on NNN also states that.
\end{enumerate}

\subsubsection{Best Objective Values Obtained}

\tablename~\ref{tab:bests} shows the best objective values obtained by the CoCo variants against those obtained by other existing solvers when running for 10 minutes. The best objective values for other solvers are obtained from the results in Section~\ref{sec:sotaComparison}, from the results reported in~\cite{wuijts2019investigation}, and also the results reported in~\cite{wagner2017case} (excluding the results of CS2SA solver~\cite{el2016population} due to a faulty evaluation in it 
as reported in~\cite{wuijts2019investigation}). Notice that CoCo variants obtain new best results on most large problem instances.

\begin{table}[!tbp]
\caption
  {Best objective values obtained by CoCo variants and other algorithms, each running for $10$ minutes on each instance in each of the three categories. The new best objective values obtained are in boldface.}
\label{tab:bests}

\small
\centering
\begin{tabular}{|l|r|r||r|r||r|r|}
\hline
\multirow{3}{*}{\bf Instance} & \multicolumn{2}{c||}{\textbf{CatA}} & \multicolumn{2}{c||}{\textbf{CatB}} & \multicolumn{2}{c|}{\textbf{CatC}}\\  \cline{2-7}

& \multicolumn{1}{c|}{\textbf{CoCo}} & \multicolumn{1}{c||}{\textbf{Other}} & \multicolumn{1}{c|}{\textbf{CoCo}} & \multicolumn{1}{c||}{\textbf{Other}} & \multicolumn{1}{c|}{\textbf{CoCo}}&\multicolumn{1}{c|}{\textbf{Other}} \\

& \multicolumn{1}{c|}{\textbf{Variants}} & \multicolumn{1}{c||}{\textbf{Solvers}} & \multicolumn{1}{c|}{\textbf{Variants}} & \multicolumn{1}{c||}{\textbf{Solvers}} & \multicolumn{1}{c|}{\textbf{Variants}}&\multicolumn{1}{c|}{\textbf{Solvers}} \\ \hline \hline

eil76 & \textbf{4109} & \textbf{4109} & 22464 & \textbf{23278} & 88211 & \textbf{88386} \\
kroA100 & 4881 & \textbf{4976} & 45812 & \textbf{46633} & 159112 & \textbf{159135} \\
ch130 & 9632 & \textbf{9682} & 61842 & \textbf{62496} & \textbf{207902} & 207654 \\
u159 & 8979 & \textbf{9064} & \textbf{61077} & 60968 & \textbf{249875} & \textbf{249875} \\ \hline
a280 & \textbf{18702} & 18452 & \textbf{116458} & 115252 & \textbf{429138} & 429082 \\
u574 & \textbf{28282} & 27238 & \textbf{261515} & 257912 & \textbf{970343} & 969247 \\
u724 & \textbf{51427} & 50402 & \textbf{323123} & 313735 & \textbf{1209029} & 1200310 \\
dsj1000 & \textbf{144426} & 144219 & \textbf{372837} & 352185 & \textbf{1496922} & 1483610 \\ \hline
rl1304 & \textbf{81921} & 81376 & \textbf{602276} & 584957 & \textbf{2214091} & 2207470 \\
fl1577 & \textbf{94066} & 93861 & \textbf{639843} & 619577 & \textbf{2500736} & 2496440 \\
d2103 & \textbf{122902} & 121981 & \textbf{927992} & 899581 & \textbf{3501889} & 3453096 \\
pcb3038 & \textbf{162321} & 160733 & \textbf{1205850} & 1190198 & \textbf{4600973} & 4596672 \\ \hline
fnl4461 & \textbf{265322} & 263040 & \textbf{1653828} & 1631325 & \textbf{6575472} & 6563377 \\
pla7397 & \textbf{402199} & 395992 & \textbf{4485629} & 4452480 & \textbf{14572352} & 14304342 \\
rl11849 & \textbf{716458} & 709512 & \textbf{4823625} & 4690137 & \textbf{18569005} & 18394454 \\
usa13509 & \textbf{817069} & 810455 & \textbf{8343799} & 8137189 & \textbf{26728716} & 26626726 \\ \hline
brd14051 & \textbf{887033} & 882244 & \textbf{6854612} & 6844392 & \textbf{24361366} & 24239842 \\
d15112 & \textbf{975930} & 957409 & \textbf{7942036} & 7733280 & \textbf{27665466} & 27340647 \\
d18512 & \textbf{1088840} & 1074510 & \textbf{7582022} & 7515276 & \textbf{27951166} & 27748430 \\
pla33810 & \textbf{1928935} & 1910480 & \textbf{16332634} & 15898501 & \textbf{58900443} & 58292399 \\ \hline 
\end{tabular}
\label{}
\end{table}

\section{Conclusion}
\label{sec:conclusion}

A travelling thief problem (TTP) has profitable items scattered over cities and a thief rents a knapsack and performs a cyclic tour to collect some items and thus maximises the profit while minimises the travelling time and so the renting cost of the knapsack. Thus a TTP has two components: one component is like the travelling salesman problem (TSP) and the other component is like the knapsack problem (KP). TTP is computationally NP-Hard since both TSP and KP are NP-Hard. TTP is a proxy to many real-world problems such as waste collection and mail delivery.

TTP research has made significant progress lately. However, most existing TTP methods do not explicitly exploit the mutual dependency of the two components and thus lack proper coordination. In this paper, we show first that a simple local search based coordination approach does not work in TTP. We then propose one coordination heuristic for changing collection plans during cyclic tour exploration and another for explicitly exploiting cyclic tours during collection plan exploration. We further propose a machine learning based coordination heuristic that captures characteristics of the human designed coordination heuristics. Our proposed coordination based approaches help our TTP solver explore better TTP solutions within given timeout limit. Consequently our proposed solver named Cooperation Coordination (CoCo) significantly outperforms existing state-of-the-art TTP solvers on a set of benchmark problems. CoCo is available from \url{https://github.com/majid75/CoCo}.



\section*{Acknowledgments}

This research has been partly supported by Data61/CSIRO, Australia.
We would like to thank our colleagues Toby Walsh, Phil Kilby, and Regis Riveret
at Data61/CSIRO for discussions leading to the improvement of this article.
We also gratefully acknowledge the support of the Griffith University
eResearch Service \& Specialised Platforms Team and the use of the Gowonda high performance computing cluster.


\section*{Appendix}

\tablename{s}~\ref{tab:sotacata}, \ref{tab:sotacatb}, and \ref{tab:sotacatc} show various statistics of the objective values obtained by MATLS, S5, CS2SA*, CoCoP, and CoCoL solvers over $10$ runs of each solver on each TTP instance. \tablename{s}~\ref{tab:mea2pcata}, \ref{tab:mea2pcatb}, and \ref{tab:mea2pcatc} show various statistics of the objective values obtained by MEA2P, S5, CS2SA*, and CoCoP solvers over $10$ runs of each solver on each TTP instance. The mean, median, and standard deviations for each instance are computed over the runs of the same solver while the maximum and the minimum for each instance are computed over all the runs of all solvers that are compared together.

\setcounter{table}{0}
\renewcommand{\thetable}{A\arabic{table}}

\begin{table}[htbp]
\caption{The median, the mean, and the standard deviation (StdDev) of the objectives values obtained over $10$ runs of each of MATLS, S5, CS2SA*, CoCoP, and CoCoL solvers on each CatA instance. Emboldened values are the largest median, the largest mean, and the smallest standard deviation of the objective values over the five solvers. Moreover, the maximum and the minimum of the objectives values obtained are over all $50$ runs of all $5$ solvers; these maximums and minimums are used in RDI computation in \tablename~\ref{tab:sota-comparison}.}
\label{tab:sotacata}
\centering
\setlength{\tabcolsep}{2.5pt}
\vspace{-2ex}
\footnotesize
\begin{tabular}{|l|l|r|r|r|r|r|r|r|}
\hline
\bf CatA & \bf Metric & \bf MATLS & \bf S5 & \bf CS2SA* & CoCoP & CoCoL & \bf Min & \bf Max\\ \hline\hline

 & Median & 3655 & \textbf{4109} & 2697 & \textbf{4109} & \textbf{4109} &  &  \\
eil76 & Mean & 3711 & \textbf{4109} & 2900 & \textbf{4109} & \textbf{4109} & \multicolumn{1}{r|}{2679} & \multicolumn{1}{r|}{4109} \\
 & Sdev & 84 & \textbf{0} & 354 & \textbf{0} & \textbf{0} &  &  \\ \hline
 & Median & 4493 & 4699 & 4300 & \textbf{4783} & 4746 &  &  \\ 
kroA100 & Mean & 4540 & 4684 & 4304 & \textbf{4785} & 4758 & \multicolumn{1}{r|}{4241} & \multicolumn{1}{r|}{4875} \\ 
 & Sdev & 100 & 57 & \textbf{56} & 92 & 64 &  &  \\ \hline
 & Median & 8799 & 9404 & 8381 & \textbf{9564} & 9560 &  &  \\
ch130 & Mean & 8799 & 9400 & 8565 & \textbf{9567} & 9564 & \multicolumn{1}{r|}{7668} & \multicolumn{1}{r|}{9632} \\ 
 & Sdev & \textbf{0} & 11 & 568 & 10 & 39 &  &  \\ \hline
 & Median & 8583 & 8634 & 8459 & 8763 & \textbf{8792} &  &  \\ 
u159 & Mean & 8579 & 8634 & 8337 & 8763 & \textbf{8805} & \multicolumn{1}{r|}{7562} & \multicolumn{1}{r|}{8920} \\
 & Sdev & 54 & \textbf{0} & 279 & \textbf{0} & 40 &  &  \\ \hline
 & Median & 17706 & 18418 & 17878 & \textbf{18556} & 18437 &  &  \\ 
a280 & Mean & 17649 & 18420 & 17814 & \textbf{18561} & 18458 & \multicolumn{1}{r|}{16943} & \multicolumn{1}{r|}{18605} \\ 
 & Sdev & 151 & \textbf{13} & 503 & 21 & 34 &  &  \\ \hline
 & Median & 26279 & 27083 & 24886 & 27383 & \textbf{27915} &  &  \\ 
u574 & Mean & 26017 & 27069 & 24735 & 27453 & \textbf{27937} & \multicolumn{1}{r|}{22973} & \multicolumn{1}{r|}{28282} \\ 
 & Sdev & 529 & \textbf{69} & 1397 & 205 & 155 &  &  \\ \hline
 & Median & 49097 & 50346 & 47909 & 51109 & \textbf{51128} &  &  \\
u724 & Mean & 49223 & 50340 & 47774 & 51120 & \textbf{51185} & \multicolumn{1}{r|}{44882} & \multicolumn{1}{r|}{51419} \\
 & Sdev & 818 & \textbf{19} & 1152 & 65 & 166 &  &  \\ \hline
 & Median & 143280 & 137889 & \textbf{144219} & \textbf{144219} & \textbf{144219} &  &  \\ 
dsj1000 & Mean & 143280 & 137866 & \textbf{144219} & \textbf{144219} & \textbf{144219} & \multicolumn{1}{r|}{137661} & \multicolumn{1}{r|}{144219} \\
 & Sdev & \textbf{0} & 134 & \textbf{0} & \textbf{0} & \textbf{0} &  &  \\ \hline
 & Median & 74799 & 81111 & 74465 & \textbf{81764} & 81411 &  &  \\ 
rl1304 & Mean & 75159 & 81018 & 74527 & \textbf{81737} & 81393 & \multicolumn{1}{r|}{73049} & \multicolumn{1}{r|}{81874} \\
 & Sdev & 913 & 443 & 1183 & \textbf{80} & 327 &  &  \\ \hline
 & Median & 88254 & 93337 & 83427 & 93571 & \textbf{93585} &  &  \\ 
fl1577 & Mean & 88376 & 93260 & 82522 & 93236 & \textbf{93511} & \multicolumn{1}{r|}{77636} & \multicolumn{1}{r|}{94066} \\
 & Sdev & \textbf{235} & 641 & 2747 & 623 & 434 &  &  \\ \hline
 & Median & 112959 & 120691 & 119142 & \textbf{121863} & 121777 &  &  \\
d2103 & Mean & 112894 & 120852 & 118834 & \textbf{121914} & 121794 & \multicolumn{1}{r|}{112720} & \multicolumn{1}{r|}{122534} \\
 & Sdev & \textbf{90} & 339 & 2017 & 317 & 196 &  &  \\ \hline
 & Median & 148429 & 160189 & 147283 & 161322 & \textbf{161540} &  &  \\ 
pcb3038 & Mean & 148610 & 160203 & 146741 & 161336 & \textbf{161558} & \multicolumn{1}{r|}{141872} & \multicolumn{1}{r|}{162321} \\
 & Sdev & 1955 & \textbf{235} & 2349 & 411 & 401 &  &  \\ \hline
 & Median & 248404 & 262174 & 240835 & \textbf{264383} & 263047 &  &  \\ 
fnl4461 & Mean & 248003 & 262090 & 240884 & \textbf{264460} & 263102 & \multicolumn{1}{r|}{239569} & \multicolumn{1}{r|}{265322} \\ 
 & Sdev & 1043 & 381 & 957 & \textbf{354} & 402 &  &  \\ \hline
 & Median & 367017 & 395132 & 315595 & \textbf{398131} & 396938 &  &  \\ 
pla7397 & Mean & 369410 & 394809 & 317847 & \textbf{398464} & 397005 & \multicolumn{1}{r|}{254937} & \multicolumn{1}{r|}{401252} \\
 & Sdev & 4543 & \textbf{1117} & 36889 & 1434 & 1957 &  &  \\ \hline
 & Median & 665469 & 707624 & 647432 & \textbf{714405} & 711001 &  &  \\
rl11849 & Mean & 664653 & 707727 & 646646 & \textbf{714188} & 711019 & \multicolumn{1}{r|}{639965} & \multicolumn{1}{r|}{715379} \\ 
 & Sdev & 4205 & 1067 & 3356 & \textbf{1050} & 2053 &  &  \\ \hline
 & Median & 748684 & 808425 & 693102 & \textbf{809100} & 807049 &  &  \\ 
usa13509 & Mean & 748957 & 808104 & 695025 & \textbf{809986} & 807339 & \multicolumn{1}{r|}{658211} & \multicolumn{1}{r|}{817069} \\
 & Sdev & 2706 & \textbf{1546} & 26436 & 3293 & 1588 &  &  \\ \hline
 & Median & 818107 & 874648 & 803894 & 879175 & \textbf{880860} &  &  \\ 
brd14051 & Mean & 817558 & 874722 & 801686 & 879741 & \textbf{880868} & \multicolumn{1}{r|}{791052} & \multicolumn{1}{r|}{885314} \\
 & Sdev & 4943 & 3487 & 5732 & 3055 & \textbf{2193} &  &  \\ \hline
 & Median & 884334 & 946065 & 871629 & 962704 & \textbf{968380} &  &  \\
d15112 & Mean & 886474 & 947268 & 872601 & 962738 & \textbf{966935} & \multicolumn{1}{r|}{857514} & \multicolumn{1}{r|}{972207} \\ 
 & Sdev & 12604 & 5182 & 11641 & 6888 & \textbf{5027} &  &  \\ \hline
 & Median & 997307 & 1071275 & 882375 & \textbf{1082018} & 1076736 &  &  \\
d18512 & Mean & 997293 & 1070819 & 900469 & \textbf{1082451} & 1077209 & \multicolumn{1}{r|}{857514} & \multicolumn{1}{r|}{1087677} \\ 
 & Sdev & 3009 & \textbf{1952} & 42711 & 2751 & 4929 &  &  \\ \hline
 & Median & 1721610 & 1895380 & 1717794 & 1912398 & \textbf{1913334} &  &  \\
pla33810 & Mean & 1730589 & 1893163 & 1707087 & 1910419 & \textbf{1912406} & \multicolumn{1}{r|}{1654707} & \multicolumn{1}{r|}{1928935} \\ 
 & Sdev & 23663 & 15501 & 26741 & \textbf{7496} & 13099 &  &  \\ \hline

\end{tabular}
\end{table}

\begin{table}[htbp]
\caption{The median, the mean, and the standard deviation (StdDev) of the objectives values obtained over $10$ runs of each of MATLS, S5, CS2SA*, CoCoP, and CoCoL solvers on each CatB instance. Emboldened values are the largest median, the largest mean, and the smallest standard deviation of the objective values over the five solvers. Moreover, the maximum and the minimum of the objectives values obtained are over all $50$ runs of all $5$ solvers; these maximums and minimums are used in RDI computation in \tablename~\ref{tab:sota-comparison}.}
\label{tab:sotacatb}
\centering
\vspace{-2ex}
\setlength{\tabcolsep}{2pt}

\footnotesize
\begin{tabular}{|l|l|r|r|r|r|r|r|r|}
\hline
\bf CatB & \bf Metric & \bf MATLS & \bf S5 & \bf CS2SA* & \bf CoCoP & \bf CoCoL & \bf Min & \bf Max\\ \hline\hline

 & Median & 22357 & 21616 & 19209 & 22312 & \textbf{22443} &  &  \\
eil76 & Mean & 22278 & 21669 & 19236 & 22290 & \textbf{22440} & \multicolumn{1}{r|}{18421} & \multicolumn{1}{r|}{22464} \\ 
 & Sdev & 330 & 247 & 581 & 35 & \textbf{7} &  &  \\ \hline
 & Median & 42303 & 45687 & 40585 & \textbf{45812} & \textbf{45812} &  &  \\ 
kroA100 & Mean & 42478 & 45300 & 40059 & \textbf{45662} & 45492 & \multicolumn{1}{r|}{39271} & \multicolumn{1}{r|}{45812} \\ 
 & Sdev & 522 & 952 & 679 & \textbf{240} & 856 &  &  \\ \hline
 & Median & 61053 & 61241 & 51270 & \textbf{61703} & 61698 &  &  \\ 
ch130 & Mean & 61023 & 61241 & 52823 & 61702 & \textbf{61712} & \multicolumn{1}{r|}{50695} & \multicolumn{1}{r|}{61842} \\ 
 & Sdev & 87 & \textbf{0} & 3369 & 1 & 82 &  &  \\ \hline
 & Median & 58000 & 60550 & 58090 & 60718 & \textbf{60920} &  &  \\ 
u159 & Mean & 58105 & 60693 & 58090 & 60814 & \textbf{60899} & \multicolumn{1}{r|}{57450} & \multicolumn{1}{r|}{61067} \\ 
 & Sdev & 793 & 197 & \textbf{0} & 171 & 162 &  &  \\ \hline
 & Median & 108803 & 109925 & 112210 & \textbf{116455} & 116446 &  &  \\ 
a280 & Mean & 109996 & 109938 & 112115 & \textbf{116453} & 116444 & \multicolumn{1}{r|}{106969} & \multicolumn{1}{r|}{116458} \\ 
 & Sdev & 2431 & 24 & 2817 & \textbf{8} & 13 &  &  \\ \hline
 & Median & 253938 & 252504 & 249741 & 259990 & \textbf{260393} &  &  \\ 
u574 & Mean & 253870 & 252378 & 249049 & 260008 & \textbf{260481} & \multicolumn{1}{r|}{245615} & \multicolumn{1}{r|}{261154} \\ 
 & Sdev & 3025 & 387 & 2235 & \textbf{57} & 250 &  &  \\ \hline
 & Median & 300806 & 308549 & 306002 & \textbf{321128} & 319219 &  &  \\ 
u724 & Mean & 300875 & 308106 & 306247 & \textbf{320751} & 319403 & \multicolumn{1}{r|}{297562} & \multicolumn{1}{r|}{321343} \\ 
 & Sdev & 2991 & 910 & 5463 & \textbf{645} & 753 &  &  \\ \hline
 & Median & 341668 & 344626 & 323717 & 369395 & \textbf{370866} &  &  \\ 
dsj1000 & Mean & 342340 & 344559 & 313118 & 369270 & \textbf{371224} & \multicolumn{1}{r|}{281196} & \multicolumn{1}{r|}{372837} \\ 
 & Sdev & 5942 & 1667 & 17236 & \textbf{265} & 689 &  &  \\ \hline 
 & Median & 564558 & 580186 & 569745 & \textbf{599321} & 598310 &  &  \\ 
rl1304 & Mean & 565251 & 580406 & 568978 & \textbf{599535} & 598248 & \multicolumn{1}{r|}{553759} & \multicolumn{1}{r|}{602276} \\ 
 & Sdev & 8069 & 3257 & 7501 & 1496 & \textbf{853} &  &  \\ \hline
 & Median & 602111 & 612594 & 590353 & \textbf{627094} & 626674 &  &  \\ 
fl1577 & Mean & 603575 & 611938 & 584300 & \textbf{627066} & 627032 & \multicolumn{1}{r|}{545224} & \multicolumn{1}{r|}{630291} \\ 
 & Sdev & 8700 & 6433 & 21683 & \textbf{690} & 1312 &  &  \\ \hline 
 & Median & 848987 & 884576 & 879209 & \textbf{923371} & 923098 &  &  \\ 
d2103 & Mean & 850689 & 886123 & 868850 & \textbf{923681} & 923201 & \multicolumn{1}{r|}{800191} & \multicolumn{1}{r|}{927992} \\ 
 & Sdev & 8676 & 9347 & 26974 & \textbf{1703} & 1975 &  &  \\ \hline
 & Median & 1171195 & 1179155 & 1172570 & \textbf{1201762} & 1199959 &  &  \\ 
pcb3038 & Mean & 1170233 & 1178436 & 1170399 & \textbf{1201624} & 1200100 & \multicolumn{1}{r|}{1118448} & \multicolumn{1}{r|}{1203695} \\ 
 & Sdev & 6489 & 3664 & 20839 & \textbf{1444} & 2203 &  &  \\ \hline
 & Median & 1616550 & 1625875 & 1627208 & \textbf{1648718} & 1643725 &  &  \\ 
fnl4461 & Mean & 1617028 & 1625227 & 1622874 & \textbf{1647646} & 1645435 & \multicolumn{1}{r|}{1606609} & \multicolumn{1}{r|}{1653828} \\ 
 & Sdev & 4283 & 2593 & 8776 & \textbf{2401} & 4972 &  &  \\ \hline
 & Median & 4322850 & 4325905 & 4156569 & \textbf{4475452} & 4445034 &  &  \\ 
pla7397 & Mean & 4278065 & 4346919 & 4071181 & \textbf{4466713} & 4430468 & \multicolumn{1}{r|}{3657856} & \multicolumn{1}{r|}{4485629} \\ 
 & Sdev & 138799 & 50782 & 237986 & \textbf{20194} & 29484 &  &  \\ \hline
 & Median & 4610620 & 4613870 & 4631178 & \textbf{4786319} & 4732950 &  &  \\ 
rl11849 & Mean & 4606189 & 4618037 & 4637653 & \textbf{4788987} & 4734139 & \multicolumn{1}{r|}{4548692} & \multicolumn{1}{r|}{4823625} \\ 
 & Sdev & \textbf{18154} & 19855 & 42376 & 29228 & 28800 &  &  \\ \hline
 & Median & 7827255 & 7938200 & 7944378 & \textbf{8226438} & 8142993 &  &  \\ 
usa13509 & Mean & 7827089 & 7927507 & 7844064 & \textbf{8229873} & 8137433 & \multicolumn{1}{r|}{7141090} & \multicolumn{1}{r|}{8343799} \\ 
 & Sdev & 55110 & \textbf{34894} & 327270 & 72267 & 75961 &  &  \\ \hline
 & Median & 6470290 & 6538740 & 6641818 & 6759473 & \textbf{6768365} &  &  \\ 
brd14051 & Mean & 6476988 & 6540131 & 6520074 & \textbf{6762582} & 6756279 & \multicolumn{1}{r|}{5499906} & \multicolumn{1}{r|}{6854612} \\ 
 & Sdev & 77957 & 64430 & 410024 & \textbf{42308} & 53386 &  &  \\ \hline
 & Median & 6922835 & 7091740 & 7564499 & 7710754 & \textbf{7732018} &  &  \\ 
d15112 & Mean & 6962344 & 7119176 & 7510784 & 7701890 & \textbf{7723037} & \multicolumn{1}{r|}{6792890} & \multicolumn{1}{r|}{7942036} \\ 
 & Sdev & 137967 & 94229 & 241942 & 149233 & \textbf{43069} &  &  \\ \hline
 & Median & 7101430 & 7182725 & 7357957 & \textbf{7490813} & 7394360 &  &  \\ 
d18512 & Mean & 7128864 & 7174375 & 7133843 & \textbf{7457973} & 7407453 & \multicolumn{1}{r|}{5845239} & \multicolumn{1}{r|}{7582022} \\ 
 & Sdev & 139510 & \textbf{54109} & 555368 & 67658 & 79192 &  &  \\ \hline
 & Median & 15350400 & 15588600 & 14344461 & \textbf{16230647} & 16034864 &  &  \\ 
pla33810 & Mean & 15386770 & 15622400 & 14515200 & \textbf{16231654} & 16013497 & \multicolumn{1}{r|}{13170032} & \multicolumn{1}{r|}{16332634} \\ 
 & Sdev & 199854 & 114891 & 806290 & \textbf{58750} & 161777 &  &  \\ \hline

\end{tabular}
\end{table}

\begin{table}[htbp]
\caption{The median, the mean, and the standard deviation (StdDev) of the objectives values obtained over $10$ runs of each of MATLS, S5, CS2SA*, CoCoP, and CoCoL solvers on each CatC instance. Emboldened values are the largest median, the largest mean, and the smallest standard deviation of the objective values over the five solvers. Moreover, the maximum and the minimum of the objectives values obtained are over all $50$ runs of all $5$ solvers; these maximums and minimums are used in RDI computation in \tablename~\ref{tab:sota-comparison}.}
\label{tab:sotacatc}
\centering
\vspace{-2ex}
\setlength{\tabcolsep}{2pt}
\footnotesize
\begin{tabular}{|l|l|r|r|r|r|r|r|r|}
\hline
\bf CatC & \bf Metric & \bf MATLS & \bf S5 & \bf CS2SA* & \bf CoCoP & \bf CoCoL & \bf Min & \bf Max\\ \hline\hline

 & Median & \textbf{87997} & 87455 & 87577 & 87806 & 87629 &  &  \\ 
eil76 & Mean & \textbf{87932} & 87392 & 86611 & 87668 & 87629 & \multicolumn{1}{r|}{81940} & \multicolumn{1}{r|}{88211} \\ 
 & Sdev & 298 & 634 & 1846 & \textbf{193} & 492 &  &  \\ \hline
 & Median & 155466 & 155582 & 149656 & \textbf{158777} & 158279 &  &  \\ 
kroA100 & Mean & 155478 & 155801 & 151911 & \textbf{158758} & 158224 & \multicolumn{1}{r|}{148491} & \multicolumn{1}{r|}{158777} \\ 
 & Sdev & \textbf{22} & 693 & 3182 & 59 & 535 &  &  \\ \hline
 & Median & 206855 & 207142 & 197555 & \textbf{207159} & \textbf{207159} &  &  \\ 
ch130 & Mean & 203149 & 207142 & 196913 & 207081 & \textbf{207313} & \multicolumn{1}{r|}{194155} & \multicolumn{1}{r|}{207671} \\ 
 & Sdev & 4929 & \textbf{0} & 1181 & 246 & 248 &  &  \\ \hline
 & Median & 243122 & 246472 & 246038 & \textbf{248508} & 248351 &  &  \\ 
u159 & Mean & 243426 & 246419 & 246038 & \textbf{248508} & 248395 & \multicolumn{1}{r|}{242201} & \multicolumn{1}{r|}{249875} \\ 
 & Sdev & 758 & 158 & 4045 & \textbf{9} & 75 &  &  \\ \hline
 & Median & 426891 & 429018 & 425076 & \textbf{429138} & 429099 &  &  \\ 
a280 & Mean & 426853 & 429015 & 424822 & \textbf{429137} & 429113 & \multicolumn{1}{r|}{421778} & \multicolumn{1}{r|}{429138} \\ 
 & Sdev & 1034 & 5 & 1747 & \textbf{2} & 19 &  &  \\ \hline
 & Median & 966046 & 966903 & 955741 & \textbf{969705} & 968371 &  &  \\ 
u574 & Mean & 966064 & 967017 & 950906 & \textbf{969708} & 968225 & \multicolumn{1}{r|}{916712} & \multicolumn{1}{r|}{970343} \\ 
 & Sdev & 2437 & 691 & 13981 & 455 & \textbf{419} &  &  \\ \hline
 & Median & 1188545 & 1190005 & 1173360 & 1203516 & \textbf{1206366} &  &  \\ 
u724 & Mean & 1189346 & 1190109 & 1175045 & 1202541 & \textbf{1205830} & \multicolumn{1}{r|}{1163702} & \multicolumn{1}{r|}{1209029} \\ 
 & Sdev & 4382 & \textbf{995} & 8920 & 2120 & 2282 &  &  \\ \hline
 & Median & 1477500 & 1479750 & 1424561 & \textbf{1495041} & 1492045 &  &  \\ 
dsj1000 & Mean & 1474439 & 1479970 & 1413535 & \textbf{1494506} & 1491253 & \multicolumn{1}{r|}{1352328} & \multicolumn{1}{r|}{1496922} \\ 
 & Sdev & 8935 & \textbf{1403} & 37247 & 2103 & 2708 &  &  \\ \hline
 & Median & 2183330 & 2189320 & 2154384 & \textbf{2213935} & 2210598 &  &  \\ 
rl1304 & Mean & 2187460 & 2190216 & 2155826 & \textbf{2213447} & 2210761 & \multicolumn{1}{r|}{2118284} & \multicolumn{1}{r|}{2214091} \\ 
 & Sdev & 11071 & 4980 & 19125 & \textbf{781} & 2306 &  &  \\ \hline
 & Median & 2458630 & 2471315 & 2412735 & \textbf{2485371} & 2483177 &  &  \\ 
fl1577 & Mean & 2463556 & 2473376 & 2397668 & \textbf{2487396} & 2482978 & \multicolumn{1}{r|}{2321699} & \multicolumn{1}{r|}{2500736} \\ 
 & Sdev & 13175 & 12833 & 53022 & 6146 & \textbf{4040} &  &  \\ \hline
 & Median & 3392755 & 3430185 & 3366324 & \textbf{3496222} & 3482613 &  &  \\ 
d2103 & Mean & 3401609 & 3429943 & 3387613 & \textbf{3496012} & 3482325 & \multicolumn{1}{r|}{3362735} & \multicolumn{1}{r|}{3501889} \\ 
 & Sdev & 21813 & 8662 & 32549 & \textbf{2766} & 12532 &  &  \\ \hline
 & Median & 4559385 & 4568970 & 4571545 & \textbf{4597719} & 4583180 &  &  \\ 
pcb3038 & Mean & 4558555 & 4571752 & 4560311 & \textbf{4597236} & 4582078 & \multicolumn{1}{r|}{4459818} & \multicolumn{1}{r|}{4600973} \\ 
 & Sdev & 4161 & 5776 & 38211 & \textbf{3205} & 5396 &  &  \\ \hline
 & Median & 6544210 & 6546530 & 6545047 & \textbf{6572527} & 6568717 &  &  \\ 
fnl4461 & Mean & 6547657 & 6545619 & 6541215 & \textbf{6572238} & 6567656 & \multicolumn{1}{r|}{6482021} & \multicolumn{1}{r|}{6575472} \\ 
 & Sdev & 8640 & 4487 & 23454 & \textbf{2840} & 3978 &  &  \\ \hline
 & Median & 13983950 & 14112500 & 13676547 & \textbf{14520941} & 14382473 &  &  \\ 
pla7397 & Mean & 13934720 & 14129270 & 13604651 & \textbf{14480550} & 14368424 & \multicolumn{1}{r|}{12669917} & \multicolumn{1}{r|}{14572352} \\ 
 & Sdev & 186508 & 78387 & 544124 & 81468 & \textbf{66425} &  &  \\ \hline
 & Median & 18275600 & 18231100 & 18246266 & 18355212 & \textbf{18415752} &  &  \\ 
rl11849 & Mean & 18289600 & 18228430 & 18237586 & 18381502 & \textbf{18419397} & \multicolumn{1}{r|}{17967813} & \multicolumn{1}{r|}{18569005} \\ 
 & Sdev & 34140 & \textbf{32019} & 129381 & 87469 & 69735 &  &  \\ \hline
 & Median & 25920700 & 25867750 & 26371355 & 26375964 & \textbf{26537543} &  &  \\ 
usa13509 & Mean & 25877870 & 25890400 & 25814704 & 26404068 & \textbf{26524737} & \multicolumn{1}{r|}{21710095} & \multicolumn{1}{r|}{26728716} \\ 
 & Sdev & 196513 & \textbf{92477} & 1494909 & 104279 & 154910 &  &  \\ \hline
 & Median & 23868750 & 23808650 & 24000036 & \textbf{24196359} & 24073026 &  &  \\ 
brd14051 & Mean & 23797310 & 23815520 & 23866512 & \textbf{24152244} & 24041815 & \multicolumn{1}{r|}{23284712} & \multicolumn{1}{r|}{24361366} \\ 
 & Sdev & 176543 & 114067 & 361304 & 137805 & \textbf{113208} &  &  \\ \hline
 & Median & 26004950 & 26254100 & 26907395 & 27424573 & \textbf{27461281} &  &  \\ 
d15112 & Mean & 25972990 & 26275900 & 26975427 & \textbf{27462464} & 27451255 & \multicolumn{1}{r|}{25760900} & \multicolumn{1}{r|}{27665466} \\ 
 & Sdev & \textbf{113162} & 206275 & 307541 & 136052 & 176104 &  &  \\ \hline
 & Median & 27112750 & 27245850 & 27523202 & \textbf{27833906} & 27524914 &  &  \\ 
d18512 & Mean & 27196570 & 27168630 & 27439923 & \textbf{27766028} & 27483840 & \multicolumn{1}{r|}{26802200} & \multicolumn{1}{r|}{27951166} \\ 
 & Sdev & 256922 & 196514 & 255881 & \textbf{171997} & 216132 &  &  \\ \hline 
 & Median & 58146450 & 57576050 & 57120599 & \textbf{58689650} & 58418340 &  &  \\ 
pla33810 & Mean & 58080700 & 57505880 & 56988987 & \textbf{58680600} & 58401815 & \multicolumn{1}{r|}{55713691} & \multicolumn{1}{r|}{58900443} \\ 
 & Sdev & 206467 & 259242 & 820900 & 155349 & \textbf{120691} &  &  \\ \hline

\end{tabular}
\end{table}

\begin{table}[htbp]
\caption{The median, mean, and standard deviation (StdDev) of execution times and objectives values obtained by $10$ runs of each of MEA2P and CoCoP solvers on each of the 8 small CatA instances. Emboldened values are the largest median, the largest mean, and the smallest standard deviation of the execution times and the objective values over the two solvers. The maximum and the minimum of objectives values are over all $20$ runs of both solvers; these maximums and minimums are used in RDI computation in \tablename~\ref{tab:mea2p}.}
\label{tab:mea2pcata}
\footnotesize
\centering
\begin{tabular}{|l|l|r|r||r|r|r|r|}
\hline
\multicolumn{2}{|c||}{\bf CatA} & \multicolumn{2}{c||}{\bf Exec Time (Sec.)} & \multicolumn{4}{c|}{\bf Objective Value}\\\hline
\bf Instance & \bf Metric & \bf MEA2P &
    \bf CoCoP & \bf MEA2P & \bf CoCoP & Min & Max \\ \hline \hline
 & Median & 45 & \textbf{25} & 3992 & \textbf{4103} &  &  \\
eil76 & Mean & 45 & \textbf{27} & 4017 & \textbf{4061} & 3952 & 4103 \\
 & StdDev & \textbf{2} & 5 & \textbf{56} & 68 &  &  \\ \hline
 & Median & 94 & \textbf{45} & \textbf{4976} & 4628 &  &  \\
kroA100 & Mean & 94 & \textbf{49} & \textbf{4956} & 4628 & 4628 & 4976 \\
 & StdDev & \textbf{2} & 8 & 55 & \textbf{0} &  &  \\ \hline
 & Median & 176 & \textbf{92} & \textbf{9638} & 9564 &  &  \\
ch130 & Mean & 178 & \textbf{101} & \textbf{9635} & 9547 & 9386 & 9682 \\
 & StdDev & \textbf{8} & 20 & \textbf{42} & 57 &  &  \\ \hline
 & Median & 429 & \textbf{62} & \textbf{8956} & 8763 &  &  \\
u159 & Mean & 431 & \textbf{69} & \textbf{8977} & 8763 & 8763 & 9064 \\
 & StdDev & 17 & \textbf{13} & 63 & \textbf{0} &  &  \\ \hline
 & Median & 1901 & \textbf{88} & \textbf{18937} & 18487 &  &  \\
a280 & Mean & 1909 & \textbf{99} & \textbf{18961} & 18481 & 18411 & 19422 \\
 & StdDev & 51 & \textbf{21} & 330 & \textbf{49} &  &  \\ \hline
 & Median & 14191 & \textbf{565} & \textbf{28735} & 27359 &  &  \\
u574 & Mean & 16174 & \textbf{638} & \textbf{28560} & 27391 & 27152 & 29457 \\
 & StdDev & 3470 & \textbf{126} & 672 & \textbf{121} &  &  \\ \hline
 & Median & 23401 & \textbf{530} & 50128 & \textbf{51363} &  &  \\
u724 & Mean & 23544 & \textbf{553} & 49866 & \textbf{51366} & 49067 & 51647 \\
 & StdDev & 1192 & \textbf{73} & 531 & \textbf{202} &  &  \\ \hline
 & Median & 19984 & \textbf{2106} & 134821 & \textbf{144219} &  &  \\
dsj1000 & Mean & 21384 & \textbf{2360} & 134750 & \textbf{144219} & 133724 & 144219 \\
 & StdDev & 2701 & \textbf{397} & 646 & \textbf{0} &  &  \\ \hline
\end{tabular}
\end{table}

\begin{table}[htbp]
\caption{The median, mean, and standard deviation (StdDev) of execution times and objectives values obtained by $10$ runs of each of MEA2P and CoCoP solvers on each of the 8 small CatB instances. Emboldened values are the largest median, the largest mean, and the smallest standard deviation of the execution times and the objective values over the two solvers. The maximum and the minimum of objectives values are over all $20$ runs of both solvers; these maximums and minimums are used in RDI computation in \tablename~\ref{tab:mea2p}.}
\label{tab:mea2pcatb}
\footnotesize
\centering
\begin{tabular}{|l|l|r|r||r|r|r|r|}
\hline
\multicolumn{2}{|c||}{\bf CatB} & \multicolumn{2}{c||}{\bf Exec Time (Sec.)} & \multicolumn{4}{c|}{\bf Objective Value}\\\hline
\bf Instance & \bf Metric &\bf MEA2P &
    \bf CoCoP & \bf MEA2P & \bf CoCoP & Min & Max \\ \hline \hline
 & Median & 132 & \textbf{29} & \textbf{23278} & 22177 &  &  \\
eil76 & Mean & 135 & \textbf{29} & \textbf{23228} & 22124 & 21581 & 23278 \\
 & StdDev & 7 & \textbf{1} & \textbf{158} & 201 &  &  \\ \hline
 & Median & 298 & \textbf{51} & \textbf{46633} & 45326 &  &  \\
kroA100 & Mean & 299 & \textbf{51} & \textbf{46605} & 44380 & 42095 & 46633 \\
 & StdDev & 6 & \textbf{2} & \textbf{46} & 1500 &  &  \\ \hline
 & Median & 590 & \textbf{95} & \textbf{62496} & 61290 &  &  \\
ch130 & Mean & 591 & \textbf{96} & \textbf{62403} & 61391 & 61184 & 62496 \\
 & StdDev & 28 & \textbf{3} & \textbf{150} & 209 &  &  \\ \hline
 & Median & 1030 & \textbf{77} & \textbf{61882} & 60656 &  &  \\
u159 & Mean & 1007 & \textbf{77} & \textbf{61820} & 60714 & 60626 & 62157 \\
 & StdDev & 54 & \textbf{2} & 417 & \textbf{143} &  &  \\ \hline
 & Median & 4591 & \textbf{112} & 116120 & \textbf{116379} &  &  \\
a280 & Mean & 4596 & \textbf{113} & 116084 & \textbf{116377} & 115145 & 116810 \\
 & StdDev & 145 & \textbf{5} & 540 & \textbf{52} &  &  \\ \hline
 & Median & 45807 & \textbf{618} & \textbf{263054} & 259990 &  &  \\
u574 & Mean & 46311 & \textbf{622} & \textbf{263011} & 260025 & 259874 & 267719 \\
 & StdDev & 2373 & \textbf{20} & 2384 & \textbf{111} &  &  \\ \hline
 & Median & 100049 & \textbf{626} & 320644 & \textbf{321323} &  &  \\
u724 & Mean & 101820 & \textbf{622} & 320703 & \textbf{321476} & 316293 & 326151 \\
 & StdDev & 13357 & \textbf{15} & 2692 & \textbf{1274} &  &  \\ \hline
 & Median & 293615 & \textbf{2398} & \textbf{371169} & 370511 &  &  \\
dsj1000 & Mean & 295897 & \textbf{2657} & \textbf{372407} & 370392 & 364690 & 383577 \\
 & StdDev & 46304 & \textbf{440} & 6899 & \textbf{935} &  &  \\ \hline
\end{tabular}
\end{table}

\begin{table}[htbp]
\caption{The median, mean, and standard deviation (StdDev) of execution times and objectives values obtained by $10$ runs of each of MEA2P and CoCoP solvers on each of the 8 small CatC instances. Emboldened values are the largest median, the largest mean, and the smallest standard deviation of the execution times and the objective values over the two solvers. The maximum and the minimum of objectives values are over all $20$ runs of both solvers; these maximums and minimums are used in RDI computation in \tablename~\ref{tab:mea2p}.}
\label{tab:mea2pcatc}
\centering
\footnotesize
\begin{tabular}{|l|l||r|r||r|r|r|r|}
\hline
\multicolumn{2}{|c||}{\bf CatC} & \multicolumn{2}{c||}{\bf Exec Time (Sec.)} & \multicolumn{4}{c|}{\bf Objective Value}\\\hline
\bf Instance & \bf Metric &\bf MEA2P &
    \bf CoCoP & \bf MEA2P & \bf CoCoP & Min & Max \\ \hline \hline
 & Median & 282 & \textbf{40} & 88386 & 87161 &  &  \\
eil76 & Mean & 281 & \textbf{40} & 88386 & 87054 & 85783 & 88386 \\
 & StdDev & 7 & \textbf{1} & 0 & 696 &  &  \\ \hline
 & Median & 599 & \textbf{65} & 159135 & 155585 &  &  \\
kroA100 & Mean & 601 & \textbf{65} & 159135 & 155603 & 155585 & 159135 \\
 & StdDev & 28 & \textbf{1} & 0 & 23 &  &  \\ \hline
 & Median & 1197 & \textbf{115} & 207907 & 207159 &  &  \\
ch130 & Mean & 1218 & \textbf{123} & 207569 & 206902 & 206242 & 207907 \\
 & StdDev & 87 & \textbf{24} & 556 & 367 &  &  \\ \hline
 & Median & 1870 & \textbf{101} & 252667 & 248498 &  &  \\
u159 & Mean & 1867 & \textbf{101} & 252667 & 248352 & 248133 & 252667 \\
 & StdDev & 44 & \textbf{1} & 0 & 188 &  &  \\ \hline
 & Median & 6946 & \textbf{150} & 427716 & 429135 &  &  \\
a280 & Mean & 7174 & \textbf{159} & 427369 & 429135 & 422492 & 429138 \\
 & StdDev & 490 & \textbf{29} & 2124 & 3 &  &  \\ \hline
 & Median & 88122 & \textbf{712} & 966125 & 970049 &  &  \\
u574 & Mean & 90718 & \textbf{754} & 964333 & 970025 & 953724 & 976021 \\
 & StdDev & 17801 & \textbf{134} & 7082 & 486 &  &  \\ \hline
 & Median & 167109 & \textbf{773} & 1198353 & 1205500 &  &  \\
u724 & Mean & 180938 & \textbf{778} & 1197832 & 1205628 & 1185219 & 1211213 \\
 & StdDev & 44145 & \textbf{30} & 6418 & 1086 &  &  \\ \hline
 & Median & 537341 & \textbf{2557} & 1495453 & 1499405 &  &  \\
dsj1000 & Mean & 541756 & \textbf{2645} & 1493313 & 1500609 & 1474506 & 1506740 \\
 & StdDev & 118628 & \textbf{332} & 8938 & 4849 &  &  \\ \hline
\end{tabular}
\end{table}


\small
\clearpage

\bibliographystyle{ieee}
\bibliography{references}

\end{document}